%% file: arxiv_draft.tex
\documentclass[11pt]{article} %
\usepackage{times}
\usepackage{url}            %
\usepackage{booktabs}       %
\usepackage{amsfonts}       %
\usepackage{nicefrac}       %
\usepackage{microtype}      %
\usepackage{mkolar_definitions}
\usepackage{xspace}
\usepackage{multirow}
\usepackage{amsmath}
\usepackage{algorithm}
\usepackage{algorithmic}
\usepackage{color}
\usepackage{color, colortbl}
\usepackage{enumitem}
\usepackage{comment}
\usepackage{bm}
\usepackage[left=2cm,top=2cm,right=2cm]{geometry}

\usepackage[numbers]{natbib}
\usepackage{mkolar_definitions}

\input{notation_arxiv}

\everypar{\looseness=-1}

\usepackage[utf8]{inputenc} %
\usepackage[T1]{fontenc}    %
\usepackage{hyperref}       %
\usepackage{url}            %
\usepackage{booktabs}       %
\usepackage{amsfonts}       %
\usepackage{nicefrac}       %
\usepackage{microtype}      %
\usepackage{xcolor}         %
\usepackage{booktabs}
\usepackage{tikz}
\usepackage{wrapfig}
\usepackage{enumitem}

\newcount\Comments  %
\definecolor{darkgreen}{rgb}{0,0.5,0}
\definecolor{darkred}{rgb}{0.7,0,0}
\definecolor{teal}{rgb}{0.3,0.8,0.8}
\definecolor{orange}{rgb}{1.0,0.5,0.0}
\definecolor{purple}{rgb}{0.8,0.0,0.8}
\newcommand{\kibitz}[2]{\ifnum\Comments=1{\textcolor{#1}{\textsf{\footnotesize #2}}}\fi}

\definecolor{Gray}{gray}{0.9}
\usepackage{algorithm}
\usepackage{algorithmic}

\usepackage{thmtools}
\newtheorem{myexp}{Example}

\title{Provably Efficient Reinforcement Learning in \\ Partially Observable Dynamical Systems}

\usepackage{authblk}

\author[1]{Masatoshi Uehara\thanks{mu223@cornell.edu Supported by Masason Foundation }}
\author[1]{Ayush Sekhari\thanks{as3663@cornell.edu }}
\author[2]{Jason D. Lee\thanks{jasonlee@princeton.edu }}
\author[1]{Nathan Kallus\thanks{kallus@cornell.edu  }}
\author[1]{Wen Sun\thanks{ws455@cornell.edu} }

\affil[1]{Cornell University} 
\affil[2]{Princeton University}

\date{}

\ifdefined\usebigfont

\usepackage{times}
\usepackage[fontsize=13pt]{scrextend}
\AtBeginDocument{\newgeometry{letterpaper,left=1.56in,right=1.56in,top=1.71in,bottom=1.77in}}
\else
\fi

\begin{document}

\maketitle

\begin{abstract}
We study Reinforcement Learning for  partially observable dynamical systems using function approximation. We propose a new \textit{Partially Observable Bilinear Actor-Critic framework}, that is general enough to include models such as observable tabular Partially Observable Markov Decision Processes (POMDPs), observable Linear-Quadratic-Gaussian (LQG), Predictive State Representations (PSRs),  as well as a newly introduced model Hilbert Space Embeddings of POMDPs and observable POMDPs with latent low-rank transition. Under this framework, we propose an actor-critic style algorithm that is capable of performing agnostic policy learning. Given a policy class that consists of memory based policies (that look at a fixed-length window of recent observations), and a value function class that consists of functions taking both memory and future observations as inputs, our algorithm learns to compete against the best memory-based policy in the given policy class. For certain examples such as undercomplete observable tabular POMDPs, observable LQGs and observable POMDPs with latent low-rank transition, by implicitly leveraging their special properties, our algorithm is even capable of competing against the globally optimal policy without paying an exponential dependence on  the horizon in its sample complexity.\looseness=-1
\end{abstract}

\input{./main_document/main_intro}

\input{./main_document/main_prelim}

\input{./main_document/main_introduce_value_bilinear}

\input{./main_document/main_example}

\input{./main_document/main_algorithm}

\input{./main_document/main_psr}

\input{./main_document/minmax_extension}

\section{Summary}

{ We propose a PO-bilinear actor-critic framework that is the first unified framework for provably efficient RL on large-scale partially observable dynamical systems. Our framework can capture not only many models where provably efficient learning has been known such as tabular POMDPs, LQG and M-step decodable POMDPs, but also models where provably efficient RL is not known such as HSE-POMDPs, general PSRs and observable POMDPs with low-rank latent transition. Our unified actor-critic based algorithm---\textbf{\ouralg} provably performs agnostic learning by searching for the best memory-based policy.  For special models such as observable tabular MDPs, LQG, and POMDPs with low-rank latent transition, by leveraging their special properties, i.e., the exponential stability of Bayesian filters in tabular and low-rank POMDPs, and existence of a compact featurization of histories in LQG, we are able to directly compete against the global optimality without paying an exponential dependence on horizon. }

\section*{Acknowledgement}

We thank Nan Jiang for valuable discussions on PSRs. 

\bibliographystyle{alpha}
\bibliography{ref}

\appendix 
\newpage 

\input{./main_document/ape_related_work}

\input{./main_document/ape_introduce_value_bilinear}

\input{./main_document/ape_example}

\input{./main_document/ape_algorithm}

\input{./main_document/ape_psr}

\input{proof}

\input{./main_document/proof_more_general}

\input{proof2}

\input{./main_document/exponential_stability}

\input{./main_document/ape_auxi}

\end{document}

%% file: notation_arxiv.tex
\newcommand{\newedit}{\color{black}}

\newcommand{\TT}{\mathbb{T}} %
\newcommand{\OO}{\mathbb{O}} %
\newcommand{\One}{\mathbf{1}} %

%% file: main_document/main_intro.tex
\section{Introduction}
Large state  space and partial observability are two key challenges of Reinforcement Learning (RL). While recent advances in RL for fully observable systems have focused on the challenge of scaling RL to large state space in both theory and in practice using rich function approximation, the understanding of large scale RL under partial observability is still limited. In POMDPs, for example, a core issue is that the optimal policy is not necessarily Markovian since the observations are not Markovian. 

A common heuristic to tackle large scale RL with partial observability in practice is to simply maintain a time window of the history of observations, which is treated as a state to feed into the policy and the value function. Such a window of history can be often maintained explicitly via truncating away older history (e.g., DQN uses a window with length 4 for playing video games \citep{mnih2013playing}; Open AI Five uses a window with length 16  for LSTMs \cite{berner2019dota}).  
Since even for planning under partial observations and  known dynamics, finding the globally optimal policy conditional on the entire history is generally NP-hard (due to the curse of the history)
\citep{littman1996algorithms,papadimitriou1987complexity,golowich2022planning}, searching for a short memory-based policy can be understood as a reasonable middle ground that balances computation and optimality. The impressive empirical results of these prior works also demonstrate that in practice, there often exists a high-quality policy (not necessarily the globally optimal) that is only a function of a short window of recent observations. 
However, these prior works that search for the best memory-based policy
 unfortunately cannot ensure sample efficient PAC guarantees due to the difficulty of strategic exploration in POMDPs. 
The key question that we aim to answer in this work is:

\begin{center}
\emph{Can we design provably efficient RL algorithms that agnostically learn the best fixed-length memory based policy with function approximation?}
\end{center}

We provide affirmative answers to the above question. More formally, we study RL for partially observable dynamical systems that include not only the classic Partially Observable MDPs (POMDPs) \citep{murphy2000survey,porta2006point,shani2013survey}, but also a more general model called Predictive State Representations (PSRs) \citep{littman2001predictive}. %
We design a model-free actor-critic framework, named \emph{PO-Bilinear Actor-Critic Class},  where we have a policy class (i.e., actors) that consists of policies that take a fixed-length window of observations as input (memory-based policy), and a newly introduced value link function class (i.e., critics) that consists of functions that take the fixed-length window of history and {(possibly multi-step if the system is overcomplete)} \emph{future observations} as inputs. 
A value link function class is an analog of the value function class tailored to partially observable systems that only relies on observable quantities (i.e., past and future observations and actions).  In our algorithm, we \emph{agnostically} search for the best memory-based policy from the given policy class.  \looseness=-1

Our framework is based on the idea of a newly introduced notion of \textit{value link function} equipped with future observations. While the idea of using future observations has been used in the literature on POMDPs, our work is the first to use this idea to learn a high-quality policy in a model-free manner. Existing works discuss how to use future observations only in a model-based manner \citep{boots2011closing,hefny2015supervised}. By leveraging these model-based viewpoints, while recent works discuss strategic exploration to learn near-optimal policies, their results are either limited to the tabular setting (and are not scalable for large state spaces)  \citep{jin2020provably,guo2016pac,azizzadenesheli2016reinforcement,xiong2021sublinear,liu2022partially} or are tailored to specific non-tabular models and unclear how to incorporate general function approximation \citep{simchowitz2020improper,lale2020regret,cai2022sample}. We break these barriers by devising a new actor-critic-based model-free view on POMDPs. We demonstrate the \emph{scalability} and \emph{generality} of our PO-bilinear actor-critic framework by showing PAC-guarantee on many models as follows (see Table~\ref{tab:my_label} for a summary).

{\textbf{Observable Tabular POMDPs.} In tabular observable POMDPs, i.e., POMDPs where \emph{multi-step} future observations retain information about latent states, the PO-bilinear rank decomposition holds. We can ensure the sample complexity is $\mathrm{\mathrm{Poly}}(S,A^M,O^M, A^K, O^K, H,1/\sigma_1)$ %
where $\sigma_1 =\min_{x} \|\OO x\|_1/\|x\|_1$ ($\OO$ is an emission matrix),and $S,A,O$ are the cardinality of state, action, observation space, respectively, $H$ is the horizon, and $K$ is the number of future observations.\footnote{{\newedit In \pref{sec:example_general}, we discuss how to get rid of $O^M,O^K$ using a model-based learning perspective. The intuition is that a tabular POMDP's model complexity has nothing to do with $M$ or $K$, i.e., number of parameters in transition and omission distribution is $S^2 A + OA$ (even if we consider the time-inhomogeneous setting, it scales with $H(S^2 A + OA)$, but no $O^M$ and $O^K$) and the PO-bilinear rank is still $S$. }} In the special undercomplete ($O\geq S$) case, our framework is also flexible enough to set the memory length according to the property of the problems in order to search for the globally optimal policy. More specifically, using the latest result from \cite{golowich2022planning} about belief contraction, we can set $ M = \tilde O( (1/\sigma^4_1)\ln( SH/\epsilon))$ with $\epsilon$ being the optimality threshold. This allows us to compete against the globally optimal policy without paying an exponential dependence on $H$.} 

\textbf{Observable Linear Quadratic Gaussian (LQG). } In observable LQG -- a classic partial observable linear dynamical system, our algorithm can compete against the \emph{globally optimal policy} with a sample complexity scaling polynomially with respect to the horizon, dimensions of the state, observation, and action spaces (and other system parameters). This is achieved by simply setting the memory length $M$ to $H$. The special linear structures of the problem allow us to avoid exponential dependence on $H$ even when using the full history as a memory. 
While the global optimality results in tabular POMDPs and LQG exist by using different algorithms, \emph{to the best of our knowledge, this is the first unified algorithm that can solve both tabular POMDPs and LQG simultaneously without paying an exponential dependence on horizon $H$.} \looseness=-1

\textbf{Observable Hilbert Space Embedding POMDPs (HSE-POMDPs).} Our framework ensures the agnostic PAC guarantee on HSE-POMDPs where policy induced transitions and omission distributions have condition mean embeddings \citep{boots2011closing,boots2013hilbert}. This model  naturally generalizes tabular POMDPs and LQG. We show that the sample complexity scales polynomially with respect to the dimensions of the embeddings.  This is the \emph{first} PAC guarantee in HSE-POMDPs.

\textbf{Predictive State Representations (PSRs).} 
We give the \emph{first} PAC-guarantee on PSRs. PSRs model partially observable dynamical systems without even using the concept of latent states and strictly generalize the POMDP model.  Our work significantly generalizes a prior PAC learning result for reactive PSRs (i.e., reactive PSRs require a strong condition that the optimal policy only depends on the latest observation) which is a much more restricted setting \citep{jiang2017contextual}.

\textbf{$M$-step decodable POMDPs \citep{efroni2022provable}.} Our framework can capture $M$-step decodable POMDPs where there is a (unknown) decoder that can perfectly decode the latent state by looking at the latest $M$-memory. Our algorithm can compete against the globally optimal policy with the sample complexity scaling polynomially with respect to horizon $H$, $S, A^M$, and the statistical complexities of function classes, without any explicit dependence on $O$. This PAC result is similar with the one from \citep{efroni2022provable}.

{\textbf{Observable POMDPs with low-rank latent transition.} Our framework captures observable POMDPs where the latent transition is low-rank. This is the \emph{first} PAC guarantee in this model.  
Under this model, we first show that with $M = \tilde O\left( (1/\sigma_1^4) \ln( d H / \epsilon  )\right)$ where $d$ is the rank of the latent transition matrix, there exists an $M$-memory policy that is $\epsilon$-near optimal with respect to the globally optimal policy. %
Then, starting with a general model class that contains the ground truth transition and omission distribution (i.e., realizability in model class), we first convert the model class to a policy class and a value link function class, and we then show that our algorithm competes against the globally optimal policy with a sample complexity scaling polynomially with respect to $H, d, |\Acal|^{(1/\sigma_1^4) \ln( d H / \epsilon  )}, 1/\sigma_1$, and the statistical complexity of the model class. 
Particularly, the sample complexity has no explicit  dependence on the size of the state and observation space, instead it just depends on the statistical complexity of the given model class.   
}

\begin{table}[!t] \label{tab:low_po_bilinear}
    \centering
  \resizebox{1.0\textwidth}{!}{
    \begin{tabular}{ccccccc}
      \toprule
      Model   &  
      \begin{tabular}{c}
         Observable   \\
         tabular   POMDPs 
      \end{tabular}
     & \begin{tabular}{c}
          Observable \\
          LQG
     \end{tabular}
      &  \begin{tabular}{c}
          Low-rank $M$-step  \\
          decodable POMDPs
      \end{tabular}
      & \begin{tabular}{c}
           Observable \\
           HSE-POMDPs
      \end{tabular} 
      
      & 
      \begin{tabular}{c}
           \\
          PSRs
      \end{tabular} &
      
            \begin{tabular}{c}
           Low rank\\ observable POMDPs
      \end{tabular}
     
      \\
   \midrule 
    PO-Bilinear Rank     &  
\begin{tabular}{c}
        $(O A)^M S (\dagger)$  \\
     {\newedit (Can be $S$) }
\end{tabular}

    & $O(M d^2_a d^2_s ) (\dagger)$ & Rank $(\dagger)$ &  
    \begin{tabular}{c}    
    Feature dimension \\ 
    on $(z, s)$ 
    \end{tabular}
    & 
    \begin{tabular}{c}
      $(OA)^M\times$      \\ 
     \# of core tests    
    \end{tabular}  
    & Rank ($\dagger$) \\
    \midrule
    PAC Learning     &  Known & Known & Known & New & New & New  \\
    \bottomrule
    \end{tabular}
    }
    \caption{Summary of settings that are from PO-Bilinear AC class. The 2nd row gives the parameters that bound the  PO-Bilinear rank. Here $M$ denotes the length of memory used to define memory-based policies $\pi(\cdot | \bar z_h )$ where  $\bar z_h = ( o_{h-M:h}, a_{h-M:h-1} )$ denotes the $M$-step memory. 
    In the 3rd row, ``known'' means that sample-efficient algorithms already exist.  
    `` New'' means our result gives the first sample-efficient algorithm. { However, even in ``known'' case, agnostic guarantees are new; hence, when the policy class is small, we can gain some benefit. 
    The symbol $\dagger$ means we can compete with the globally optimal policy without paying an exponential dependence on horizon $H$.
    For the tabular case, the PO-bilinear rank can be improved to $S$ when we use the most general definition (Refer to \pref{sec:example_general}. }
    For LQG, $d_a$ and $d_s$ are the dimension of action and state spaces. 
    For PSRs, $O$ and $A$ denote the size of observation and action spaces. 
    \looseness=-1
}
    \label{tab:my_label}
\end{table}

\subsection{Related Works}\label{sec:related_work}

\paragraph{Generalization and function approximation of RL in MDPs. } In Markovian environments, there is a growing literature that gives PAC bounds with function approximation under certain models. Some of the representative models are linear MDPs \citep{jin2020provably,yang2020reinforcement}, block MDPs \citep{du2019provably,misra2020kinematic,zhang2022efficient}, and low-rank MDPs \citep{agarwal2020flambe,uehara2021representation}.  
Several general frameworks in  \citep{jiang2017contextual,sun2019model,jin2021bellman,foster2021statistical,du2021bilinear} characterize sufficient conditions for provably efficient RL. Each above  model is captured in these frameworks as a special case. 
While our work builds on the bilinear/Bellman rank framework \citep{du2021bilinear,jiang2017contextual}, when we na\"ively reduce POMDPs to MDPs, the bilinear/Bellman rank is $\Theta(A^H)$. These two frameworks are only shown applicable to reactive POMDPs where the optimal policy only depends on the latest observation. However, this assumption makes the POMDP model very restricted.

\paragraph{Online RL  for POMDPs. }
Prior works \citep{kearns1999approximate,even2005reinforcement} showed $A^H$-type sample complexity bounds for general POMDPs. Exponential dependence can be circumvented with more structures. First, in the tabular setting, under observability assumptions, in \citep{azizzadenesheli2016reinforcement,guo2016pac,jin2020sample,liu2022partially,golowich2022learning}, favorable sample complexities are obtained by leveraging the spectral learning technique \citep{hsu2012spectral} %
(see section 1.1 in \cite{jin2020sample} for an excellent summary). 
Second,  in LQG, which is a partial observable version of LQRs, in \citep{lale2020regret,simchowitz2020improper}, sub-linear regret algorithms are proposed.  %
These works use random policies for exploration, which is sufficient for LQG. Since random exploration strategy is not enough for tabular POMDPs, it is unclear if the existing techniques from LQG can be applied to solve general POMDPs.  Third, the recent work \citep{efroni2022provable} provides a new model called $M$-step decodable POMDP (when $M = 1$, it is Block MDP) with an efficient algorithm. 

Our framework captures \emph{all} above mentioned POMDP models. In addition, we propose a new model called HSE-POMDPs which extends prior works on HSE-HMM\citep{boots2013hilbert} to POMDPs and includes LQG and tabular POMDPs.  Our algorithm delivers the first PAC bound for this model.  

{ Finally, we remark there are several existing POMDP models that it is unclear whether our framework can capture. The first model is a POMDP \citep{cai2022sample} where emissions and transitions are modeled by linear mixture models. The second model is a latent POMDP \citep{kwon2021rl}.  We leave it as future works. }

 \looseness=-1%

\paragraph{System identification for uncontrolled partially observable systems.}

There is a long line of work on system identification for uncontrolled partially observable systems, among which the spectral learning based methods are related to our work \citep{van2012subspace,hsu2012spectral,song2010hilbert,boots2011closing,hamilton2013modelling,nishiyama2012hilbert,boots2013hilbert,kulesza2015spectral,hefny2015supervised,sun2016learning}. Informally, these methods leverage the high-level idea that under some observability conditions, one can use the sufficient statistics of (possibly multi-step) future observations as a surrogate for the belief states, thus allowing the learning algorithms to ignore the latent state inference and completely rely on observable quantities. Our approach shares a similar spirit in the sense that we use sufficient statistics of future observations to replace latent states, and our algorithm only relies on observable quantities. The major difference is that these prior works only focus on passive system identification for uncontrolled systems, while we need to find a high-performance policy by actively interacting with the systems for information acquisition. \looseness=-1

\paragraph{Reinforcement learning in PSRs.} PSRs \citep{jaeger2000observable,littman2001predictive,singh2004predictive,boots2011closing,thon2015links} are models that generalize POMDPs. PSRs also rely on the idea of using the sufficient statistics of multi-step future observations (i.e., predictive states) to serve as a summary of the history.  Prior works on RL for PSRs \citep{boots2011closing,kulesza2015spectral,downey2017predictive,li2020efficient,izadi2008point} do not address the problem of strategic exploration and operate under the assumption that a pre-collected diverse training dataset is given and the data collection policy is a blind policy (i.e., it does not depend on history of observations).  To our knowledge, the only existing PAC learning algorithm for PSRs is limited to a reactive PSR model where the optimal policy depends just on the latest observation \cite{jiang2017contextual}. Our framework captures standard PSRs models that are strictly more general than reactive PSRs.%

\paragraph{Value link functions.}

Analogue of value link functions (referred to as bridge functions) are used in the literature of causal inference (offline contextual bandits) \citep{miao2018confounding,deaner2018proxy,cui2020semiparametric,kallus2021causal,mastouri2021proximal,singh2021finite,xu2021deep} and offline RL with unmeasured confounders \citep{bennett2021,shi2021minimax}. However, their settings are not standard POMDPs in the sense that  their setting is a POMDP with unmeasured confounders following \citep{tennenholtz2020off}.  Our setting is a standard POMDP without unmeasured confounders. Here, we emphasize that their setting does \emph{not} capture our setting. More specifically, by taking \citep{shi2021minimax} as an example, they require that logged data is generated by policies that can depend on latent states but cannot depend on observable states. Thus, their definition of link functions (called as bridge functions) is not applicable to our setting since the data we use is clearly generated by policies that depend on observations. Due to this difference, their setting prohibits us from using future observations, unlike our setting. Finally, we stress that our work is online, while their setting is offline. Hence, they do not discuss any methods for exploration.

\subsection{Organization}

{In \pref{sec:preliminary}, we introduce the notation, definition of POMDPs, and our function-approximation setup such as the policy and value link function class.  In \pref{sec:def_value_bridge}, we define value link functions and the PO-bilinear actor-citric class. In \pref{sec:example}, we give examples that admit PO-bilinear actor-citric class including observable undercomplete tabular POMDPs, observable overcomplete tabular POMDPs, observable LQG, and observable HSE-POMDPs. In \pref{sec:algorithm}, we give a unified algorithm for the PO-bilinear actor-citric class, and  the sample complexity of the algorithm. We also instantiate this general result for examples presented in \pref{sec:example}.  In \pref{sec:psr}, we show that PSRs, which are more general models than POMDPs, also admit PO-bilinear rank decomposition. In \pref{sec:general}, we give a more general definition of PO-bilinear actor-critic class, followed by showing that two additional examples --- $M$-step decodable POMDPs and observable POMDPs with low-rank latent transition --- fall into this general definition (\pref{sec:example_general}). Both the examples use general nonlinear function approximation and their sample complexities do not explicitly depend on the size of the state and observation spaces, but only on the statistical complexities of the function classes. 
As a by-product, we can refine the sample complexity result in the tabular case in \pref{sec:algorithm}. Most of the proofs are deferred to the Appendix. 
}

%% file: main_document/main_prelim.tex
\section{Preliminary} \label{sec:preliminary}

\begin{wrapfigure}{!r}{0.30\textwidth}  \label{fig:pomdp}
\centering
\begin{tikzpicture}
\draw[step=0.5cm,gray,very thin] (0,0) grid (4.26,0.5);
\draw (0.25 cm,0.25cm) -- (0.25 cm,0.25cm)node  {$o_3$};
\draw (0.75 cm,0.25cm) -- (0.75 cm,0.25cm)node  {$a_3$};
\draw (1.25 cm,0.25cm) -- (1.25 cm,0.25cm)node  {$o_4$};
\draw (1.75 cm,0.25cm) -- (1.75 cm,0.25cm)node  {$a_4$};
\draw (2.25 cm,0.25cm) -- (2.25 cm,0.25cm)node  {$o_5$};
\draw (2.75 cm,0.25cm) -- (2.75 cm,0.25cm)node  {$a_5$};
\draw (3.25 cm,0.25cm) -- (3.25 cm,0.25cm)node  {$o_6$};
\draw (3.75 cm,0.25cm) -- (3.75 cm,0.25cm)node  {$a_6$};
\draw[red, <->] (0,-0.25) -- (3.5,-0.25); 
\draw (2 cm,-0.5cm) -- (2 cm,-0.55cm)node  {$\bar z_{6}$};
\draw[red, <->] (0,0.75) -- (3.0,0.75); 
\draw (1.75 cm,1.0cm) -- (1.75 cm,1.0cm) node  {$z_{5}$};
\draw[blue, ->] (2.25,-0.5) .. controls (3.25,-0.75) ..(3.75,0); 
\end{tikzpicture}
\caption{Case with M$=3$. A 3-memory policy determines action $a_6$ based on $\bar z_6$.}
\vspace{-0.3cm}
\end{wrapfigure}
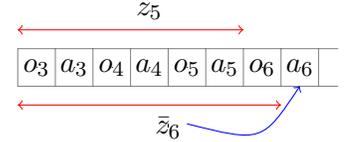 

We introduce background for POMDPs here and defer the introduction of PSRs to \pref{sec:psr}. 
We consider an episodic POMDP specified by $\Mcal = \langle \Scal, \Ocal, \Acal, H, \TT, \OO \rangle $, where $\Scal$ is the \emph{unobserved} state space, $\Ocal$ is the observation space, $\Acal$ is the action space, $H$ is the horizon, $\TT:\Scal \times \Acal \to \Delta( \Scal)$ is the transition probability, $\OO: \Scal \to \Delta(\Ocal)$ is the emission probability, and $r : \Ocal \times \Acal \to \RR$ is the reward. %
Here, $\TT,\OO$ are unknown distributions. For notation simplicity, we consider the time-homogeneous case in this paper; Extension to the time-inhomogeneous setting is straightforward. 

In our work, we consider $M$-memory policies. Let $\Zcal_h =  (\Ocal \times \Acal )^{\min \{h,M\}}$ and $\bar \Zcal_h =\Zcal_{h-1} \times \Ocal$. An element $z_h \in \Zcal_h$ is represented as $z_h = [ o_{\max(h-M+1, 1):h}, a_{\max(h-M+1,1):h} ]$, and an element $\bar z_h \in \bar \Zcal_h$ is represented as $\bar z_h = [ o_{\max(h-M, 1):h}, a_{\max(h-M,1):h-1} ]$ (thus, $\bar z_h = [ z_{h-1}, o_h ]$). %
\pref{fig:pomdp} illustrates this situation.  An $M$-memory policy is defined as $\pi =\{\pi_h\}_{h=1}^H$ where each $\pi_h$ is a mapping from  $\bar \Zcal_h$ to a distribution over actions $\Delta(\Acal)$. 

In a POMDP, an $M$-memory policy generates the data as follows.  Each episode starts with the initial state $s_1$ sampled from some unknown distribution. At each step $h \in [H]$, from $s_h \in \Scal$, the agent observes $o_h \sim \OO(\cdot |s_h)$, executes action $a_h \sim \pi_h(\cdot | \bar z_h)$, receives reward $r(s_h, a_h)$, and transits to the next latent state $s_{h+1} \sim \TT(\cdot | s_h, a_h)$. Note that the agent does not observe the underlying states but only the observations \(\crl{o_h}_{h \leq H}\). We denote $J(\pi)$ as the value of the policy $\pi$, i.e.,  $\EE[\sum_{h=1}^{H} r_h; a_{1:H}\sim \pi]$ where the expectation is taken w.r.t.~the stochasticity of the policy $\pi$, emissions distribution \(\OO\) and transition dynamics \(\TT\). 

We define a value function for a policy \(\pi\) at step $h$ to be the expected cumulative reward to go under the policy $\pi$ starting from a $z \in \Zcal_{h-1}$ and $s \in \Scal$, i.e. $V^{\pi}_h:\Zcal_{h-1}  \times \Scal \to \RR$ where $V^{\pi}_h(z, s)=\EE[\sum_{h'=h}^{H} r_{h'} \mid z_{h-1}=z,s_h=s; a_{h:H} \sim \pi]$. The notation $\EE[\cdot\,; a_{h:H} \sim \pi]$ means the expectation is taken under a policy $\pi$ from $h$ to $H$. Compared to the standard MDP setting, the expectation is conditional on not only $s_h$ but also $z_{h-1}$ since we consider $M$-memory policies.  The corresponding Bellman equation for $V^{\pi}_h$ is $
 V^{\pi}_h( z_{h-1}, s_h) = \EE \left[ r_h + V^{\pi}_{h+1}(z_h, s_{h+1})\mid z_{h-1},s_h ; a_h \sim \pi \right]$.

\paragraph{The Actor-critic function approximation setup.} Our goal is to find a near optimal policy that maximizes the policy value $J(\pi)$ in an online manner. Since any POMDPs can be converted into MDPs  by setting the state at level $h$ to the observable history up to $h$, any off-the-shelf online provably efficient algorithms for MDPs can be applied to POMDPs. By defining $\Hcal_h$ as the whole history up to step $h \in [H]$ (i.e., a history $\tau_h\in \Hcal_h$ is in the form of $o_{1:h}, a_{1:h-1}$)
, these na\"ive algorithms ensure that output policies can compete against the globally optimal policy $  \pi^{\star}_{\mathrm{gl}} =\argmax_{\pi \in \bar \Pi}J(\pi)$ where $\tilde \Pi=\{\bar \Pi_h\}, \tilde \Pi_h=[\Hcal_h \to \Delta(\Acal)]$.  
However, this conversion results in the error with exponential dependence on the horizon $H$, which is prohibitively large in the long horizon  setting. %

Instead of directly competing against the globally optimal policy, we aim for \emph{agnostic policy learning}, i.e., compete against the best policy in a given $M$-memory policy class. Our function approximation setup consists of two function classes, $(a)$ A policy class $\Pi$ consisting of %
 $M$-memory policies $\Pi := \{\Pi_h\}_{h=1}^H$ where $\Pi_h  \subset [\bar \Zcal_h \to \Delta(\Acal)]$ (i.e., actors), $(b)$ A set of value link functions $\Gcal = \{\Gcal_h\}_{h=1}^H$ where $\Gcal_h \subset [\bar \Zcal_h \to \RR]$, whose role is to approximate $V^\pi_h$ (i.e., critics). Our goal is to provide an algorithm that outputs a policy $\hat \pi= \{\hat \pi_h\}$ that has a low excess risk, where excess risk is defined by $R(\pi):=J(\hat \pi)-J(\pi^{\star})$ where $ \pi^{\star} = \argmax_{\pi \in \Pi} J(\pi)$ is the best policy in class $\Pi$. To motivate this agnostic setting, $M$-memory policies are also widely used in practice, e.g., DQN \citep{mnih2013playing} sets $M = 4$.
 Besides, there are natural examples where $M$-memory policies are close to the globally optimal policy with $M$ being only polynomial with respect to other problem dependent parameters, e.g., observable POMDPs \citep{golowich2022planning} and LQG \citep{lale2020regret,simchowitz2020improper,mania2019certainty}. 
 We will show the global optimality in these two examples later, without any exponential dependence on $H$ in the sample complexity. 
 
 \begin{remark}[Limits of existing MDP actor-critic framework]\label{rem:limit}
 While general actor-critic framework proposed in MDPs \citep{jiang2017contextual} is applicable to POMDPs via the na\"ive POMDP to MDP reduction, it is unable to leverage any benefits from the restricted policy class. This na\"ive reduction (from POMDP to MDP) uses full history and will incur sample complexity that scales exponentially with respect to the horizon. 
 \end{remark}

\noindent 
\textbf{Additional notation.} Let $[H]=\{1,\cdots,H\}$ and $[t]=\{1,\cdots,t\}$. Give a matrix $A$, we denote its pseudo inverse by $A^{\dagger}$ and the operator norm by $\|A\|$. We define the $\ell_1$ norm $\|A\|_1 = \max_{x: x\neq 0} \| Ax \|_1 / \|x\|_1$. 
The outer product is denoted by $\otimes$. Let $d^{\pi}_h(\cdot) \in \bar Z_{h}\times \Scal$ be the marginal distribution at $h$ and  $\delta(\cdot)$ be the Dirac delta function. We denote the policy $\delta(a=a')$ by $\mathrm{do}(a')$. We denote a uniform action by $\Ucal(\Acal)$. Given a function class $\Gcal$, we define $\|\Gcal\|_{\infty}=\sup_{g\in \Gcal}\|g\|_{\infty}$.

%% file: main_document/main_introduce_value_bilinear.tex
\section{Value Link Functions and the PO-bilinear Framework}\label{sec:def_value_bridge}

Unlike MDPs, we cannot directly work with value functions $V^\pi_h(s)$ (or Q functions) in POMDPs, since they depend on the unobserved state $s$. To handle this issue, below we first introduce new value link functions  by using future observations, and then discuss the PO-bilinear framework. 

\subsection{Value Link Functions} 

\begin{definition}[K-step value link functions]
Fix a set of policies $\pi^{out}=\{\pi^{out}_i\}_{i=1}^K $ where $\pi^{out}_i:\Ocal \to \Delta(\Acal)$. Value link functions $g^\pi_h: \Zcal_{h-1} \times \Ocal^{K}\times \Acal^{K-1} \to \mathbb{R}$ at step $h \in [H]$ for a policy $\pi$ are defined as the solution to the following integral equation:
\begin{align*}
  \forall z_{h-1}\in \Zcal_{h-1}, s_h \in\Scal, \qquad  \EE[ g^{\pi}_h(z_{h-1},o_{h:h+K-1},a_{h:h+K-2})\mid z_{h-1},s_h; a_{h:h+K-2}\sim \pi^{out} ] = V^{\pi}_h(z_{h-1},s_h), 
\end{align*}
where the expectation is taken under the policy $\pi^{out}$. 
\end{definition}

Link functions do not necessarily exist, nor are needed to be unique. At an intuitive level, K-step value link functions are embeddings of the value functions onto the observation space, and its existence essentially means that K-step futures have sufficient information to recover the latent state dependent value function. {The proper choice of $\pi^{out}$ would depend on the underlying models. For example, we use uniform policy in the tabular case, and $\delta(a=0)$ in LQG. }
For notational simplicity, we mostly focus on the case of $K=1$, though we will also discuss the general case of $K \geq 2$. The simplified definition for 1-step link functions is provided in the following. Note that this definition is agnostic to $\pi^{out}$. 

\begin{definition}[1-step value link functions] 
One-step value link functions $g^\pi_h: \Zcal_{h-1} \times \Ocal \to \mathbb{R}$ at step $h \in [H]$ {\newedit for a policy $\pi$} are defined as the solution to the following integral equation: 
\begin{align}\label{eq:bridge}
     \forall z_{h-1} \in \Zcal_{h-1}, s_h \in\Scal: \qquad \EE[ g^{\pi}_h(z_{h-1}, o_h)\mid z_{h-1},s_h] = V^{\pi}_h(z_{h-1},s_h). 
\end{align}
\end{definition}

In \pref{sec:example}, we will demonstrate the form of the value link function for various examples. The idea of encoding latent state information using the statistics of (multi-step) futures have been widely used in learning models of HMMs \citep{song2010hilbert,hsu2012spectral}, PSRs \citep{boots2011closing,boots2013hilbert,hefny2015supervised,sun2016learning}, and system identification for linear systems \citep{van2012subspace}. Existing provably efficient (online) RL works for POMDPs elaborate on this viewpoint \citep{jin2020provably,guo2016pac,azizzadenesheli2016reinforcement}. Compared to them, the novelty of link functions is that it is introduced to recover \emph{value functions} but not \emph{models}. This model-free view differs from the existing dominant model-based view in online RL for POMDPs. In our setup, we can control systems if we can recover value functions on the underlying states even if we fail to identify the underlying model.

\subsection{The PO-Bilinear Actor-critic Framework for POMDPs}

With the definition of value link functions, we are now ready to introduce the PO-bilinear actor-critic (AC) class for POMDPs. We will focus on the case of $K=1$ here. Let $\Gcal=\{\Gcal_h\}_{h=1}^H$, where $\Gcal_h \subset [\bar \Zcal_{h}  \to \RR]$, be a class consisting of functions that satisfy the following realizability assumption w.r.t.~ the policy class \(\Pi\). 

\begin{assum}[Realizability]\label{assum:realizability}
We assume that $\Gcal$ is realizable w.r.t.~the policy class \(\Pi\), i.e., $\forall \pi\in \Pi, h \in [H]$, there exists \emph{at least one} $g^\pi_h \in \Gcal_h$ such that \(g^\pi_h\) is a value link function w.r.t.~the policy \(\pi\). Note that realizability implicitly requires the existence of link functions. %
\end{assum}

{ 
We next introduce the PO-Bilinear Actor-critic class. For each level $h\in [H]$, we first define the Bellman loss: 
\begin{align*}
    \mathrm{Br}_h(\pi,g;\pi^{in}) := \EE[g_h(\bar z_h)-r_h-g_{h+1}(\bar z_{h+1}):a_{1:h-1}\sim \pi^{in},a_h \sim \pi ]
\end{align*}
given M-memory policies $\pi =\{\pi_h\},\pi^{in} = \{\pi^{in}_h\}$ and $g=\{g_h\}$. Letting $g^{\pi} = \{g^{\pi}_h\}_{h=1}^H$ be a link function for $\pi$, our key observation is that value link functions satisfy 
\begin{align*}
  0=\mathrm{Br}_h(\pi,g^{\pi};\pi^{in})
\end{align*}
for any M memory roll-in policy $\pi^{in}=\{\pi^{in}_h\}_{h=1}^H$, and any evaluation pair $(\pi, g^\pi)$. }
This is an analog of Bellman equations on MDPs. The above equation tells us that $\mathrm{Br}_h(\pi,g;\pi^{in})$ is a right loss to quantify how much the estimator $g$ is different from $g^{\pi}_h$. When  $\mathrm{Br}_h(\pi,g;\pi^{in})$ has a low-rank structure in a proper way, we can efficiently learn a near optimal M memory policy. The following definition precisely quantifies the low-rank structure that we need for sample efficient learning.

\begin{definition}[PO-bilinear AC Class, $K =1$]\label{def:simple_bilinear}
The model is a PO-bilinear Actor-critic class of rank $d$ if $\Gcal$ is realizable, and there exist $W_h:\Pi \times \Gcal \to \RR^d$ and $X_h:\Pi \to \RR^d$ such that for all $\pi',\pi \in \Pi, g \in \Gcal$ and $h\in [H]$,  \looseness=-1
\begin{enumerate}
    \item  $  \EE[ g_h(\bar z_h)-r_h-g_{h+1}(\bar z_{h+1}); a_{1:h-1}\sim \pi', a_h\sim \pi   ]  =  \langle W_h(\pi,g), X_h(\pi') \rangle.$
     \item $ W_h(\pi,g^{\pi})=0 $ for any $\pi \in \Pi$ and the corresponding value link function $g^{\pi} \in \Gcal$ . 
\end{enumerate} We define $d$ as the PO-bilinear rank. 
\end{definition}

{ While the above definition is enough to capture most of the examples we discuss later in this work, including undercomplete tabular POMDPs, LQG, HSE-POMDPs, we provide two useful extensions. The first extension incorporates discriminators into the framework, which can be used to capture the M-step decodable POMDPs and POMDPs with low-rank latent transition (see \pref{sec:general}). The second extension incorporates multi-step futures, which can be used to capture overcomplete POMDPs and general PSRs. In the next section, we introduce the multi-step future version. 
}

{\subsection{PO-bilinear Actor-critic Class with Multi Step Future}\label{subsec:bilinear_multi_step}

In this section, we provide an extension to \pref{def:simple_bilinear} to incorporate multiple-step futures (i.e., $K > 1$). For simplicity, we assume that $\pi^{out}=\Ucal(\Acal)$. 

The definition is then as follows. The main difference is that we roll out a policy $\Ucal(\Acal)$, $K-1$ times to incorporate multi-step link functions. 
We introduce the notation 
\begin{align*}
     (z_{h-1}, o_{h:h+K-1},a_{h:h+K-2})=\bar z^K_{h} \in \bar \Zcal^K_{h}=\Zcal_{h-1}\times \Ocal^{K}\times \Acal^{K-1}.  
\end{align*}
Then, combining the Bellman equation for state-value functions and the definition of K-step link functions, we have
\begin{align*}
    0  &= \EE[V^{\pi}_{h+1}(z_{h},s_{h+1})+ r_h - V^{\pi}_{h}(z_{h-1},s_{h})\mid z_{h-1}, s_h;  a_h\sim \pi  ] \\ 
    &= \EE[g^{\pi}_{h+1}(\bar z^K_{h+1}) \mid z_{h-1}, s_h ; a_h\sim \pi,a_{h+1:h+K-1}\sim  \Ucal(\Acal)  ] + \EE[r_h \mid z_{h-1},s_h ; a_h\sim \pi] \\
    & -  \EE[g^{\pi}_{h}(\bar z^K_{h}) \mid z_{h-1}, s_h ; a_{h:h+K-2}\sim  \Ucal(\Acal)  ] 
\end{align*} 
Thus, by taking expectations further with respect to $(z_{h-1},s_h)$ (i.e., $z_{h-1},s_h$ can be sampled from some roll-in policy), we have
\begin{align*}
    0 &= \EE[g^{\pi}_{h+1}(\bar z^K_{h+1})  ; a_{1:h-1}\sim \pi'
    , a_h\sim \pi,a_{h+1:h+K-1}\sim  \Ucal(\Acal)  ] + \EE[r_h  ;a_{1:h-1}\sim \pi', a_h\sim \pi] \\
    & -  \EE[g^{\pi}_{h}(\bar z^K_{h})  ;a_{1:h-1}\sim \pi', a_{h:h+K-2}\sim  \Ucal(\Acal)  ]. 
\end{align*} %
Hence, the Bellman loss of a pair $(\pi,g)$ under a roll-in $\pi'$ denoted by $\mathrm{Br}_h(\pi,g;\pi')$ at $h\in [H]$ is defined as 
\begin{align*}
     \mathrm{Br}_h(\pi,g;\pi') &= \EE[g_{h+1}(\bar z^K_{h+1})  ; a_{1:h-1}\sim \pi'
    , a_h\sim \pi,a_{h+1:h+K-1}\sim  \Ucal(\Acal)  ] + \EE[r_h  ;a_{1:h-1}\sim \pi', a_h\sim \pi] \\
    & -  \EE[g_{h}(\bar z^K_{h})  ;a_{1:h-1}\sim \pi', a_{h:h+K-2}\sim  \Ucal(\Acal)  ]. 
\end{align*}
The above is a proper loss function when we use multi-step futures. Here is the structure we need for $\mathrm{Br}_h(\pi, g;\pi')$. 

\begin{definition}[PO-bilinear AC Class for POMDPs with multi-step future] \label{def:bilinear_multi_step}
The model is a PO-bilinear class of rank $d$ if $\Gcal$ is realizable (regarding general K-step link functions), and there exists $W_h:\Pi \times \Gcal \to \RR^d$ and $X_h:\Pi \to \RR^d$ such that for all $\pi',\pi \in \Pi, g \in \Gcal$ and $h\in [H]$, 
\begin{enumerate}
    \item We have:
    \begin{align*}
      &  \EE[g_{h+1}(\bar z^K_{h+1})  ; a_{1:h-1}\sim \pi' 
    , a_h\sim \pi,a_{h+1:h+K-1}\sim  \Ucal(\Acal)  ] + \EE[r_h  ;a_{1:h-1}\sim \pi', a_h\sim \pi] \\
    & -  \EE[g_{h}(\bar z^K_{h})  ;a_{1:h-1}\sim \pi', a_{h:h+K-2}\sim  \Ucal(\Acal)  ]= \langle W_h(\pi,g),X_h(\pi') \rangle , 
    \end{align*}
     \item $ W_h(\pi,g^{\pi})=0 $ for any $\pi \in \Pi$ and the corresponding value link function $g^{\pi}$ in $\Gcal$ . 
\end{enumerate}
We define $d$ as the PO-bilinear rank.
\end{definition} 
}

%% file: main_document/main_example.tex
\section{Examples of PO-Bilinear Actor-critic Classes} 
\label{sec:example}

We consider three examples (observable tabular POMDPs, LQG, HSE-POMDPs) that admit PO-bilinear rank decomposition.  
Our framework can also capture  PSRs and $M$-step decodable POMDPs, of which the discussions are deferred to \pref{sec:psr} and \pref{sec:example_general}, respectively. 
We mainly focus on one-step future, i.e., $K=1$,  and briefly discuss extension to $K > 1$ in the the tabular case. In this section, except for LQG, we assume $r_h \in [0,1]$ for any $h\in [H]$. All the missing proofs are deferred to \pref{sec:ape_example}. 

\subsection{Observable Undercomplete Tabular POMDPs}

\begin{myexp}[Observable undercomplete tabular POMDPs]\label{ex:under_tabular}
Let $\OO \in \mathbb{R}^{|\Ocal|\times |\Scal|}$ where the entry indexed by a pair $(o,s)$ is defined as $\OO_{o,s} = \OO(o | s)$.  Assume that $\rank(\OO)=|\Scal|$, which we call observability. This requires undercompletenes $|\Ocal| \geq |\Scal|$.
\end{myexp}

The following lemma shows that $\OO$ being full rank implies the existence of value link functions. %
\begin{lemma}\label{lem:tabular_pomdp_value_bridge} 
For \pref{ex:under_tabular}, there exists a one-step value link function $g^{\pi}_h$ for any $\pi \in \Pi$ and $h \in [H]$.
\end{lemma} 
\begin{proof}
 Consider any function $f:\Zcal_{h-1}\times \Scal \to \RR$ (thus, this captures all possible $V^\pi_h$). Denote $\one(z)$ as the one-hot encoding of $z$ over $\Zcal_{h-1}$ (similarly for $\one(s)$). We have $f(z, s) = \langle f, \one(z)\otimes \one(s) \rangle = \langle f ,  \one(z) \otimes (\OO^\dagger \OO \one(s) ) \rangle $, where we use the assumption that $\rank(\OO)=|\Scal|$ and thus $\OO^\dagger \OO = I$.
Then,  
\begin{align}\label{eq:key_tabular}
    f(z, s) = \langle f , \one(z) \otimes ( \OO^\dagger \EE_{o\sim O( s)}  \one(o)) \rangle = \EE_{o\sim O( s)} \langle f , \one(z) \otimes  \OO^\dagger   \one(o) \rangle,
\end{align}
which means that the value link function corresponding to $f$ is $ g(z, o) := \langle f, \one(z) \otimes \OO^\dagger \one(o) \rangle$. %
\end{proof}

We next show that the PO-Bilinear rank (\pref{def:simple_bilinear}) of tabular POMDPs is bounded by  $|\Scal|(|\Ocal||\Acal|)^M$. 
\begin{lemma}\label{lem:tabular_pomdp_bilinear}
Assume $\OO$ is full column rank. Set the value link function class $\Gcal_h = [\Zcal_{h-1} \times  \Ocal \to [0,C_{\Gcal}]]$ for certain $C_{\Gcal}\in \mathbb{R}$, and policy class $\Pi_h = [\bar \Zcal_h \to \Delta(\Acal)]$. Then, the model is a PO-biliner AC class (\pref{def:simple_bilinear}) with PO-bilinear rank at most $|\Scal|(|\Ocal||\Acal|)^M$.
\end{lemma}

{ Later, we will see that the PO-bilinear rank in the more general definition is just $|\Scal|$ in \pref{sec:general}. This fact will result a significant improvement in terms of the sample complexity, and will result in a sample complexity that does not incur $|\Ocal|^M$. }

{
\subsection{Observable Overcomplete Tabular POMDPs }\label{subsec:overcomplete_pomdps_tabular}

We consider overcomplete POMDPs with multi-step futures. The proofs are deferred to Section \ref{subsec:over_complete_ape}. We have the following theorem. This is a generalization of \pref{lem:tabular_pomdp_value_bridge}, i.e., when $K=1$, it is \pref{lem:tabular_pomdp_value_bridge}.

\begin{lemma}\label{lem:over_complete_bridge}
Define a $|\Tcal^K | \times |\Scal|$-dimensional matrix $\OO^K$ whose entry indexed by $(o_{h:h+K-1},a_{h:h+K-2}) \in \Tcal^K$ and $s_h \in \Scal$ is equal to  $\PP(o_{h:h+K-1},a_{h:h+K-2} \mid s_h;a_{h:h+K-2}\sim \Ucal(\Acal) )$. When this matrix is full-column rank, K-step link functions with respect to $\Ucal(\Acal)$  exist. 
\end{lemma}

Note a sufficient condition to satisfy the above is that a matrix $\OO^K(a'_{h:h+K-2})\in \RR^{|\Ocal|^K \times |\Scal|}$ whose entry indexed by $o_{h:h+K-1} \in \Ocal^K$ and $s_h \in \Scal$ is equal to $\PP(o_{h:h+K-1}\mid s_h; a_{h:h+K-2}=a'_{h:h+K-2})$ is full-column rank for certain $a'_{h:h+K-2}\in \Acal^{K-1}$. It says there is (unknown) action sequence with length $K$ that retains information about latent states. 

We next calculate the PO-bilinear rank. Importantly, this does \emph{not} depend on $|\Acal|^K$ and $|\Ocal|^K$.  

\begin{lemma}\label{lem:multi_step_bilinear}
Set a value link function class $\Gcal_h = [\bar \Zcal^K \to [0,C_{\Gcal}]]$ for certain $C_{\Gcal}\in \mathbb{R}^+$ and a policy class $\Pi_h =[\bar \Zcal_h \to \Delta(\Acal) ]$. Then, the model satisfies PO-bilinear rank condition with PO-bilinear rank (\pref{def:bilinear_multi_step}) at most $|\Scal|(|\Ocal| |\Acal|)^M$. 
\end{lemma}

Note that the bilinear rank is still $|\Scal|(|\Ocal||\Acal|)^M$  (just $|\Scal|$ in the more general definition in \pref{sec:general}). Crucially, it does not depend on the length of futures $K$. 
}

\subsection{Observable Linear Quadratic Gaussian }\label{sec:lqg}

The next example is Linear Quadratic Gaussian (LQG) with continuous state and action spaces. The details are deferred to \pref{sec:sample_complexity_lqg}. Here, we set $M = H-1$ so that the policy class $\Pi$ contains the globally optimal policy.

\begin{myexp}[Linear Quadratic Gaussian (LQG)] \label{ex:lqqs}
Consider LQG:
\begin{align*}
    s' =A s + B a + \epsilon,~o = C s+ \tau,~ r= -( s^{\top}Q s + a^{\top}R a) 
\end{align*}
    where $\epsilon,\tau$ are Gaussian distribution with mean $0$ and variances $\Sigma_{\epsilon}$ and $\Sigma_{\tau}$, respectively, and $s \in \mathbb{R}^{d_s}, o\in \mathbb{R}^{d_o}$, and $a\in \mathbb{R}^{d_a}$, and $Q,R$ are positive definite matrices. %
\end{myexp}
 
We define the policy class as the linear policy class $\Pi_h=\{\delta(a_h= K_h  \bar z_{h}) \mid{} K_h \in \RR^{|\Acal|\times d_{\bar z_h}} )$,  where $d_{\bar z_h}$ is a dimension of $\bar z_h \in \bar \Zcal_h$.  %
This choice is natural since the globally optimal policy is known to be linear with respect to the entire history \citep[Chapter 4]{bertsekas2012dynamic}. We define two quadratic features, $\phi_h( z_{h-1}, s_h ) = (1,  [z^{\top}_{h-1},s^{\top}_h]\otimes [z^{\top}_{h-1},s^{\top}_h] )^{\top}$ with $z_{h-1}\in \Zcal_{h-1},s_h\in\Scal$, and $\psi_h(z_{h-1}, o_h) = (1, [z^{\top}_{h-1},o^{\top}_h]\otimes [z^{\top}_{h-1},o^{\top}_h])^{\top}$ with $z_{h-1}\in \Zcal_{h-1},o_h \in \Ocal$. We have the following lemma.

 \begin{lemma}[PO-bilinear rank of observable LQG]\label{lem:lqg_value}
 Assume $\rank(C)= d_s$. Then, the following holds: 
\begin{itemize}
    \item For any policy $\pi$ linear in $\bar z_h$, a one-step value link function $g^{\pi}_h(\cdot)$ exists, and is linear in $\psi_h(\cdot)$.  
    \item Letting $d_{\psi_h}$ be the dimension of $\psi_h$, we set $\Gcal_h = \{ \theta^\top \psi_h(\cdot) | \theta\in\mathbb{R}^{ d_{\psi_h} }  \}$ and $\Pi$ being linear in $\bar z_h$. Then LQG satisfies \pref{def:simple_bilinear} with PO-bilinear rank at most $O(\{1+ (H-1)(d_o + d_a) + d_s \}^2)$
\end{itemize}
\end{lemma} 

We have two remarks. First, when $\pi^{out}_t=\delta(a=0)$, K-step link functions exist when  $[ C^\top, (C A)^\top, \dots,  (C A^{K-1})^\top ]$ is full raw rank. %
This assumption is referred to as observability in control theory \citep{hespanha2018linear}.
Secondly, the PO-bilinear rank scales polynomially with respect to $H,d_o,d_a,d_s$ even with $M= H-1$. As we show in \pref{sec:sample_complexity_lqg}, due to this fact, we can compete against the \emph{globally} optimal policy with polynomial sample complexity. \looseness=-1

\subsection{Observable Hilbert Space Embedding POMDPs}

We consider HSE-POMDPs that generalize tabular POMDPs and LQG. Proofs here are deferred to Section~ \ref{subsec:hse_pomdps_ape}.
Consider any $h \in [H]$. Given a policy $\pi_h: \bar \Zcal_h \to \Acal $, we define the induced transition operator $\TT_{\pi}=\{\TT_{\pi;h}\}_{h=1}^H$ as $ (z_{h}, s_{h+1}) \sim \TT_{\pi; h}(  z_{h-1}, s_h )$, where we have $o_h \sim \OO( s_h), a_h \sim \pi_h( \bar z_h ), s_{h+1} \sim \TT( s_h, a_h )$.  
Namely, $\TT_{\pi}$ is the transition kernel of some Markov chain induced by the policy $\pi$. The HSE-POMDP assumes two conditional distributions $\OO(\cdot | s)$ and $\TT_{\pi}( \cdot, \cdot |  z, s)$ have conditional mean embeddings.  %

\begin{myexp}[HSE-POMDPs]
\label{ex:linear}
We introduce features $\phi_{h}:\Zcal_{h-1} \times \Scal \to \RR^{d_{\phi_h}}, \psi_h:\Zcal_{h-1} \times \Ocal \to \RR^{d_{\psi_h}}$. We assume the existence of the conditional mean embedding operators: (1) there exists a matrix $K_h$ such that for all $z \in \Zcal_{h-1}, s\in \Scal$, $\EE_{o\sim \OO(\cdot | s)} \psi_h(z, o) = K_h \phi_h(z, s)$ and (2) for all $\pi\in \Pi$, there exists a matrix $T_{\pi;h}$, such that $\EE_{z_{h}, s_{h+1} \sim \TT_{\pi;h}(z_{h-1}, s_h)} \phi_{h+1}(  z_h, s_{h+1} ) = T_{\pi;h} \phi_h(z_{h-1}, s_h) $. 
\end{myexp}

The existence of conditional mean embedding is a common assumption in prior RL works on learning dynamics of  HMMs, PSRs, \citep{song2009hilbert,boots2013hilbert} and Bellman complete linear MDPs \citep{zanette2020learning,duan2020minimax,chowdhury2020no,hao2021sparse}. %
HSE-POMDPs naturally capture tabular POMDPs and LQG. For tabular POMDPs, $\psi_h$ and $\phi_h$ are one-hot encoding features. In LQG, $\phi_h$ and $\psi_h$ are quadratic features we define in \pref{sec:lqg}. Here for simplicity, we focus on finite-dimensional features $\phi_h$ and $\psi_h$. Extension to infinite-dimensional Reproducing kernel Hilbert Space is deferred to Section~ \ref{subsec:hse_pomdps_ape}.

The following shows the existence of value link functions and the PO-bilinear rank decomposition. 
\begin{lemma}[PO-bilinear rank of observable HSE-POMDPs] \label{lem:hse_existence}
Assume $K_h$ is full column rank (observability), and  $V_h^{\pi}(\cdot)$ is linear in $\phi_{h}$ for any $\pi \in \Pi,h \in [H]$. Then the following holds. 
\begin{itemize}
    \item A one-step value link function $g^{\pi}_h(\cdot)$ exists for any $\pi \in \Pi,h \in [H]$, and is linear in $\psi_h$.  
    \item We set a value function class $\Gcal_h = \{ w^{\top} \psi_h(\cdot) | w\in \mathbb{R}^{d_{\psi_h}}\}$, policy class $\Pi_h \subset [\bar\Zcal_h \to \Delta(\Acal)]$. Then HSE-POMDP satisfies \pref{def:simple_bilinear} with PO-bilinear rank at most $\max_{h\in [H]} d_{\phi_h}$.
\end{itemize}

\end{lemma}

The first statement can be verified by noting that when $V^{\pi}_h(\cdot)=\langle \theta_h , \phi_h(\cdot) \rangle $, value link functions take the following form  $g_h^{\pi}(\cdot)=\langle ( K^{\dagger}_h )^{\top} \theta_h), \psi_h(\cdot) \rangle$ where we leverage the existence of the conditional mean embedding operator $K_h$, and that $K_h$ is full column rank (thus $K_h^\dagger K_h = \II_{d_{\phi_h}}$). Note that the PO-bilinear rank depends only on the dimension of the features $\phi_h$ without any explicit dependence on the length of memory.

%% file: main_document/main_algorithm.tex
\section{Algorithm and Complexity}\label{sec:algorithm}
In this section, we first give our algorithm followed by a general sample complexity analysis. We then instantiate our analysis to specific models considered in \pref{sec:example}.

\subsection{Algorithm}

\begin{algorithm}[!t] 
\caption{PaRtially ObserVAble BiLinEar (\ouralg) {\textcolor{blue}{\# multi-step version is in \pref{alg:OACP_multi_multi}}} } \label{alg:OACP}
\begin{algorithmic}[1]
  \STATE {\bf  Input:}  Value class $\Gcal=\{\Gcal_h\}, \Gcal_h \subset [\Zcal_{h-1}\to \RR]$, Policy class $\Pi = \{\Pi_h\}, \Pi_h \subset [\bar \Zcal_{h-1} \to \RR]$, parameters $m \in \mathbf{N}, R \in \RR$
  \STATE Initialize $\pi^0 \in \Pi$
  \STATE Form the first step dataset $\Dcal^0 = \{ o^i \}_{i=1}^m$, with $o^i \sim  \OO(\cdot | s_1)$
  \FOR{$t = 0 \to T-1$}
	\STATE For any $h\in [H]$, collect $m$ i.i.d tuple as follows: $(\bar z_h,s_h)\sim d^{\pi^t}_h, a_h \sim \Ucal(\Acal), r_h = r_h(o_h,a_h) , s_{h+1}  \sim \TT(s_h,a_h), o_{h+1} \sim \OO(\cdot | s_{h+1}) $.   \label{line:data_collect}
	\STATE Define $\Dcal^t_h = \{(\bar z^i_h, a^i_h, r^i_h,o^i_{h+1} ) \}_{i=1}^m$  \hfill  {\textcolor{blue}{\# note latent state $s$ is not in the dataset}}
	\STATE Define the Bellman error $\forall (\pi,g) \in \Pi \times \Gcal$, $$\textstyle \sigma^t_h(\pi, g) :=  \EE_{\Dcal^t_h}\left[\pi_h(a_h \mid \bar z_h)|\Acal | \{g_{h+1}(\bar z_{h+1}) + r_h\} -g_h(\bar z_h) \right].  \label{line:loss}$$
	\STATE Select policy optimistically as follows \begin{align*}\textstyle
	(\pi^{t+1}, g^{t+1}) := \argmax_{\pi \in \Pi, g \in\Gcal }  \EE_{\Dcal^0}[g_1(o)] \quad \mathrm{s.t.}\quad \forall h \in [H], \forall i \in [t], (\sigma^i_h( \pi, g ))^2 \leq R.   
	\end{align*}
  \ENDFOR
  \STATE  {\bf  Output:} Randomly choose $\hat \pi$ from $(\pi_1,\cdots,\pi_{T})$. 
\end{algorithmic}
\end{algorithm}

\noindent
{We first focus on the cases where models satisfy the PO-bilinear AC model (i.e., \pref{def:simple_bilinear}) with finite action and with one-step link function.  
We discuss the extension to handle continuous action in \pref{rem:extension} and multi-step link functions at the end of this subsection. } 

We present our algorithm \ouralg{} in \pref{alg:OACP}. Note \ouralg{} is agnostic to  the form of $X_h$ and $W_h$. 
Inside iteration $t$, given the latest learned policy $\pi^t$, we define Bellman error for all pairs $(\pi, g)$ where the Bellman error is averaged over the samples from $\pi^t$. {Here, to evaluate the Bellman loss for any policy $\pi \in \Pi$, we use importance sampling by running $\Ucal(\Acal)$ rather than executing a policy $\pi$ so that we can reuse samples.\footnote{This choice might limit the algorithm to the case where $\Acal$ is discrete. However, for examples such as LQG, we show that we can replace $\Ucal(\Acal)$ by a G-optimal design over the quadratic polynomial feature of the actions.} }
A pair $(\pi,g)$ that has a small total Bellman error intuitively means that given the data so far, $g$ could still be a value link function for the policy $\pi$. Then in the constrained optimization formulation, we only focus on $(\pi,g)$ pairs whose  Bellman errors are small so far. Among these $(\pi, g)$ pairs, we select the pair using the principle of optimism in the face of uncertainty. We remark the algorithm leverages some design choices from the Bilinear-UCB algorithm for MDPs \citep{du2021bilinear}. The key difference between our algorithm and the Bilinear-UCB is that we leverage the actor-critic framework equipped with value link functions to handle partially observability and agnostic learning.

\paragraph{With multi-step link functions.}

Finally, we consider the case with multi-step futures in \pref{alg:OACP_multi_multi} when $\pi^{out}=\Ucal(\Acal)$.  Recall the notation $\bar z^K_{h} = (z_{h-1}, o_{h:h+K-1},a_{h:h+K-2})$. 
The only difference is in the process of data collection. Particularly, at every iteration $t$, we roll-in using $\pi^t$ to (and include) time step $h-1$, we then roll-out by switching to $\Ucal(\Acal)$ for $K$ steps.

\begin{algorithm}[!t] 
\caption{PaRtially ObserVAble BiLinEar (\ouralg)  {\textcolor{blue}{\# multi-step version  }}}  \label{alg:OACP_multi_multi}
\begin{algorithmic}[1]
  \STATE {\bf  Input:} Value link function class $\Gcal =\{\Gcal_h\}, \Gcal_h \subset [\bar Z^K_h \to \RR]$, Policy class $\Pi = \{\Pi_h\}, \Pi_h \subset [\bar \Zcal_h \to \Delta(\Acal)]$, parameters $m \in \mathbb{N}, R \in \RR$
  \STATE Define  
  { 
  \begin{align*}
      l_h(\bar z^K_{h},a_{h+K-1}, r_h,o_{h+K} ; \pi ,g) := |\Acal|  \pi_h(a_h\mid \bar z_h) \left( g_{h+1}(\bar z^K_{h+1}) +  r_h \right) -  g_{h}(\bar z^K_{h}).  
  \end{align*}
  }
  \STATE Initialize $\pi^0 \in \Pi$
  \STATE Form the first step dataset $\Dcal^0 = \{ \bar z^{K;i}_1 \}_{i=1}^m$ where each $\bar z^K$ is generated by following $a_{1:K-1}\sim U(\Acal)$ in an i.i.d manner. 
  \FOR{$t = 0 \to T-1$}
   \STATE   
   for any $h \in [H]$, define the Bellman error 
   \begin{align*}
      \sigma^t_h(\pi,g)= \EE_{\Dcal^t_h }[l_h( \bar z^K_{h},a_{h+K-1}, r_h,o_{h+K}  ; \pi ,g) ] 
   \end{align*}
   where $\Dcal^t_h$ means empirical approximation by executing  $a_{1:h-1}\sim \pi^t,a_{h:h+K-1} \sim \Ucal(\Acal)$ and collecting $m$ i.i.d tuples.  
   
	\STATE Select policy optimistically as follows (here note $g=\{g_h\}_{h=1}^H$) \begin{align*}
		&(\pi^{t+1}, g^{t+1}) := \argmax_{\pi \in \Pi, g \in\Gcal }  \EE_{\Dcal^0}[g_1(\bar z^K_1)] \quad \mathrm{s.t.}\quad \forall h \in [H], \forall i \in [t],\,\sigma^i_h( \pi, g )^2 \leq R.   
	\end{align*}
  \ENDFOR
  \STATE  {\bf  Output:} Randomly choose $\hat \pi$  from $(\pi_1,\cdots,\pi_{T})$. 
\end{algorithmic}
\end{algorithm}

\begin{remark}[Continuous control]\label{rem:extension}
Algorithms so far implicitly assume the action is finite. However, we can consider LQG, which has continuous action. By employing a G-optimal design over actions, our algorithm can handle the continuous action. The discussion is deferred to \pref{sec:ape_algo}.  \
\end{remark}

\subsection{Sample Complexity}
We show a sample complexity result by using reduction to supervised learning analysis. We begin by stating the following assumption which is ensured by standard uniform convergence results. %
\begin{assum}[Uniform Convergence]\label{assum:uniform}
Fix $h \in [H]$. Let $\Dcal'_h$ be a set of $m$ i.i.d tuples following $(z_{h-1},s_h,o_h)\sim d^{\pi^t}_h, a_h \sim \Ucal(\Acal),s_{h+1}  \sim \TT(s_h,a_h), o_{h+1} \sim \OO(s_{h+1}) $.   With probability $1-\delta$,  
\begin{align*}\textstyle
  \sup_{\pi \in \Pi, g \in \Gcal}|(\EE_{\Dcal'_h} -\EE)[\pi_h(a_h \mid \bar z_h)|\Acal | \{g_{h+1}(\bar z_{h+1}) + r_h\} -g_h(\bar z_h) ]|\leq \epsilon_{gen,h}(m,\Pi,\Gcal,\delta)
\end{align*}
For $h=1$, we also require $$\sup_{g_1 \in \Gcal_1}|\EE_{\Dcal'_1}[g_1(o_1)]-\EE[\EE_{\Dcal'_1}[g_1(o_1)] ]  |\leq \epsilon_{ini,1}(m,\Gcal,\delta).$$
\end{assum} 

\begin{remark}[Finite function classes] \label{re:uniform_converge}The term $\epsilon_{gen}$ depends on the statistical complexities of the function classes $\Pi$ and $\Gcal$. 
As a simple example, we consider the case where $\Pi$ and $\Gcal$ are discrete. In this case, we have $\epsilon_{gen,h}(m,\Pi,\Gcal,\delta) = O( \sqrt{   \ln( |\Pi| |\Gcal| /\delta )  / m  }  )$, and $\epsilon_{ini,1}(m,\Gcal,\delta) = O( \sqrt{ \ln(|\Gcal|/\delta) / m }  )$, which are standard statistical complexities for discrete function classes $\Pi$ and $\Gcal$. Achieving this result simply requires standard concentration and a union bound over all functions in $\Pi,\Gcal$.
\end{remark}

Under Assumption \pref{assum:uniform}, when the model is PO-bilinear with rank $d$, we get the following. 

\begin{theorem}[PAC guarantee of \ouralg] \label{thm:online}
Suppose we have a PO-bilinear AC class with rank $d$. Suppose Assumption \pref{assum:uniform}, $\sup_{\pi \in \Pi}\|X_h(\pi)\| \leq B_X$ and $\sup_{\pi \in \Pi,g\in \Gcal}\|W_h(\pi,g)\| \leq B_W$ for any $h \in [H].$ \\
By setting
$ T=2Hd \ln\left(4Hd \left (\frac{B^2_XB^2_W}{\tilde \epsilon^2_{gen} } +1 \right) \right),R= \epsilon^2_{gen}$ where
\begin{align*}
  & \textstyle \epsilon_{gen}:=\max_h \epsilon_{gen,h}(m,\Pi, \Gcal,\delta/(TH+1)),\tilde \epsilon_{gen}:=\max_h \epsilon_{gen,h}(m,\Pi, \Gcal,\delta/H). 
\end{align*}
With probability at least $1-\delta$, letting $\pi^{\star}=\argmax_{\pi \in \Pi} J(\pi^{\star})$, we have 
\begin{align*} \textstyle
     J(\pi^{\star})- J(\hat \pi)\leq  5\epsilon_{gen}\sqrt{d H^2 \cdot \ln \left(4Hd \left( {B^2_XB^2_W}/{\tilde \epsilon^2_{gen}}+1\right)\right)} + 2\epsilon_{ini,1}(m,\Gcal,\delta/(TH+1)). 
\end{align*}
The total number of samples used in the algorithm is $mTH$. 
\end{theorem}

Informally, when $\epsilon_{gen} \approx \tilde{O}(1/\sqrt{m})$, to achieve $\epsilon$-near optimality, the above theorem indicates that we just need to set $m \approx \tilde{O}(1/\epsilon^2)$, which results a sample complexity scaling $\tilde{O}(1/\epsilon^2)$ (since $T$ only scales $\tilde{O}(dH)$). We give detailed derivation and examples in the next section.

\subsection{Examples}

Hereafter, we show the sample complexity result by using \pref{thm:online}. For complete results, refer to \pref{sec:sample_complexity_finite}--\ref{sec:sample_complexity_psrs}. The result of $M$-step decodable POMDPs and observable low-rank POMDPs are deferred to Section \ref{sec:m_step_decodable_sample}.

\subsubsection{Finite Sample Classes}

We first consider the case where the hypothesis class is finite when the class admits PO-bilinear rank decomposition.  

\begin{myexp}[Finite Sample Classes]
Consider the case when $\Pi$ and $\Gcal$ are finite and the PO-bilinear rank assumption is satisfied. When $\Pi$ and $\Gcal$ are infinite hypothesis classes, $|\Fcal|$ and $|\Gcal|$ are replaced with their $L^{\infty}$-covering numbers, respectively. 

\begin{theorem}[Sample complexity for discrete $\Pi$ and $\Gcal$ (informal)]
Let $\|\Gcal_h\|_{\infty} \leq C_{\Gcal}, r_h \in [0,1]$ for any $h \in [H]$ and the PO-bilinear rank assumption holds with PO-biliear rank $d$. By letting $|\Pi_{\max}|=\max_h |\Pi_h|, |\Gcal_{\max}| = \max_h |\Gcal_h|$, with probability $1-\delta$, we can achieve $J(\pi^{\star})-J(\hat \pi)\leq \epsilon$ when we use samples 
\begin{align*}
    \tilde O\prns{ d^2 H^4 \max(C_{\Gcal},1)^2 |\Acal|^2 \ln(|\Gcal_{\max}| |\Pi_{\max}|  /\delta)\ln^2(B_XB_W/\delta)(1/\epsilon)^2  }. 
\end{align*}
Here, $\mathrm{Polylog}(d,H,|\Acal|,\ln(|\Gcal|),\ln(|\Pi|),\ln(1/\delta),\ln(B_X),\ln(B_W),\ln(1/\delta),1/\epsilon)$ are omitted. 
\end{theorem}
\end{myexp}

\subsubsection{Observable Undercomplete Tabular POMDPs}\label{subsec:observable_tabular_ab}

We start with tabular POMDPs. The details here is deferred to \pref{sec:sample_complexity_tabular}. 

\begin{myexp}[continues=ex:under_tabular]
{In tabular models, recall the PO-bilinear rank is at most $d = |\Ocal|^{M}|\Acal|^{M}|\Scal|$. We suppose $r_h \in [0,1]$ for any $h \in [H]$. Assuming $\OO$ is full-column rank, to satisfy the realizability, we set $\Gcal_h = \{\langle \theta, \One(z)\otimes \OO^{\dagger}\One(o) \rangle \mid \|\theta\|_{\infty} \leq H \}$ where $\|\OO^{\dagger} \|_1 \leq 1/\sigma_1$ and $\One(z),\One(o)$ are one-hot encoding vectors over $\Zcal_{h-1}$ and $\Ocal$, respectively. We set $\Pi_h=[\bar Z_h \to \Delta(\Acal)]$. Then, the following holds. 
}
\begin{theorem}[Sample complexity for unrercomplete tabular models (Informal)] 
With probability $1-\delta$, we can achieve $J(\pi^{\star})-J(\hat \pi)\leq \epsilon$ when we use samples  at most 
$ \tilde O\prns{ |\Scal|^2 |\Acal|^{3M+3}|\Ocal|^{3M+1} H^6 (1/\epsilon)^2(1/\sigma_1)^2 \ln(1/\delta) }. $ \\
Here, $\mathrm{polylog}( |\Scal|,|\Ocal|, |\Acal|, H, 1/\sigma_1, \ln(1/\delta))$ are omitted. 
\end{theorem} 

{Firstly, while the above error incurs $|\Ocal|^M |\Acal|^M$, we will later see in \pref{sec:undercomplete_pomdp_finite} when we use the more general definition of PO-bilinear AC class and combine a model-based perspective, we can remove $|\Ocal|^M$ from the error bound. The intuition here is that the statistical complexity still scales with $|\Scal|^2 |\Acal| +|\Ocal| |\Acal|$ and does not incur $|\Ocal|^M$. At the same time, although PO-bilinear rank currently scales with $|\Ocal|^M||\Acal|^M|\Scal|$, we can show that it can be just $|\Scal|$ with a more refined definition. Secondly, $\|\OO^{\dagger} \|_1 \leq 1/\sigma_1$ can be replaced with other analogous conditions $\|\OO^{\dagger} \|_2 \leq 1/\sigma_2$. Here, note $ \|\OO^{\dagger} \|_1  = 1/\{\min_{x} \|\OO x\|_1/\|x\|_1\},\|\OO^{\dagger} \|_2  = 1/\{\min_{x} \|\OO x\|_2/\|x\|_2\}$. The reason why we use $1$-norm is to invoke the result \citep{golowich2022planning} to achieve the near global optimality as in the next paragraph. }

\paragraph{Near global optimality.}

Finally, we consider the PAC guarantee against the globally optimal policy. As shown in \citep{golowich2022planning}, it is enough to set $M= O((1/\sigma^4_1) \ln (SH/\epsilon))$ to compete with the globally optimal policy $\pi^{\star}_{\mathrm{gl}}$. Thus we achieve a quasipolynomial sample complexity when competing against $\pi^\star_{\mathrm{gl}}$.

\begin{theorem}[Sample complexity for undercomplete tabular models (Informal) --- competing against $\pi^\star_{\mathrm{gl}}$]
With probability $1-\delta$, we can achieve $J(\pi^{\star}_{\mathrm{gl}})-J(\hat \pi)\leq \epsilon$ when we use samples  at most 
\begin{align*}\textstyle
      \mathrm{poly}(|\Scal|, |\Acal|^{\ln (|\Scal| H /\epsilon)/\sigma^4_1 }, |\Ocal|^{\ln (|\Scal| H /\epsilon)/\sigma^4_1 },H,1/\sigma_1,1/\epsilon,\ln(1/\delta) ).
\end{align*}
\end{theorem}

\end{myexp}

\subsubsection{Observable Tabular Overcomplete POMDPs}

We consider obvercomplete tabular POMDPs. In this case, the PO-bilinear rank is at most $|\Ocal|^M |\Acal|^M |\Scal|$. We suppose $r_h \in [0,1]$ for any $h\in [H]$. Assuming $\OO^K$ is full-column rank, to satisfy the realizability, we set $\Gcal_h = \{\langle \theta, \One(z) \otimes \{\OO^K\}^{\dagger}\One(t^K) \rangle \mid \|\theta\|_{\infty}\leq H \}$ where $\|\OO^K\|_1\leq 1/\sigma_1$ and $\One(z),\One(t^K)$ are one-hot encoding vectors over $\Zcal_{h-1}$ and $\Ocal^K \times \Acal^{K-1}$, respectively. We set $\Pi_h = [\bar \Zcal_h \to \Delta(\Acal)]$. Then, the following holds. 

\begin{theorem}[Sample complexity for overcomplete tabular models]
With probability $1-\delta$, we can achieve $J(\pi^{\star})-J(\hat \pi)\leq \epsilon $
when we use samples  at most 
$ \tilde O\prns{ |\Scal|^2 |\Acal|^{3M+K+2}|\Ocal|^{3M+K} H^6 (1/\epsilon)^2(1/\sigma_1)^2 \ln(1/\delta) }. $ \\
Here, $\mathrm{polylog}( |\Scal|,|\Ocal|, |\Acal|, H, 1/\sigma_1, \ln(1/\delta))$ are omitted. 
\end{theorem}

When we use K-step futures, in the above theorem, we additionally incur $|\Acal|^K|\Ocal|^K$, which is coming from a naive parameterization of $\Gcal_h$. In \pref{sec:overcomplete_pomdp_finite}, we will see that under the model-based learning perspective (i.e., we parameterize $\TT,\OO$ first and then construct $\Pi$ and $\Gcal$ using the model class), we will get rid of the dependence $|\Ocal|^{M+K}$ and $|\Acal|^K$. This is because the complexity of the model class is independent of $M$ or $K$ (i.e., number of parameters in $\TT,\OO$ are $O(|\Scal|^2 |\Acal| |\Ocal|)$).

\subsubsection{Observable LQG}

Now let us revisit LQG. The detail here is deferred to \pref{sec:sample_complexity_lqg}. 
We show that \ouralg{} can compete against the globally optimal policy with polynomial sample complexity.

\begin{myexp}[continues=ex:lqqs]  
In LQG, by setting $H=M-1$, we achieve a polynomial sample complexity when competing against the globally optimal policy $\pi^\star_{\mathrm{gl}}$.

\begin{theorem}[Sample complexity for LQG (informal) -- competing against $\pi^\star_{\mathrm{gl}}$] 
Consider a linear policy class $\Pi_h = \{\delta(a_h=\bar K_{h}\bar z_h) \mid \|\bar K_{h}\| \leq \Theta \}$. 
and assume $\max(\|A\|,\|B\|,\|C\|,\|Q\|,\|R\|)\leq \Theta$ and  all policies induce a stable system (we formalize in \pref{sec:sample_complexity_lqg}). With probability $1-\delta$, we can achieve $J(\pi^{\star}_{\mathrm{gl}})-J(\hat \pi)\leq \epsilon$ when we use samples  at most 
\begin{align*}\textstyle 
    \mathrm{poly}(H, d_s,d_o,d_a, \Theta, \|C^{\dagger}\|,\ln(1/\delta)  ) \times (1/\epsilon)^2. 
\end{align*}
\end{theorem}
\end{myexp}

\subsubsection{Observable HSE-POMDPs}

Next, we study HSE-POMDPs. The details here is deferred to \pref{sec:sample_hse_pomdps}. 

\begin{myexp}[continues=ex:linear]
In HSE-POMDPs, PO-bilinear rank is at most $\max_h d_{\phi_h}$. Suppose  $\|\psi_h\|\leq 1$ and $V^{\pi}_h(\cdot)=\langle \theta^{\pi}_h, \phi_h(\cdot) \rangle $ such that $\|\theta^{\pi}_h\|\leq \Theta_V$ for any $h \in [H]$. Then, to satisfies the realizability, we set $\Gcal_h = \{\langle \theta, \psi_h(\cdot) \rangle \mid \|\theta\|\leq \Theta_V/\sigma_{\min}(K) \}$ where $\sigma_{\min}(K)=\min_{h \in [H]}1/\|K^{\dagger}_h\| $. 

\begin{theorem}[Sample complexity for HSE-POMDPs (Informal)]
Let $d_{\psi}=\max_h\{d_{\psi_h}\},d_{\phi}=\max_h\{d_{\psi_h}\}, |\Pi_{\max}| =\max_h (|\Pi_h|)$. Suppose $r_h$ lies in $[0,1]$ for any $h\in [H]$. Then, with probability $1-\delta$, we can achieve $J(\pi^{\star})-J(\hat \pi)\leq \epsilon$ when we use samples 
\begin{align*}\textstyle 
       \tilde O \prns{ d^2_{\phi} H^4|\Acal|^2\max(\Theta_V,1)^2\{d_{\psi}+\ln(|\Pi_{\max}|/\delta)\} {(1/\sigma_{\min}(K))^2 }\cdot (1/\epsilon)^2 }  . 
\end{align*}
Here, $\mathrm{polylog}(d_{\phi},d_{\psi},|\Acal|,\Theta_V, \ln(|\Pi_{\max}|), 1/\sigma_{\min}(K),1/\epsilon, \ln(1/\delta),\sigma_{\max}(T),\sigma_{\max}(K) )$ are omitted and $\sigma_{\max}(K)=\max_{h \in [H]}\|K_h\| ,\sigma_{\max}(T)=\max_{\pi \in \Pi, h \in [H]}\|T_{\pi:h}\| $. 
\end{theorem}
\end{myexp}
Note that the sample complexity above does not explicitly depend on the memory length $M$, instead it only explicitly depends on the dimension of the features $\phi, \psi$. In other words, if we have a feature mapping $\psi_h$ that can map the entire history (i.e., $M = H$) to a low-dimensional vector (e.g., LQG), our algorithm can immediately compete against the global optimality $\pi^\star_{\mathrm{gl}}$.

%% file: main_document/main_psr.tex
\section{Predictive State Representations}\label{sec:psr}

In this section, we demonstrate that our definition and algorithm applies to PSRs --- models that strictly generalize POMDPs \citep{littman2001predictive,singh2004predictive}.
Below, we first briefly introduce PSRs, followed by showing that it is a PO-bilinear AC model. Throughout this section, we will focus on discrete linear PSRs. We also suppose reward at $h$ is deterministic function of { $(o_h,a_h)$ conditional on $\tau^a_{h+1}$ where $\tau^a_{h}=(o_1,a_1,\cdots, o_{h-1},a_{h-1})$. %
Given $\tau^a_h$, the dynamical system generates $o_h\sim \PP(\cdot | \tau^a_h)$.} Here we use the superscript $a$ on $\tau_h^a$ to emphasize that the $\tau^a_h$ ends with the action $a_{h-1}$. 

PSRs use the concept of \emph{test}, which is a sequence of future observations and actions, i.e., for some test $t = ( o_{h:h+W - 1}, a_{h:h+W-2})$ with length $W\in \NN^+$, we define the probability of test $t$ being successful $\PP(t | \tau^a_h)$ as
$\PP(t | \tau^a_h ) := \PP( o_{h:h+W-1}  | \tau^a_h; \text{do}(a_{h:h+W-2}))$  which is the probability of observing $o_{h:h+W-1}$ by actively executing actions $a_{h:h+W-2}$ conditioned on history $\tau^a_h$. %

We now explain one-step observable PSRs while deferring the general multi-step observable setting to \pref{sec:psrs}. A one-step observable PSR uses the observations in $\Ocal$ as tests, i.e., tests with length 1.

\begin{definition}[Core test set and linear PSRs] \label{def:linear_psrs}
A core test set $\Tcal \subset \Ocal$ contains a finite number of tests (i.e., observations from $\Ocal$). For any $h$, any history $\tau^a_h$, any future test $t_h = (o_{h:h+W-1}, a_{h:h+W-2})$ for any $W\in \NN^+$,  there exists a vector $m_{t_h} \in \mathbb{R}^{|\Tcal|}$, such that the  probability of $t_h$ succeeds conditioned on $\tau^a_h$ can be expressed as:
$
\PP(t_h | \tau^a_h) = m_{t_h}^{\top}  [ \PP( o | \tau^a_h)    ]_{o\in \Tcal}, 
$ where we denote ${\bf q}_{\tau^a_h} := [ \PP( o | \tau^a_h)    ]_{ o \in \Tcal}$ as a vector in $\RR^{|\Tcal|}$ with entries equal to $\PP( o |\tau^a_h)$ for $o\in\Tcal$. The vector ${\bf q}_{\tau^a_h}$ is called predictive state.  
\end{definition}

 A core test set $\Tcal$ that has the smallest number of tests is called a \emph{minimum core test set} denoted as $\Qcal$. PSRs are strictly more expressive than POMDPs in that all POMDPs can be embedded into PSRs whose size of the minimum core tests is at most $|\Scal|$; however, vice versa does not hold \citep{littman2001predictive}.
 For example, in observable undercomplete POMDPs (i.e., $\OO$ full column rank) , the observation set $\Ocal$ can serve as a core test set, but the minimum core test set $\Qcal$ will have size $|\Scal|$. Here, we assume we know a core test set $\Tcal$ that contains $\Qcal$; however, we are agnostic to which set is the actual $\Qcal$. In the literature on PSRs, this setting is often referred to as transform PSRs \citep{boots2011closing,rosencrantz2004learning}.

Now we define a value link function in PSRs. %
First, given an $M$-memory policy, define $\Vcal^\pi_h(\tau^a_h) = \EE[ \sum_{t=h}^H r_t | \tau^a_h; a_{h:H}\sim \pi ] $, i.e., the expected total reward under $\pi$, conditioned on the history $\tau^a_h$.  Note that our value function here depends on the entire history. 
 
\begin{definition}[General value link functions] \label{def:general_value} Consider an $M$-memory policy $\pi$.
One-step general value link functions $g^{\pi}_h:\Zcal_{h-1} \times \Tcal \to \RR$ at step $h\in [H]$ are defined as solutions to 
\begin{align}\label{eq:new_definition}
\Vcal^{\pi}_h(\tau^a_{h}) =  \EE[g^{\pi}_h(z_{h-1}, o_h )  \mid \tau^a_{h}]. 
\end{align}
\end{definition}
This definition is more general than \pref{def:simple_bilinear} since \eqref{eq:new_definition} implies \pref{eq:bridge} in POMDPs by setting $\Ocal=\Tcal$.  %
In PSRs, we can show the existence of this general value link function. 

\begin{lemma}[The existence of link functions for PSRs]
Suppose $\Tcal$ is a core test set. Then,  a one-step value link function $g^\pi_h$ always exists.  
\end{lemma}

{The high-level derivation is as follows. Using the linear PSR property, one can first show that $\Vcal^\pi_h(\tau^a_h)$ has a bilinear form $\Vcal^\pi_h(\tau^a_h) = \one(z_{h-1})^\top \mathbb{J}_h^\pi {\bf q}_{\tau^a_h}$, where $\one(z) \in \RR^{|\Zcal_{h-1}|}$ denotes the one-hot encoding vector over $\Zcal_{h-1}$,  and $\mathbb{J}_h^\pi$ is a $|\Zcal_{h-1}| \times |\Tcal|$ matrix. Then, given any $\tau^a_h$ and $o\sim \PP(\cdot|\tau^a_h)$, for some $|\Zcal_{h-1}| \times |\Tcal|$ matrix $\mathbb{J}$, we can show  $g_h(z_{h-1}, o):= \one( z_{h-1} )^\top \mathbb{J}  [ \one(t = o) ]_{t\in\Tcal} $ satisfies the above, where $ [ \one(t = o) ]_{t\in\Tcal} \in \RR^{|\Tcal|}$ is a one-hot encoding vector over $\Tcal$ and 
serves as an unbiased estimate of ${\bf q}_{\tau^a_h}$.} %

Finally, we show that PSR admits PO-bilinear rank decomposition (\pref{def:simple_bilinear}). 

\begin{lemma}
Suppose  a core test set $\Tcal$ includes a minimum core test set $\Qcal$. Set $\Pi_h=[\bar \Zcal_h \to \Delta(\Acal)]$ and $\Gcal_h = \{ (z_{h-1},o)\mapsto \one(z_{h-1})^\top \mathbb{J}  [ \one(t = o) ]_{t\in\Tcal} \mid \mathbb{J} \in \RR^{|\Zcal_{h-1}|\times|\Tcal|}  \}$, the PO-bilinear rank is at most $(|\Ocal| |\Acal| )^M|\Qcal|$. 
\end{lemma}
Then, \pref{alg:OACP} is directly applicable to PSRs. 
Note that here the PO-bilinear rank, fortunately, rank scales with $|\Qcal |$ but not $|\Tcal|$. %
 The dependence $(|\Ocal| |\Acal|)^M$ comes from the dimension of the ``feature" of memory $\one(z_{h-1})$. If one has a compact feature representation $\phi: \Zcal_{h-1} \to \mathbb{R}^d$, such that $\Vcal^\pi_h(\tau^a_h) = \phi(z_{h-1})^\top \mathbb{J}^\pi_h {\bf q}_{\tau^a_h}$ is linear with respect to feature $\phi(z_{h-1})$, then the PO-bilinear rank is $d |\Qcal|$. This implies that if one has a compact featurization of memory a priori, one can avoid exponential dependence on $M$.

{ 

\paragraph{Sample complexity.}

We finally brifely mention the sample complexity  result. The detail is deferred to \pref{sec:sample_complexity_psrs}. 
The sample complexity to satisfy $J(\pi^\star) - J(\hat\pi) \leq \epsilon$ is given as 
\begin{align*}
    \tilde O \prns{ \frac{ |\Ocal|^M |\Acal|^{M-1} |\Qcal|^2 \max(\Theta,1) H^4 |\Acal|^2 \ln(|\Gcal_{\max} | |\Pi_{\max} | /\delta) \ln(\Theta_W)^2 }{\epsilon^2} } 
\end{align*}
where $\Theta_W$ and $\Theta$ some parameters associated with PSRs.  Here, there is no explicit dependence on $|\Tcal|$.  
Note that in the worst case,  $\ln |\Gcal_{\max} |$ scales as  $O( |\Zcal_{h-1}||\Tcal| )$, and $\ln |\Pi_{\max}|$ scales as $O( |\Zcal_{h-1}||\Ocal||\Acal| )$. %

}

%% file: main_document/minmax_extension.tex
{ 
\section{Generalization of PO-Bilinear AC Class} \label{sec:general}

We extend our previous definition of PO-Bilinear AC framework. We first present an even more general framework that captures all the previous examples that we have discussed so far. We then provide two more examples that can be covered by this framework: (1) $M$-step decodable POMDPs, and (2) observable POMDPs with low-rank latent transition. Using the result in (2), we can obtain  refined results in the tabular setting compared to the result from Section \ref{subsec:observable_tabular_ab}.

The following is a general PO-Bilinear AC Class. Recall $M(h):=\max(h-M,1)$. We consider one-step future, i.e., $ K = 1$, but the extension to $K>1$ is straightforward. Comparing to \pref{def:simple_bilinear}, we introduce another class of functions termed as discriminators $\Fcal$ and the loss function $l$.

\begin{definition}[General PO-Bilinear AC Class] \label{def:bilinear_minimax}
Consider a tuple $\langle \Pi,\Gcal,l,\Pi^{e}, \Fcal \rangle $ consisting of a policy class $\Pi$, a function class $\Gcal$, a loss function $l=\{l_h\}_{h=1}^H$ where $l_h(\cdot; f, \pi,g):\Hcal_{h-1} \times \Ocal \times \Acal \times \RR \times \Ocal \to \RR $, a set of estimation policies $\Pi^{e}:= \{ \pi^{e}(\pi): \pi\in \Pi \}$ where $\pi^{e}_h(\pi) : \bar \Zcal_h \to \Delta(\Acal)$, and a discriminator class $\Fcal = \{\Fcal_h\}$ with $\Fcal_h \subset  [\Hcal_h \to \mathbb{R}]$. Consider %
a non-decreasing function
$\zeta:\RR^+ \to \RR$ with $ \zeta(0) = 0$.

The model is a PO-bilinear class of rank $d$ if $\Gcal$ is realizable, and there exist $W_h:\Pi \times \Gcal \to \RR^d$ and $X_h:\Pi \to \RR^d$ such that for all $\pi,\pi' \in \Pi, g \in \Gcal$ and $h\in [H]$, 
\begin{enumerate}[label=(\alph*)]
    \item  $|\EE[g_h(\bar z_h)-r_h-g_{h+1}(\bar z_{h+1});a_{1:h}\sim \pi ] |\leq  |\langle W_h(\pi,g),X_h(\pi) \rangle | $,  
    \item  
    \begin{align*}
    \zeta(\max_{f\in\Fcal_h } |\EE[l_h( \tau_h, a_h, r_h,o_{h+1} ;f,  \pi ,g) ;a_{1:M(h)-1}\sim \pi',a_{M(h):h}\sim \pi^{e}(\pi') ] |) \geq  |\langle W_h(\pi,g), X_h(\pi') \rangle | .
    \end{align*}%
    (In $M$-step decodable POMDPs and POMDPs with low-rank latent transition,  we set $\pi^e(\pi)=\Ucal(\Acal)$ and in the previous sections, we set $\pi^e(\pi')=\pi'$. )
     \item $$ \max_{f\in\Fcal_h }| \EE[l_h( \tau_{h}, a_h, r_h,o_{h+1} ;f,  \pi ,g^{\pi}) ;a_{1:M(h)-1}\sim \pi',a_{M(h):h}\sim \pi^{e}(\pi') ] | = 0 $$ for any $\pi \in \Pi$ and the corresponding value link function $g^{\pi}$ in $\Gcal$ . 
\end{enumerate}

\end{definition}

The first condition states the average Bellman error under $\pi$ is upper-bounded by the quantity in the bilinear form. %
The second condition states that we have a known loss function $l$ that can be used to estimate an upper bound (up to a non-decreasing transformation $\zeta$) of the value of the bilinear form. Our algorithm will use the surrogate loss $l(\cdot)$. As we will show, just being able to estimate an upper bound of the value of the bilinear form suffices for deriving a PAC algorithm. The discriminator $\Fcal$ and the non-decreasing functions $\zeta$ give us additional freedom to design the loss function. For simple examples such as tabular POMDPs and LQG, as we already see, we simply set the discriminator class $\Fcal = \emptyset$ (i.e., we do not use discriminators) and $\zeta$ being the identity mapping. %

With this definition, we slightly modify \ouralg{} to incorporate the discriminator to construct constraints. The algorithm is summarized in \pref{alg:OACP_minimax} that is named as \oursecondalg. There are two modifications: (1) when we collect data, we switch from the roll-in policy $\pi^t$ to the policy $\pi^e$ at time step $M(h)$; (2) the Bellman error constraint $\sigma^t_h$ is defined using the loss $l$ together with the discriminator class $\Fcal_h$.
\begin{algorithm}[!t] 
\caption{PaRtially ObserVAble BiLinEar with DIScriminators (\oursecondalg)   } \label{alg:OACP_minimax}
\begin{algorithmic}[1]
  \STATE {\bf  Input:} {Value link function class $\Gcal = \{\Gcal_h\}, \Gcal_h \subset [\bar \Zcal_{h} \to \RR]$, discriminator class $\Fcal = \{\Fcal_h\}, \Fcal_h \subset [\Hcal_h \to \RR]$, policy class $\Pi =\{\Pi_h\}, \Pi_h \subset [\bar \Zcal_{h} \to \Delta(\Acal)]$, parameters $m \in \mathbb{N}, R \in \RR$}
  \STATE Initialize $\pi^0 \in \Pi$
  \STATE Form the first step dataset $\Dcal^0 = \{ o^i \}_{i=1}^m$, with $o^i \sim  \OO(\cdot | s_1)$ 
  \FOR{$t = 0 \to T-1$}
	\STATE For any $h\in [H]$, define the Bellman error $$\forall (\pi,g)\in \Pi \times \Gcal:\sigma^t_h(\pi, g) :=  \max_{f \in \Fcal_h} |\EE_{\Dcal^t_h}\left[l_h(\tau_h,a_h,r_h,o_{h+1};f,\pi,g)  \right]|$$    
	where $\Dcal^t_h$ is the empirical approximation by executing $a_{1:M(h)-1} \sim \pi^t, a_{M(h):h}\sim \pi^e(\pi^t)$ and collecting $m$ i.i.d tuples.  
	\STATE Select policy optimistically as follows \begin{align*} 
		&(\pi^{t+1}, g^{t+1}) := \argmax_{\pi \in \Pi, g \in\Gcal }  \EE_{\Dcal^0}[g_1(o)] \quad \mathrm{s.t.}\quad \forall h \in [H],\forall i \in [t], \sigma^i_h( \pi, g ) \leq R.   
	\end{align*}
  \ENDFOR
  \STATE  {\bf  Output:} Randomly choose $\hat \pi$  from $(\pi_1,\cdots,\pi_{T})$. 
\end{algorithmic}
\end{algorithm}

The following theorem shows the sample complexity of \pref{alg:OACP_minimax}. For simplicity, we direct consider the case where $\Pi, \Gcal, \Fcal$ are all discrete.  

\begin{assum}[Uniform Convergence]\label{assum:uniform_dis_general}
Fix $h \in [H]$. Let $\Dcal'_h$ be a set of $m$ i.i.d tuples by executing $a_{1:M(h)-1} \sim \pi^t, a_{M(h):h}\sim \pi^e$  With probability $1-\delta$,  
\begin{align*}\textstyle
  \sup_{\pi \in \Pi, g \in \Gcal,f\in \Fcal}|(\EE_{\Dcal'_h} -\EE)[l_h(\tau_h,a_h,r_h,o_{h+1};f,\pi,g) ]|\leq \epsilon_{gen,h}(m,\Pi,\Gcal,\Fcal,\delta)
\end{align*}
For $h=1$, we also require $$\sup_{g_1 \in \Gcal_1}|\EE_{\Dcal'_1}[g_1(o_1)]-\EE[\EE_{\Dcal'_1}[g_1(o_1)] ]  |\leq \epsilon_{ini,1}(m,\Gcal,\delta).$$
\end{assum}

\begin{theorem}[Sample complexity of \pref{alg:OACP_minimax}]\label{thm:online_general}
Suppose we have a PO-bilinear AC class with rank $d$ in \pref{def:bilinear_minimax}. Suppose Assumption \pref{assum:uniform_dis_general}, $\sup_{\pi \in \Pi}\|X_h(\pi)\| \leq B_X$ and $\sup_{\pi \in \Pi,g\in \Gcal}\|W_h(\pi,g)\| \leq B_W$ for any $h \in [H].$ \\
By setting
$ T=2Hd \ln\left(4Hd \left (\frac{B^2_XB^2_W}{\zeta^2(\tilde \epsilon_{gen}) } +1 \right) \right),R= \epsilon_{gen}$ where
\begin{align*}
  & \textstyle \epsilon_{gen}:=\max_h \epsilon_{gen,h}(m,\Pi, \Gcal,\Fcal,\delta/(TH+1)),\tilde \epsilon_{gen}:=\max_h \epsilon_{gen,h}(m,\Pi, \Gcal,\Fcal, \delta/H). 
\end{align*}
With probability at least $1-\delta$, letting $\pi^{\star}=\argmax_{\pi \in \Pi} J(\pi^{\star})$, we have 
\begin{align*} \textstyle
         J(\pi^{\star})- J(\hat \pi)\leq  H^{1/2} \bracks{4 \zeta(\epsilon_{gen})^{2} +  2T \zeta(2\epsilon_{gen})^{2} Hd \ln(4Hd (B^2_XB^2_W/\zeta^2(\tilde \epsilon_{gen}) +1))  }^{1/2} +2\epsilon_{ini}. 
\end{align*}
The total number of samples used in the algorithm is $mTH$. 
\end{theorem}
}

{This reduces to \pref{thm:online} when we set $\zeta$ as an identify function and $\pi^e(\pi') =\pi'$. When $\zeta^{-1}(\cdot)$ is a strongly convex function, we can gain more refined rate results. For example, when $\zeta(x)=\sqrt{x}$, i.e., $\zeta^{-1}(x)=x^2$, with $\epsilon_{gen} = O(1/\sqrt{m})$, the above theorem implies a slow sample complexity rate $1/\epsilon^4$. However, by leverage the strong convexity of the square function $\zeta^{-1}(x) := x^2$, a refined analysis can give the fast rate $1/\epsilon^2$. We will see such two examples in the next sections. }

\section{Examples for Generalized PO-Bilinear AC Class}\label{sec:example_general}

We demonstrate that our generalized framework captures two models: (1) $M$-step decodable POMDPs, and (2) observable POMDPs with the latent low-rank transition. In this section, we assume $r_h \in [0,1]$ for any $h\in [H]$. 

\subsection{$M$-step decodable POMDPs}

The example we include here is a model that involves nonlinear function approximation but has a unique assumption on the exact identifiability of the latent states. %

\begin{myexp}[$M$-step decodable POMDPs \citep{efroni2022provable}] \label{ex:decodable}
There exists an unknown decoder $\iota_h: \bar \Zcal_{h} \to \Scal$, such that for every reachable trajectory $(s_{1:h},a_{1:h-1},o_{1:h} ) $,  we have $s_h= \iota_h(\bar z_h)$ for all $h\in [H]$.
\end{myexp}
Note that when $M = 0$, this model is reduced to the well-known Block MDP model \citep{du2019provably,misra2020kinematic,zhang2022efficient}.

\paragraph{Existence of value link functions.} From the definition, using a value function $V^{\pi}_h(z_{h-1},s_h)$ over $z_{h-1} \in \Zcal_{h-1}, s_h \in \Scal$, we can define a value link function $v^\pi_h: \Zcal_{h-1}\times \Ocal \to \RR$ as
\begin{align*}
    v^{\pi}_h(z_{h-1},o_h) = V^{\pi}_h(z_{h-1},\iota_h(\bar z_h))  
\end{align*}
since it satisfies %
\begin{align*}
    &\EE_{o_h \sim \OO(s_h)}[v^{\pi}_h(z_{h-1},o_h)\mid z_{h-1},s_h]=\EE_{o_h \sim \OO(s_h)}[V^{\pi}_h(z_{h-1},\iota_h(\bar z_h))    \mid z_{h-1},s_h ]=V^{\pi}_h(z_{h-1},s_h). 
\end{align*}
This is summarized in the following lemma. 

\begin{lemma}[Existence of link functions in $M$-step decodable POMDPs]
In $M$-step decodable POMDPs, link functions exist. 
\end{lemma}

$M$-step decodable POMDPs showcase the \emph{generality of value link functions}, which not only capture standard observability conditions where future observations and actions are used to replace belief states (e.g., observable tabular POMDPs and observable LQG), but also capture a model where history is used to replace latent states.

\paragraph{PO-Bilinear Rank.}
Next, we calculate the PO-bilinear rank based on \pref{def:bilinear_minimax}. In the tabular case, we can na\"ively obtain the PO-bilinear decomposition with rank $|\Ocal|^{M} |\Acal|^M |\Scal| $ following \pref{ex:under_tabular}. Here, we consider the nontabular case where function approximation is used and $|\Ocal|$ can be extremely large.
We define the following Bellman operator associated with $\pi$ at  step $h$: 
\begin{align}\label{eq:bellman}
    &\Bcal^{\pi}_h: \Gcal \to [\bar \Zcal_h \to \RR]   ; \\
    & \forall \bar z_h: \left[ \Bcal^\pi_h g \right](\bar z_h) := \EE_{a_h\sim \pi(\bar z_h)}\left[ r_h( \iota_h(\bar z_h), a_h ) + \EE_{o_{h+1} \sim \OO\circ \TT (\iota_h(\bar z_h), a_h)}[g_{h+1}( \bar z_{h+1} )] \right].  \nonumber
\end{align} 
Note that above we use the ground truth decoder $\iota_h$ to decode from $\bar z_h$ to its associated latent state $s_h$. The existence of this Bellman operator $\Bcal^\pi_h$ is crucially dependent on the existence of such decoder $\iota_h$.

We show that $M$-step decodable POMDPs satisfy the definition in \pref{def:bilinear_minimax}. %
{ We assume that the latent state-wise transition model is low-rank. In MDPs, this assumption is widely used in \citep{yang2020reinforcement,jin2020provably,agarwal2020flambe,uehara2021representation}. Here, we do \emph{not} need to know $\mu,\phi$ in the algorithm. 

\begin{assum}[Low-rankness of latent transition]\label{assum:low_rank}
Suppose $\TT$ is low-rank, i.e., $\TT(s'\mid s,a)=\langle \phi(s,a),\mu(s') \rangle (\forall (s,a,s')) $ where $\phi,\mu$ are (unknown) $d$-dimensional features. As technical conditions, we suppose $\|\phi(s,a)\|\leq 1$ for any $(s,a)$ and $|\int \mu(s)v(s)d(s)|\leq \sqrt{d}$ for any $\|v\|_{\infty} \leq 1$. 
\end{assum}
}

\begin{lemma}[Bilinear decomposition of low-rank $M$-step decodable POMDPs  ] \label{lem:decodable}
Suppose Assumption \ref{assum:low_rank}, $\|\Gcal_h\|_{\infty}\leq H, \|\Fcal_h\|_{\infty}\leq H$, $r_h \in [0,1]$ for any $h \in [H]$. Assume a discriminator class is Bellman complete, i.e., $$\forall \pi \in \Pi, \forall g \in \Gcal:  (\Bcal^{\pi}_h  g) - g_h  \in \Fcal_h,$$ for any $h\in [H]$. The loss function is designed as 
\begin{align}\label{eq:loss_m_step}
    l_h(  \tau_h, a_h, r_h, o_{h+1}; f, \pi, g ) := \pi_h(a_h \mid \bar z_h)|\Acal| f( \bar z_h ) ( g_h(\bar z_h) - r_h - g_{h+1}( \bar z_{h+1} )   ) - 0.5 f(\bar z_h)^2.
\end{align}
Then, there exist $W_h(\pi,g),X_h(\pi')$ so that the PO-bilinear rank is at most $d$ such that 
\begin{align}
    & |\EE[g_h(\bar z_h) - r_h - g_{h+1}(\bar z_{h+1}):a_{1:h}\sim \pi]| = |\langle W_h(\pi,g), X_h(\pi) \rangle |,\label{eq:first_condition} \\ 
    & \left\lvert  \max_{f\in \Fcal_h}\EE[l_h( \tau_h, a_h, r_h, o_{h+1}; f, \pi, g);a_{1:M(h)-1}\sim \pi', a_{M(h):h} \sim \Ucal(\Acal) ] \right\rvert \geq  \frac{ 0.5 \langle W_h(\pi,g), X_h(\pi') \rangle^2}{ |\Acal|^{M}}, \label{eq:second_condition}
\end{align}
and 
\begin{align}
    \left\lvert  \max_{f \in \Fcal_h}\EE[l_h( \tau_h, a_h, r_h, o_{h+1}; f, \pi, g^{\pi});a_{1:M(h)-1}\sim \pi', a_{M(h):h} \sim \Ucal(\Acal) ] \right\rvert = 0. \label{eq:third_condition}
\end{align}
\end{lemma}
\begin{proof}
The proof is deferred to Section \ref{subsec:proof_po_bilinear}. Note that \eqref{eq:first_condition}, \eqref{eq:second_condition}, \eqref{eq:third_condition} correspond to (a), (b), (c) in \pref{def:bilinear_minimax}. 
\end{proof}

We use the most general bilinear class definition from \pref{def:bilinear_minimax}, where
$\zeta(a) =  |\Acal|^{M/2} a^{1/2}$ for scalar $a \in \RR^+$. Hence %
$\zeta$ is a non-decreasing function ($\zeta$ is non-decreasing in $\RR^+$). The proof of the above lemma leverages the novel trick of the so-called moment matching policy introduced by \cite{efroni2022provable}.
When the latent state and action space are discrete, it states that the bilinear rank is $|\Scal||\Acal|$, which is much smaller than $|\Ocal|^{M} |\Acal|^M  |\Scal|$.  Note we here introduce $-0.5f(\bar z_h)^2$ in the loss function \eqref{eq:loss_m_step} to induce strong convexity w.r.t $f$  as in \citep{uehara2021finite,DikkalaNishanth2020MEoC,chen2019information}, which is important to obtain the fast rate later. 

The concrete sample complexity of \ouralg{} (\pref{alg:OACP_minimax}) for this model is summarized in the following. Recall that the bilinear rank is $d$ where $d$ is the rank of the transition matrix. We set $\Gcal_h \subset [\bar \Zcal_h \to [0,H]]$. Then, we have the following result.

\begin{theorem}[Sample complexity for $M$-step decodable POMDPs (Informal)] \label{thm:sample_m_step}
Suppose Assumption \ref{assum:low_rank}, Bellman completeness, $\|\Gcal_h\|_{\infty}\leq H, \|\Fcal_h\|_{\infty}\leq H$, $r_h \in [0,1]$ for any $h \in [H]$. With probability $1-\delta$, we can achieve $J(\pi^{\star})-J(\hat \pi)\leq \epsilon$ when we use samples at most 
\begin{align*}
\tilde O\prns{ \frac{d^2 H^6 |\Acal|^{2+M}\ln(|\Pi_{\max}||\Fcal_{\max}||\Gcal_{\max}|/\delta)}{\epsilon^2} }.    
\end{align*}
Here, $\mathrm{polylog}(d,H,|\Acal|,1/\epsilon,\ln(|\Pi_{\max}|),\ln(|\Fcal_{\max}|),\ln(|\Gcal_{\max}|),\ln(1/\delta))$ are omitted. 
\end{theorem}

{The followings are several implications. First, the error rate scales with $O(1/\epsilon^2)$. As we promised, by leveraging the strong convexity of loss functions, we obtain a rate $O(1/\epsilon^2)$, which is faster than $O(1/\epsilon^4)$ that are attained when we naively invoke \pref{thm:online_general} with $\xi(x) \propto \sqrt{x}$. Secondly, the error bound incurs $|\Acal|^M$. As showed in \citep{efroni2022provable}, this is inevitable in $M$-step decodable POMDPs. Thirdly, in the tabular case, when we use the na\"ive function classes for $\Gcal,\Fcal,\Pi$, i.e., $\Gcal_h= \{\bar \Zcal_h \to [0,H]\}$, $\Fcal_h= \{\bar \Zcal_h \to [0,H]\}$, $\Pi_h= \{\bar \Zcal_h \to \Delta(\Acal)\}$, the bound could incur additional $|\Ocal|^M$ since the complexity of the function classes can scale with respect to $(|\Ocal| |\Acal|)^M$ (e.g., $\log(|\Gcal_h|)$ can be in the order of $O(|\Ocal|^M|\Acal^M)$, and similarly for  $\log(|\Fcal_h|), \ln(\Pi_h )$). However, when we start form a realizable model class that captures the ground truth transition and omission distribution, we can remove  $|\Ocal|^M$. See Section \pref{subsec:tabular_decodable} for an example. }

Note that \cite{efroni2022provable} uses a different function class setup where they assume one has an M memory-action dependent $Q$ function class $\Qcal_h: \bar \Zcal_{h}\times \Acal \to \RR$ which contains $Q^\star_h( \bar z_h, a )$ while we use the actor-critic framework $v^{\pi}_h,\pi$. The two function class setups are not directly comparable. Generally, we mention that such optimal $Q^\star$ with truncated history does not exist when the exact decodability does not hold (e.g., such $Q^\star$ with truncated history does not exist in LQG). This displays the potential generality of the actor-critic framework we propose here.  %

{ 

\subsection{Observable POMDPs with Latent Low-rank Transition: a model-based perspective}

The final example we include in this work is a POMDP with the latent low-rank transition. 
We first introduce the model, and then we introduce our function approximation setup and show the  sample complexity. Finally, we revisit the sample complexity for observable tabular POMDPs and $M$-step decodable tabular POMDPs using the improved algorithm that elaborates on the model-based approach in this section. %

\begin{myexp}[Observable POMDPs with latent low-rank transition] \label{ex:low_rank}
The latent transition $\TT(s'|s,a)$ is factorized as $\TT(s'| s,a) = \mu^\star(s')^\top \phi^\star(s,a), \forall s,a,s'$ where $\mu^\star:\Scal\to \RR^d$ and $\phi^\star:\Scal\times\Acal \to \RR^d$.  The observation $|\Ocal|\times |\Scal|$ matrix  $\OO$ has full-column rank.

\end{myexp}

In the tabular POMDP example, we have $d \leq |\Scal|$. However in general $d$ can be much smaller than $|\Scal|$. Note that in this section, we will focus on the setting where $\Scal, \Ocal$ are discrete to avoid using measure theory languages, but their size could be extremely large. Particularly, our sample complexity will not have explicit polynomial or logarithmic dependence on $|\Ocal|, |\Scal|$, instead it will only scale polynomially with respect to the complexity of the hypothesis class and the rank $d$.

\paragraph{Model-based function approximation.}

Our function approximation class consists of a set of models $\Mcal = \{ (  \mu, \phi, O )  \}$ where $\mu, \phi$ together models latent transition as $\mu(\cdot)^\top \phi(s,a)\in \Delta(\Scal)$, and $O: \Scal\to \Delta(\Ocal)$ models $\OO$, and $O$ is full column rank. For notation simplicity, we often use $\theta := (\mu, \phi, O) \in \Mcal$ to denote a model $(\mu, \phi, O)$.  We impose the following assumption. 

\begin{assum}[Realizability]\label{assum:realizable}
We assume realizability, i.e., $(\mu^\star, \phi^\star, \OO) \in \Mcal$.
\end{assum}

We assume $\Mcal$ is discrete, but $|\Mcal|$ can be large such that a linear dependence on $|\Mcal|$ in the sample complexity is not acceptable. Our goal is to get a bound that scales polynomially with respect to $\ln( | \Mcal | )$, which is the standard statistical complexity of the discrete hypothesis class $\Mcal$. %

Next, we construct $\Pi, \Gcal, \Fcal$ using the model class $\Mcal$.  Given $\theta := (\mu, \phi, O)$, we denote $\pi^\theta$ as the optimal $M$-memory policy, i.e., the $M$-memory policy that maximizes the total expected reward. We set $$\Pi = \{ \pi^\theta: \theta\in \Mcal  \}.$$

We consider the value function class for $\theta := (\mu, \phi, O)$ with $O$ being full column rank. For each $\theta$, we can define the corresponding value function of the policy $\pi$ at $h \in [H]$: $V^{\pi}_{\theta;h}(z_{h-1},s_h):\Zcal_{h-1}\times \Scal \to \RR$. Then, since $O$ is full column rank, as we see in the proof of \pref{lem:tabular_pomdp_value_bridge}, a corresponding value link function is 
\begin{align*}
    g^{\pi}_{\theta;h}(z,o) = \langle f^{\pi}_{\theta,h}, \One(z) \otimes \OO^{\dagger}\One(o) \rangle
\end{align*}
where $V^{\pi}_{\theta;h}(z_{h-1},s_h)= \langle f^{\pi}_{\theta;h}, \One(z) \otimes \One(s) \rangle$.  %
Then, we construct $\Gcal = \{\Gcal_h\}$ as:
\begin{align}\label{eq:original}
\forall h\in [H]: \; \Gcal_h = \{\bar \Zcal_{h}\ni \bar z_{h-1} \mapsto g^\pi_{\theta;h}(\bar z_{h-1}) \in \mathbb{R}: \pi \in \Pi, \theta \in \Mcal \}.
\end{align} By construction, since $\theta^\star := ( \mu^\star, \phi^\star, \OO ) \in \Mcal$, we must have $g^\pi \in \Gcal,\forall \pi\in \Pi$, which implies $\Gcal$ is realizable (note $g^{\pi}_h=g^{\pi}_{\theta^{\star};h}$). { Here, from the construction and the assumption $r_h \in [0,1]$ for any $h\in [H]$, we have $|\Gcal_h| \leq |\Mcal|^2$ and $\|\Gcal_h\|_{\infty}\leq H/\sigma_1$, which can be seen from
\begin{align*}
\forall (z,o);    \langle f^{\pi}_{\theta;h}, \One(z) \otimes O^{\dagger}\One(o) \rangle \leq \| f^{\pi}_{\theta;h}\|_{\infty} \| \One(z) \otimes O^{\dagger}\One(o)\|_1 \leq H \times \|O^{\dagger}\One(o)\|_1 \leq H/\sigma_1 
\end{align*}
by assuming $\|O^{\dagger}\|_1 \leq 1/\sigma_1$ and $\|f^{\pi}_{\theta;h}\|_{\infty} \leq H$. }

To construct a discriminator class $\Fcal$, we first define the Bellman operator $\Bcal^\pi_{\theta;h}$ for $\pi\in \Pi, h\in [H], \theta \in \Mcal$: 
\begin{align*}
&\Bcal^\pi_{\theta;h}: \Gcal \to [\Hcal_h \to \RR]; \\ 
& \forall \tau_h; \left(\Bcal^\pi_{\theta;h} g \right)( \tau_h ) = \EE_{a_h \sim \pi_h(\bar z_h)}\left[ r_h + \EE_{ o_{h+1} \sim \PP_{\theta}(\cdot|\tau_h,a_h) }g_{h+1}( \bar z_{h + 1}) \right],
\end{align*} where $\Hcal_h$ is the whole history space up to $h$ ($\tau_h=(a_{1:h-1},o_{1:h})$, and $\bar z_h$ is just part of this history) and  $\PP_\theta(o_{h+1} | \tau_h, a_h)$ is the probability of generating $o_{h+1}$ conditioned on $\tau_h,a_h$ under model $\theta$. { Then, we construct $\Fcal = \{\Fcal_h\}$ such that 
\begin{align}\label{eq:discriminator}
\forall h \in [H]: \Fcal_h = \{\Hcal_h \ni \tau_h \mapsto \{g_h - \Bcal^\pi_{\theta;h} g\}(\tau_h) \in \RR:\pi\in \Pi, g\in \Gcal, \theta\in \Mcal  \}.    
\end{align}}
so that we can ensure the Bellman completeness: 
\begin{align*}%
    - (\Bcal^{\pi}_h \Gcal) + \Gcal_h \subset \Fcal_h.
\end{align*}
noting $\Bcal^\pi_{\theta^{\star};h} = \Bcal^{\pi}_h $. Here, from the construction, $|\Fcal_h| \leq |\Mcal|^2 \times |\Mcal|^2 \times |\Mcal|^2= |\Mcal|^{6}$ and $\|\Fcal_h\|_{\infty}\leq 3H/\sigma_1$. 

We define the loss as the same as the one we used in $M$-step decodable POMDPs, except that our discriminators now take the entire history as input: 
\begin{align}\label{eq:loss}
l_h(  \tau_h, a_h, r_h, o_{h+1}; f, \pi, g ) := \pi_h(a_h \mid \bar z_h)|\Acal| f( \tau_h ) ( g_h(\bar z_h) - r_h - g_{h+1}( \bar z_{h+1} )   ) - 0.5 f( \tau_h)^2.
\end{align}

Finally, as in the case of $M$-step decodable POMDPs (\pref{lem:decodable}), we get the following lemma that states that our model is a PO-bilinear AC class (\pref{def:bilinear_minimax}) under the following model assumption.  

\begin{assum}\label{assum:regular}
We assume $\|O^{\dagger}\|_1 \leq 1/\sigma_1$ for any $O$ in the model. Suppose $\mu(\cdot)^{\top}\phi(s,a) \in \Delta(\Scal)$ for any $(s,a)$, $\mu(\cdot)$ and $\phi(\cdot)$ in the model. Suppose $\|\phi(s,a)\| \leq 1$ for any $\phi$ in the model and $(s,a)\in \Scal\times \Acal$. Suppose for any $v:\Scal \to [0,1]$ and  for any $\mu$ in the model, we have $\|\int v(s)\mu(s)\mathrm{d}(s)\|_2 \leq \sqrt{d}$. 
\end{assum} %

\begin{lemma}[PO-bilinear decomposition for Observable POMDPs with low-rank transition]\label{lem:low_rank_po_bilinear}
Suppose Assumption \ref{assum:realizable}, \ref{assum:regular}. Consider observable POMDPs with latent low-rank transition. Set $\Gcal$ as in \pref{eq:original}, $\Fcal$ as in \pref{eq:discriminator} and $l$ as in \pref{eq:loss}. Then, there exist $W_h(\pi,g),X_h(\pi')$ that admits the PO-bilinear rank decomposition in \pref{def:bilinear_minimax} with rank $d$.  %
\end{lemma}

The above lemma ensures that the PO-bilinear rank only depends on $d$, and is independent of the length of the memory. For example, in the tabular case, it is $|\Scal|$.

Next, we show the output from \oursecondalg can search for the best in class $M$-memory policy as follows. 

\begin{theorem}[Sample complexity of \oursecondalg{} for observable POMDPs with latent low-rank transition]\label{thm:sample_low_rank}
Consider observable POMDPs  with latent low-rank transition. 
Suppose Assumption \ref{assum:realizable}, \ref{assum:regular}. 
With probability $1-\delta$, we can achieve $J(\pi^{\star})-J(\hat \pi)\leq \epsilon$ when we use samples at most 
\begin{align*}
\tilde O\prns{ \frac{d^2 H^6 |\Acal|^{2+M}\ln(|\Mcal|/\delta)}{\epsilon^2 \sigma^2_1} }. 
\end{align*}
Here, we omit $\mathrm{polylog}(d, H, |\Acal|, \ln(1/\delta), \ln(|\Mcal|),1/\sigma_1,1/\epsilon)$. 
\end{theorem}
 
Here, we emphasize that there is no explicit polynomial or logarithmic dependence on $|\Scal|$ and $|\Ocal|$, which permits learning for large state and observation spaces. We also do not have any explicit polynomial dependence on $|\Ocal|^M$, as we construct $\Pi$ and $\Gcal$ from the model class $\Mcal$ which ensures the complexities of $\pi$ and $\Gcal$ are in the same order as that of $\Mcal$. 

\subsubsection{Global Optimality}
 
We show a quasi-polynomial sample complexity bound for competing against the globally optimal policy $\pi^\star_{\text{gl}}$. To compete against the globally optimal policy $\pi^\star_{\mathrm{gl}}$, we need to set $M$ properly. We use the following lemma. The proof is given in \pref{sec:exponential_stability}. 

\begin{lemma}[Near global optimaltiy of $M$-memoruy policy]
Consider $\epsilon \in (0, H]$, and a POMDP with low-rank latent transition and $\OO$ being full column rank with $\| \OO^\dagger \|_1 \leq 1/\sigma_1$. When $M  = \Theta( C (1/\sigma_1)^{-4} \ln(d H / \epsilon))$ (with $C$ being some absolute constant), there must exists an $M$-memory policy $\pi^\star$, such that $  J(\pi^\star_{\mathrm{gl}}) - J({\pi^\star})  \leq \epsilon$
\end{lemma}
Note that the memory $M$ above is independent of $|\Scal|$ instead it only depends on the rank $d$. To prove the above lemma, we first show a new result on belief contraction for low-rank POMDPs under the $\ell_1$-based observability. The proof of the belief contraction borrows some key lemma from \citep{golowich2022planning} but extends the original result for small-size tabular POMDPs to low-rank POMDPs.  We leverage the linear structure of the problem and the G-optimal design to construct an initial distribution over $\Scal$ that can be used as a starting point for belief propagation along the memory.

We conclude the study on the POMDPs with low-rank latent transition by the following theorem, which demonstrates a quasi-polynomial sample complexity for learning the globally optimal policy. 

\begin{theorem}[Sample complexity of \oursecondalg{} for POMDPs with low-rank latent transition --- competing against $\pi^\star_{\mathrm{gl}}$]
Consider observable POMDPs  with latent low-rank transition. Fix some $\epsilon\in (0,H), \delta\in(0,1)$.
Suppose Assumption \ref{assum:realizable}, \ref{assum:regular}. We construct $\Pi, \Gcal, \Fcal$, and the loss $l$ as we described above. With probability at least $1-\delta$,  when $M = \Theta(C \sigma_1^{-4} \ln(d H /\epsilon))$, \oursecondalg{} outputs a $\hat\pi$ such that $J(\pi^\star_{\mathrm{gl}}) - J(\hat\pi) \leq \epsilon$, with  number of samples scaling 
$$\tilde O\left(  \frac{ d^2 H^6  |\Acal|^2  \ln(|\Mcal / \delta |) }{ \epsilon^2 \sigma_1^2} \cdot |\Acal|^{  \ln(dH/\epsilon) / \sigma_1^4 }  \right).$$
\end{theorem}

\begin{remark}[Comparison to \cite{wang2022embed}]
We compare our results to the very recent work \cite{wang2022embed} that studies POMDPs with the low-rank latent transition. The results are in general not directly comparable, but we state several key differences here. %
First, \cite{wang2022embed} considers a special instance of low-rank transition, i.e., \cite{wang2022embed} assumes $\TT$ has low non-negative rank, which could be exponentially larger than the usual rank \citep{agarwal2020flambe}. Second, \cite{wang2022embed} additionally assumes \emph{short past sufficiency}, a condition which intuitively says that for any roll-in policy, the sufficient statistics of a short memory is enough to recover the belief over the latent states, and their sample complexity has an exponential dependence on the length of the memory. While our result also relies on the fact that the globally optimal policy can be approximated by an $M$-memory policy with small $M$, this fact is derived directly from the standard observability condition. 
\end{remark}

\subsubsection{Revisiting Observable Undercomplete Tabular POMDPs}\label{sec:undercomplete_pomdp_finite}

We reconsider the sample complexity of undercomplete tabular POMDPs using \pref{thm:sample_low_rank}. In this case, we will start from a model class that captures the ground truth latent transition $\TT$ and omission distribution $\OO$.  By constructing $\epsilon$-nets over the model class,%
 we can set $\ln(|\Mcal|) = \tilde O(|\Scal|^3|\Ocal| |\Acal|)$ since $\TT,\OO$ have $|\Scal|^2 |\Acal|$ and $|\Ocal||\Scal|$ many parameters, respectively. Besides, the PO-bilinear rank is $d = |\Scal|$.  Therefore, the sample complexity is 
\begin{align*}
\tilde O\prns{ \frac{|\Scal|^5 |\Ocal| H^6 |\Acal|^{2+M}\ln(1/\delta)}{\epsilon^2 \sigma^2_1} }.  
\end{align*}
We leave the formal analysis to future works.

Compared to results in Section \ref{subsec:observable_tabular_ab}, there is no $|\Ocal|^M$ term. This is due to two improvements. The first improvement is that we refine the rank from $|\Ocal|^M|\Acal|^M|\Scal|$ to $|\Scal|$. The second improvement is we model the value link function class and policy class starting from the model class whose complexity has nothing to do with the length of memory $M$ (note that previously, from a pure model-free perspective, the statistical complexity of $\Gcal$ can scale as $|\Ocal|^M|\Acal|^M |\Scal|$ in the worst case). 

\subsubsection{Revisiting Observable Overcomplete POMDPs} \label{sec:overcomplete_pomdp_finite}

We reconsider the sample complexity of overcomplete tabular POMDPs using \pref{thm:sample_low_rank} with slight modification to incorporate multi-step future. Suppose $\|\{\OO^K\}^{\dagger}\|_1 \leq 1/\sigma_1$ (recall $\OO^K$ is defined in \pref{lem:over_complete_bridge} in Section \ref{subsec:overcomplete_pomdps_tabular}). Then, we can achieve a sample complexity  
\begin{align*}
\tilde O\prns{ \frac{|\Scal|^5 |\Ocal| H^6 |\Acal|^{2+M}\ln(1/\delta)}{\epsilon^2 \sigma^2_1} }
\end{align*}
since the PO-bilinear rank is $|\Scal|$. Note that there is no $|\Ocal|^{M+K}$ dependence, since both the policy class and the value link function class are built from the model class whose complexity has nothing to do with $M,K$. 
 
Note that due to our definition of $\OO^K$,  there is no $|\Acal|^K$ term. However, when we use a different definition, for instance, $\min_{a'_{h:h+K-2} \in \Acal^{K-1}}   \|\{\OO^K(a'_{h:h+K-2})\}^{\dagger}\|_1\leq 1/\alpha_1$   (recall $\OO^K(a'_{h:h+K-2})$ is defined in Section \ref{subsec:overcomplete_pomdps_tabular}), we would incur $|\Acal|^K$.  This is because if we only know that there is an unknown sequence of actions $a'_{h:h+K-2}$ such that $\OO^K(a'_{h:h+K-2})$ is full column rank, we need to use uniform samples $|\Acal|^K$ in the importance sampling step to identify such a sequence. More formally, we can see that
\begin{align}\label{eq:singular_value_relation}
  |\Acal|^K \min_{a'_{h:h+K-2} \in \Acal^{K-1}}   \|\{\OO^K(a'_{h:h+K-2})\}^{\dagger}\|_1 \geq  \|\{\OO^K\}^{\dagger} \|_1. 
\end{align}

\subsubsection{Revisiting $M$-step Decodable Tabular POMDPs}\label{subsec:tabular_decodable}

We reconsider the sample complexity of tabular $M$-step decodable POMDPs by constructing $\Fcal,\Gcal,\Pi$ from the model class $\Mcal$ as we did for the low-rank POMDP. In this case, by constructing $\epsilon$-nets, %
we can set $\ln(|\Mcal|) = \tilde O(|\Scal|^3|\Ocal| |\Acal|)$ since $\TT, \OO$ have $|\Scal|^2 |\Acal|$ and $|\Ocal||\Scal|$ parameters, respectively. 
Therefore, the sample complexity is 
\begin{align*}
\tilde O\prns{ \frac{H^6|\Scal|^5 |\Ocal|  |\Acal|^{2+M}\ln(1/\delta)}{\epsilon^2} }.  
\end{align*}
Again, we leave the formal analysis to future works. Compared to the naive result mentioned after \pref{thm:sample_m_step} where $\ln(\Gcal), \ln(\Pi)$ could scale in the order of $|\Ocal|^M$ in the tabular case,  we do not have $|\Ocal|^M$ dependence here.

}

%% file: main_document/ape_related_work.tex
\tableofcontents

%% file: main_document/ape_introduce_value_bilinear.tex
\section{Supplement for \pref{sec:def_value_bridge}} \label{sec:ape_def_value_bridge}

We generalize \pref{def:simple_bilinear} to capture more models. The first extension is to use multi-step link functions. This extension is essential to capture overcomplete POMDPs and multi-step PSRs. The second extension is to use minimax loss functions with discriminators so that we can use not only absolute value loss functions but also squared loss functions. This extension is important to capture M-step decodable POMDPs.

%% file: main_document/ape_example.tex
\section{Supplement for \pref{sec:example}}\label{sec:ape_example}

\subsection{Observable Undercomplete Tabular POMDPs }

We need to prove \pref{lem:tabular_pomdp_bilinear}. In the tabular case, by setting
\begin{align*}
    \psi_h(z,o) = \One(z) \otimes \One(o), \phi_h(z,s) = \One(z) \otimes \One(s), K_h = I_{|\Zcal_{h-1}|}\otimes \OO 
\end{align*}
where $\One(z),\One(o),\One(s)$ are one-hot encoding vectors over $\Zcal_{h-1},\Ocal,\Scal$, respectively. Then, we can regard the tabular model as an HSE-POMDP. We can just invoke \pref{lem:hse_existence}.

\subsection{Observable Overcomplete POMDPs} \label{subsec:over_complete_ape}

We consider overcomplete POMDPs with multi-step futures. We have the following theorem. This is a generalization of \pref{lem:tabular_pomdp_value_bridge}.

\paragraph{Proof of \pref{lem:over_complete_bridge}}
Consider any function $f:\Zcal_{h-1}\times \Scal \to \RR$ (thus, this captures all possible $V^\pi_h$). Denote $\one(z)$ as the one-hot encoding of $Z_{h-1}$ (similarly for $\one(s)$ over $\Scal$ and $\one(t)$ over $\Tcal^K$). We have $f(z, s) = \langle f, \one(z)\otimes \one(s) \rangle = \langle f ,  \one(z) \otimes ((\OO^K)^\dagger \OO^K \one(s) ) \rangle $, where we use the assumption that $\rank(\OO^K)=|\Scal|$ and thus $(\OO^K)^\dagger \OO^K = I$.
Then,  
\begin{align*}
f(z_{h-1}, s_h) &= \langle f , \one(z_{h-1}) \otimes (\OO^K)^\dagger \EE[\one(o_{h:h+K-1},a_{h:h+K-2}) \mid s_h;a_{h:h+K-2} \sim \pi^{out} ] \rangle \\
 &= \EE[ \langle f , \one(z_{h-1}) \otimes  (\OO^K)^\dagger   \one(o_{h:h+K-1},a_{h:h+K-2}) \rangle \mid z_{h-1},s_h;a_{h:h+K-2} \sim \pi^{out} ] .  
\end{align*}
which means that the value bridge function corresponding to $f(\cdot)$ is $$g(z,t):=\langle f , \one(z) \otimes  (\OO^K)^\dagger   \one(t) \rangle.$$ $\quad \blacksquare$

\paragraph{Proof of \pref{lem:multi_step_bilinear}}

Recall we want to show the low-rank property of the following loss function: 
 \begin{align*}
      &  \EE[g_{h+1}(\bar z^K_{h+1})  ; a_{1:h-1}\sim \pi' 
    , a_h\sim \pi,a_{h+1:h+K-1}\sim  \pi^{out}  ] + \EE[r_h  ;a_{1:h-1}\sim \pi', a_h\sim \pi] \\
    & -  \EE[g_{h}(\bar z^K_{h})  ;a_{1:h-1}\sim \pi', a_{h:h+K-1}\sim  \pi^{out}  ] . 
\end{align*}
We consider an expectation conditioning on $z_{h-1}$ and $s_h$. For some vector $\theta_{\pi,g} \in \RR^{|\Zcal_{h-1}|\times |\Scal|}$, which depends on $\pi$, we write it in the form of $\langle \theta_{\pi,g},  \one(z_{h-1},s_h) \rangle$ where $\one(z_{h-1},s_h)$ is the one-hot encoding vector over $\Zcal_{h-1} \times \Scal$. Then, the loss for $(\pi,g)$ is equal to 
\begin{align*}
  \langle \theta_{\pi,g},  \EE[\one(z_{h-1},s_h)  ;a_{1:h-1}\sim \pi' ] \rangle. 
\end{align*}
Hence, we can take $X(\pi') = \EE[\one(z_{h-1},s_h)  ;a_{1:h-1}\sim \pi' ]$ and $ W(\pi) = \theta_{\pi,g}$. $\quad \blacksquare$

\subsection{Observable Linear Quadratic Gaussian}

We need to prove \pref{lem:lqg_value}. The proof is further deferred to \pref{sec:sample_complexity_lqg}.

\subsection{Observable HSE-POMDPs} \label{subsec:hse_pomdps_ape}

We first provide the proof of \pref{lem:hse_existence}. Then, we briefly mention how we extend to the infinite-dimensional setting. 

\paragraph{Proof of the first statement in \pref{lem:hse_existence}}

First, we need to show value bridge functions exist. This is proved noting 
\begin{align*}
    \EE_{o \sim \OO(s)}[\langle (K^{\dagger}_h)^{\top}\theta^{\pi}_h, \psi_h(\bar z_h) \rangle ] = \langle (K^{\dagger}_h)^{\top} \theta^{\pi}_h, K_h  \phi_h(z_{h-1},s_h) \rangle = \langle \theta^{\pi}_h, \phi_h(z_{h-1},s_h) \rangle = V^{\pi}_h(z_{h-1},s_h). 
\end{align*}
Thus, $\langle (K^{\dagger}_h)^{\top}\theta^{\pi}_h, \psi_h(\bar z_h) \rangle $ is a value bridge function. $\quad \blacksquare$

\paragraph{Proof of the second statement in \pref{lem:hse_existence}}

Consider a triple $(\pi', \pi, g) \in \Pi\times \Pi\times \Gcal$, with $g_h(\cdot) = \theta_h^\top \psi_h(\cdot)$ and $g^{\pi}_h = \langle \theta^{\star}_h,\psi_h(\cdot) \rangle$, we have:
\begin{align*}
&  \mathrm{Br}_h(\pi,g;\pi') \\ 
&= \EE \left[ \theta_h^{\top} \psi(\bar z_h) - r_h - \theta_{h+1}^\top \psi(\bar z_{h+1}); a_{1:h-1}\sim \pi',a_h \sim \pi   \right] \\
& = \EE \left[ \theta_h^{\top} K_h \phi_h(z_{h-1}, s_h) - r_h - \theta_{h+1}^\top  K_{h+1}( \phi_{h+1}( z_h, s_{h+1}  ) ); a_{1:h-1}\sim \pi',a_h \sim \pi \right] \\
& = \EE  \left[ (\theta_h-\theta^{\star}_h)^{\top} K_h \phi_h(z_{h-1}, s_h)  -(\theta_{h+1}-\theta^{\star}_{h+1})^\top  K_{h+1}( T_{\pi;h} \phi_{h}( z_{h-1}, s_{h}  ) );   a_{1:h-1}\sim \pi'\right] \\
& = \left\langle   \EE[\phi_h(z_{h-1}, s_h) ;   a_{1:h-1}\sim \pi'], \quad  K_h^{\top} (\theta_h - \theta^{\star}_h) - T_{\pi;h}^\top K_{h+1}^\top (\theta_{h+1}-\theta^{\star}_{h+1})  \right\rangle,
\end{align*} which verifies the bilinear structure, i.e., $X_h(\pi') = \EE[\phi_h(z_{h-1}, s_h) ;   a_{1:h-1}\sim \pi']$, and $W_h(\pi, g) =K_h^{\top} (\theta_h - \theta^{\star}_h) - T_{\pi;h}^\top K_{h+1}^\top (\theta_{h+1}-\theta^{\star}_{h+1}) $, and shows that the bilinear rank is at most $\max_{h} d_{\psi_h}$.$\quad \blacksquare$

\paragraph{Infinite dimensional HSE-POMDPs}

Consider the case $\phi_h$ and $\psi_h$ are features in infinite dimensional RKHS. By assuming that the spectrum of the operator $K_h$ is decaying with a certain order, we can still ensure the existence of value bridge functions even if $d_{\phi_h}$ and $d'_{\psi_h}$ are infinite dimensional. 

Next, we consider the PO-bilinear rank. We can still use the decomposition in the proof above. While the PO-bilinear rank itself in the current definition  is infinite-dimensional, when we get the PAC result later, the dependence on the PO-bilinear rank comes from  the information gain based on $X_h(\pi)$, which is the intrinsic dimension of $X_h(\pi)$. Thus, we can easily get the sample complexity result by replacing $d_{\psi_h}$ with the information gain over $\psi_h(\cdot)$ \citep{srinivas2009gaussian}. Generally, to take infinite dimensional models into account, the PO-bilinear rank in \pref{def:simple_bilinear} can be generalized using the critical information gain \citep{du2021bilinear}.

%% file: main_document/ape_algorithm.tex
\section{Supplement for \pref{sec:algorithm} (Algorithm for LQG with Continuous Action)}\label{sec:ape_algo}

In this section, we present a modification to handle LQG with continuous action in \pref{def:simple_bilinear}.

Our algorithm so far samples $a_h$ from $\Ucal(\Acal)$ and performs importance weighting in designing the loss $\sigma^t_h$, which will incur a polynomial dependence on $|\Acal|$ as we will see in the next section. However, among the examples that we consider in \pref{sec:example}, LQG has continuous action. If we na\"ively sample $a_h$ from a ball in $\RR^{d_a}$ and perform (nonparametric) importance weighting, we will pay $\exp(d_a)$ in our sample complexity bound, which is not ideal for high-dimension control problems. To avoid exponential dependence on $d_a$, here we replace $\Ucal(\Acal)$ with a $d$-optimal design over the action's quadratic feature space.

Here,  we want to evaluate the Bellman error of $(\pi, g)$ pair under a roll-in policy $\pi'$: %
\begin{align*}
   \mathrm{Br}_h(\pi,g;\pi'):= \EE[u_h(\bar z_h,a_h,r_h,o_{h+1};\theta); a_{1:h-1} \sim \pi', a_h \sim \pi(\bar z_h) ] 
\end{align*}
where $u_h(\bar z_h,a_h,r_h,o_{h+1};\theta)= \theta^{\top}_h \psi(\bar z_h) - r_h(s_h,a_h)-\theta^{\top}_{h+1}\psi_{h+1}(\bar z_{h+1})$ for any linear deterministic policy $\pi \in \Pi$ (here $g_h(\cdot ):= \theta_h^\top \psi(\cdot)$) using a \emph{single} policy. In other words, we would like to get a good loss $l_h$ such that 
\begin{align*}
       \mathrm{Br}_h(\pi,g;\pi')=\EE[l_h(\bar z_h,a_h,r_h,o_{h+1};\theta,\pi); a_{1:h-1} \sim \pi', a_h \sim \pi^{e} ]
\end{align*}
for some policy $\pi^{e}$ without incuring exponential dependence on $d_a$. We explain how to design such a loss function $l_h(\cdot;\pi,g)$ step by step.  

\paragraph{First Step}

The first step is to consider the conditional expectation on $(\bar z_h,s_h,a_h)$. Here, using the quadratic form of $\psi$, we can show that there are some $c_0: \bar Z_h \times \Scal \to \RR,c_1: \bar Z_h \times \Scal \to \RR^{(d_a + d_s +d_{\bar z_h})^2}, c_2 \in \RR$:  
\begin{align*}
       \mathrm{Br}_h(\pi,g;\pi')&=\EE[u_h(\bar z_h,a_h,r_h,o_{h+1};\theta)\mid \bar z_h,s_h,a_h ; a_{1:h-1} \sim \pi, a_h \sim \pi(\bar z_h) ]\\
    &= \langle c_2(\theta), [1, [\bar z^{\top}_h,s^{\top}_h,a^{\top}_h]\otimes [\bar z^{\top}_h,s^{\top}_h,a^{\top}_h]]^{\top} \rangle  \\ 
    &= c_0(\bar z_h,s_h;\theta) +  c^{\top}_1(\bar z_h,s_h;\theta) \kappa(a_h)
\end{align*} 
where $\kappa(a) = [ a^\top, (a\otimes a)^{\top}]^\top$. Then, the Bellman loss we want to evaluate can be written in the form of 
\begin{align*}
     &\EE[u_h(\bar z_h,a_h,r_h,o_{h+1};\theta); a_{1:h-1} \sim \pi', a_h \sim \pi(\bar z_h) ] \\
    & = \EE[c_0(\bar z_h,s_h;\theta) +  c^{\top}_1(\bar z_h,s_h;\theta) \kappa(\pi(\bar z_h)); a_{1:h-1}\sim \pi'].
\end{align*}

\paragraph{Second step} The second step is to compute a d-optimal design for the set $\{ \kappa(a):  a \in \RR^{d_a}, \|a\|_2 \leq Z \}$ for certain enough large $Z \in \RR$, and denote $a^1, \dots, a^{d^{\diamond}}$ as the supports on the d-optimal design. Note in LQG, though we cannot ensure the action lives in the compact set, we can still ensure that in high probability and it suffices in our setting as we will see. Since the dimension of $k(a)$ is $d_a+d^2_a$, we can ensure $d^{\diamond} \leq (d_a+d_a^2) (d_a+d_a^2+1)/2$ \citep{lattimore2020bandit,kiefer1960equivalence}.  Here is a concrete theorem we invoke. 

\begin{theorem}[Property of G-optimal design]\label{thm:g_optimal}
Suppose $\Xcal \in \mathbb{R}^d$ is a compact set. There exists a distribution $\rho$ over $\Xcal$ such that:
\begin{itemize}
    \item $\rho$ is supported on at most $d(d+1)/2$ points. 
    \item For any $x'\in \Xcal$, we have $x'^{\top} \EE_{x\sim \rho}[xx^{\top}]^{-1} x'\leq d$. 
\end{itemize}
\end{theorem}
We have the following handy lemma stating any $\kappa(a)$ is spanned by $\{\kappa(a^i)\}_{i=1}^{d^{\diamond}}$.  

\begin{lemma}\label{lem:norm_alpha}
{\newedit Let $ K = [\rho^{1/2}(a^1)\kappa(a^1),\rho^{1/2}(a^2)\kappa(a^2),\cdots,\rho^{1/2}(a^{d^{\diamond}}) \kappa(a^{d^{\diamond}})]$ and $ \alpha(a)=K^{\top}(KK^{\top})^{-1}k(a)$. } Then, it satisfies $$\kappa(a) = K \alpha( a ),\quad \|\alpha(a)\|  \leq (d_a+d_a^2)^{1/2},\quad \alpha_i(a)/\rho^{1/2}(a^i)\leq (d_a+d_a^2) $$
\end{lemma}
\begin{proof}
Since $K$ is full-raw rank from the construction of G-optimal design, $KK^{\top}$ is invertible. Then, we have
\begin{align*}
   \sum_{i=1}^{d^{\diamond}} \alpha_i( a ) \rho^{1/2}(a^i)\kappa(a^i)= KK^{\top}(KK^{\top})^{-1} \kappa(a)=\kappa(a)
\end{align*}
For the latter statement, we have 
\begin{align*}
   \langle K^{\top}(KK^{\top})^{-1}k(a) , K^{\top}(KK^{\top})^{-1}k(a) \rangle & =  k(a)^{\top}(KK^{\top})^{-1}k(a) \leq  (d_a+d_a^2). 
\end{align*}
We use a property of G-optimal design in \pref{thm:g_optimal}. 

For the last statement, we have
\begin{align*}
    \kappa^{\top}(a^i)(KK^{\top})^{-1}\kappa(a) \leq \|\kappa^{\top}(a^i)\|_{ (KK^{\top})^{-1}}\|\kappa^{\top}(a)\|_{ (KK^{\top})^{-1}} \leq (d_a+d_a^2).
\end{align*}
from CS inequality. 
\end{proof}

\paragraph{Third Step} The third step is combining current facts. Recall we want to evaluate 
\begin{align*}
\EE[u_h(\bar z_h,a_h,r_h,o_{h+1};\theta); a_{1:h-1} \sim \pi', a_h \sim \pi(\bar z_j) ] = \EE[c_0(\bar z_h,s_h;\theta) +  c^{\top}_1(\bar z_h,s_h;\theta) \kappa(\pi(\bar z_h)); a_{1:h-1}\sim \pi']. 
\end{align*}
In addition, the following also holds: 
\begin{align*}
\EE[u_h(\bar z_h,a_h,r_h,o_{h+1};\theta); a_{1:h-1} \sim \pi', a_h \sim do(a^i) ] &=    \EE[c_0(\bar z_h,s_h;\theta) +  c^{\top}_1(\bar z_h,s_h;\theta) \kappa(a^i); a_{1:h-1}\sim \pi'] \\ 
\EE[u_h(\bar z_h,a_h,r_h,o_{h+1};\theta); a_{1:h-1} \sim \pi', a_h \sim do(0) ] &=    \EE[c_0(\bar z_h,s_h;\theta); a_{1:h-1}\sim \pi'] 
\end{align*}
Here, we use $\kappa(0)=0$. This concludes that 
\begin{align*}
    &\EE[u_h(\bar z_h,a_h,r_h,o_{h+1};\theta); a_{1:h-1} \sim \pi', a_h \sim \pi(\bar z_j) ]\\
    &= \EE[c_0(\bar z_h,s_h;\theta) +  c^{\top}_1(\bar z_h,s_h;\theta) \kappa(\pi(\bar z_h)); a_{1:h-1}\sim \pi'] \\ 
     &= \EE[c_0(\bar z_h,s_h;\theta) +  c^{\top}_1(\bar z_h,s_h;\theta) \{\sum_{i=1}^{d^{\diamond}} \alpha_i(\pi(\bar z_h))\kappa(a^i)\}; a_{1:h-1}\sim \pi'] \\ 
  &=\EE\bracks{c_0(\bar z_h,s_h;\theta)\prns{1- \sum_{i=1}^{d^{\diamond}} \alpha_i(\pi(\bar z_h))} +   \sum_{i=1}^{d^{\diamond}} \alpha_i(\pi(\bar z_h))\prns{ c^{\top}_1(\bar z_h,s_h;\theta) \kappa(a^i) +c_0(\bar z_h,s_h;\theta)}; a_{1:h-1}\sim \pi'}\\
  &= \EE \bracks{ \prns{1- \sum_{i=1}^{d^{\diamond}} \alpha_i(\pi(\bar z_h))} u_h(\bar z_h,a_h,r_h,o_{h+1};\theta); a_{1:h-1} \sim \pi', a_h \sim do(0) } \\
  &+ \sum_{i=1}^{d^{\diamond}}\EE \bracks{  \alpha_i(\pi(\bar z_h)) u_h(\bar z_h,a_h,r_h,o_{h+1};\theta); a_{1:h-1} \sim \pi', a_h \sim do(a^i) }. 
\end{align*}
Thus, we can perform policy evaluation for a policy $\pi$ if we can do intervention from $do(0),do(a^1),\cdots,do(a^{d^{\diamond}})$.

\paragraph{Fourth Step} The fourth step is replacing $do(0),do(a^1),\cdots,do(a^{d^{\diamond}})$ with a single policy that uniformly randomly select actions from the set $\{0, a^1,\dots, a^{d^\diamond}\}$, which we denote as $a\sim U(1+d^{\diamond})$.
Using importance weighting, we define the loss function for $\pi,\theta$ as follows:
\begin{align}\label{eq:loss_lqg}
    \EE[f_h(\bar z_h,a_h,r_h,o_{h+1};\theta,\pi) ; a_{1:h-1} \sim \pi', a_h \sim U(1+d^{\diamond})]
\end{align}
where $U(1+d^{\diamond})$ is a uniform action over $0,a^{1},\cdots,a^{d^{\diamond}}$ and 
\begin{align*}
   &f_h(\bar z_h,a_h,r_h,o_{h+1};\theta,\pi) \\
   &= |1+d^{\diamond}| \prns{ \II(a_h = 0 ) \prns{1- \sum_{i=1}^{d^{\diamond}} \alpha_i(\pi(\bar z_h))}  +\sum_{i=1}^{d^{\diamond}} \II(a_h = a^i)  \alpha_i(\pi(\bar z_h)) }u_h(\bar z_h,a_h,r_h,o_{h+1};\theta).  
\end{align*}
The term \ref{eq:loss_lqg}is equal to $\mathrm{Br}_h(\pi,g;\pi')$ we want to evaluate. 

\paragraph{Summary}

To summarize, we just need to use the following loss function in \pref{line:loss} in \pref{alg:OACP}: 
\begin{align*}
      \EE_{\Dcal^t_h}[l_h(\bar z_h,a_h,r_h,o_{h+1};\theta,\pi)]
\end{align*}
where $l_h(\bar z_h,a_h,r_h,o_{h+1};\theta,\pi)$ is 
\begin{align*}
  \II(\|\bar z_h\|\leq Z_1) \II(\|r_h\|\leq Z_2) \II(\|o_{h+1}\|\leq Z_3)f_h(\bar z_h,a_h,r_h,o_{h+1};\theta,\pi)
\end{align*}
and $\Dcal^t_h$ is a set of $m$ i.i.d samples following the distribution induced by executing $a_{1:h-1} \sim \pi', a_h \sim U(1+d^{\diamond})$. Values $Z_1,Z_2,Z_3$ in indicators functions are some large values selected properly later. Due to unbounded Gaussian noises in LQG, indicators functions for truncation is introduced here for technical reason to get valid concentration in Assumption \ref{assum:uniform}.

%% file: main_document/ape_psr.tex
\section{Supplement for \pref{sec:psr} }\label{sec:psrs}
We first add several discussions to explain core tests in detail. Next, we show the existence and form of link functions. Finally, we calculate the PO-bilinear rank. In this section, we will focus on the general case where tests could be multiple steps.

\subsection{Definition of PSRs}

We first define core tests and predictive states \citep{littman2001predictive,singh2004predictive}. This definition is a generalization of \pref{def:linear_psrs} with multi-step futures.  

We slighly abuse notation and  denote $\tau^a_h := (o_1,a_1,\dots, o_{h-1}, a_{h-1})$ throughout this whole section --- note that $\tau^a_h$ here does not include $o_h$.

\begin{definition}[Core test sets and PSRs]\label{def:psrs}
A set $ \Tcal \subset \cup_{C \in \mathbf{N}^+ } \Ocal^{C} \times \Acal^{C-1}$  is called a core test  set if for any $h \in [H]$, $W \in \mathbf{N}^+ $, any possible future (i.e., test) $t_h=(o_{h:h+W-1},a_{h:h+W-2}) \in \Ocal^W \times \Acal^{W-1}$ and any history $\tau^a_h $,  there exists $m_{t_h} \in \RR^{|\Tcal|}$ such that 
\begin{align*}
    \PP(o_{h:W+h-1} \mid \tau^a_h ;do(a_{h:W+h-2}) ) =  \langle m_{t_h}, [\PP(t \mid \tau^a_h ) ]_{t\in\Tcal} \rangle. 
\end{align*}
The vector $ [\PP(t \mid \tau^a_h ) ]_{t\in \Tcal} \in \RR^{|\Tcal|}$ is referred to as the predictive state. 
\end{definition}

We often denote $\mathbf{q}_{\tau^a_h } = [\PP(t \mid \tau^a_h) ]_{t\in \Tcal_h}$. To understand the above definition, we revisit observable undercomplete POMDPs and overcomplete POMDPs.  

\begin{example}[Observable undercomplete POMDPs] \label{exa:undercomplete_psrs}
In undercomplete POMDPs, when $\OO$ is full-column rank, $\Ocal$ is a core test. Recall $\OO$ is a matrix in $\RR^{|\Ocal|\times |\Scal|}$ whose entry indexed by $o_i \in \Ocal, s_j \in \Scal$ is equal to $\OO(o_i \mid s_j)$. 

\begin{lemma}[Core tests in undercomplete POMDPs] \label{lem:undercomplete_pomdps}
When $\OO$ is full-column rank, $\Ocal$ is a core test set. 
\end{lemma}
\begin{proof}
Consider any $h\in [H]$. Given a $|\Scal|$-dimensional belief state $\mathbf{s}_{\tau^a_h }= [\PP(\cdot \mid \tau^a_h)]_{|\Scal|}$ with each entry $\PP(s_h \mid \tau^a_h)$, for any future $t=(o_{h:h+W},a_{h:h+W-1})$, there exists a $|\Scal|$-dimensional vector $\mathbf{m'}_t$ such that  $\PP( o_{h:h+W}\mid \tau^a_h; do(a_{h:h+W-1})) = \langle \mathbf{m'}_t,\mathbf{s}_{\tau^a_h }\rangle$. More specifically, ${\bf m'}_t$ can be written as:
\begin{align*}
({\bf m'}_t)^{\top} =  \OO(o_{h+W} \mid{} \cdot)^{\top} \prod_{\tau = h}^{h+W-1}  \TT_{a_h} \diag(\OO( o_h \mid{} \cdot )) 
\end{align*}
where $\OO(o|\cdot) \in \RR^{|\Scal|}$ is a vector with the entry indexed by $s$  equal to $\OO(o | s)$, $\TT_a \in \RR^{|\Scal|\times |\Scal| } $ is a matrix with the entry indexed by $(s,s')$ equal to $\TT(s'\mid s,a_h)$. Here, note given a vector $C$, $\diag(C)$ is define as a $|C|\times |C|$ diagonal matrix where the diagonal element corresponds to $C$. 
Thus, we have 
\begin{align*}
    \PP( o_{h:h+W}\mid \tau^a_h; do(a_{h:h+W-1})) = \langle \mathbf{m'}_t,\mathbf{s}_{\tau^a_h }\rangle = \langle \mathbf{m'}_t, \OO^{\dagger} \mathbf{q}_{\tau^a_h }\rangle = \langle (\OO^{\dagger} )^{\top}\mathbf{m'}_t, \mathbf{q}_{\tau^a_h }\rangle,
\end{align*} where ${\bf q}_{\tau^a_h}\in \RR^{|\Ocal|}$ and ${\bf q}_{\tau^a_h}(o) = \PP(o | \tau^a_h)$.
This concludes the proof. 
\end{proof}
\end{example}

\begin{example}[Overcomplete POMDPs]\label{exa:overcomplete_psrs}

We consider overcomplete POMDPs so that we can permit $|\Scal|\geq |\Ocal|$. 

\begin{lemma}[Core tests in overcomplete POMDPs] 
Recall $\Tcal^K =\Ocal  \times ( \Ocal\times\Acal )^{K-1}$. 
Define a $|\Tcal^K | \times |\Scal|$-dimensional matrix $\OO^K$ whose entry indexed by $(o_{h:h+K-1},a_{h:h+K-2}) \in \Tcal^K, s_h \in \Scal$ is equal to  $\PP(o_{h:h+K-1},a_{h:h+K-2} \mid s_h;a_{h:h+K-2}\sim \Ucal(\Acal))$. When this matrix is full-colmun rank for all $h$, $\Tcal^K$ is a core test set. 
\end{lemma}
\begin{proof}
Fix a test $t = ( o'_{h:h+K-1},a'_{h:h+K-2})$ and consider a step $h\in [H]$. Then, 
\begin{align*}
    & \PP(o'_{h:h+K-1},a'_{h:h+K-2} \mid s_h; a_{h:h+K-2}\sim \Ucal(\Acal)) \\
    &=\EE[ \one(o_{h:h+K-1}= o'_{h:h+K-1}, a_{h:h+K-2}=a'_{h:h+K-2}) \mid s_h; a_{h:h+K-2}\sim \Ucal(\Acal))]\\
 &=\EE[ (1/|\Acal|^{K-1})\one(o_{h:h+K-1}= o'_{h:h+K-1}, a_{h:h+K-2}=a'_{h:h+K-2}) \mid s_h; a_{h:h+K-2}\sim  do(a'_{h:h+K-2}) ] \\
  &=\EE[ (1/|\Acal|^{K-1})\one(o_{h:h+K-1}= o'_{h:h+K-1}) \mid s_h; a_{h:h+K-2}\sim  do(a'_{h:h+K-2}) ] \\
  &= (1/|\Acal|^{K-1})\PP( o'_{h:h+K-1}\mid s_h; do(a'_{h:h+K-2}) ). 
\end{align*}
Thus, the assumption that $\Tcal^K$ is full column rank implies that that the matrix $\mathbb{\bar J}_h\in \RR^{ |\Tcal^K| \times |\Scal|}$ with the entry indexed by $(t,s_h)$ being equal to $\PP( o'_{h:h+K-1}\mid s_h; do(a'_{h:h+K-2}) )$  is full-column rank. 

Define a $|\Tcal^K|$-dimensional state $\mathbf{q}_{\tau^a_h}=[\PP(t \mid \tau^a_h)]_{t\in \Tcal^K }$ given history $\tau^a_h$. By definition, we have
\begin{align*}
    \mathbf{q}_{\tau^a_h}=\mathbb{\bar J}_h \mathbf{s}_{\tau^a_h}
\end{align*}
Using $\mathbb{\bar J}_h$ is full-column rank, we have $\mathbf{s}_{\tau^a_h} = \bar M^{\dagger}_h \mathbf{q}_{\tau^a_h}$. Thus, using the format of ${\bf m'}_t$ from the proof of \pref{lem:undercomplete_pomdps},  we can conclude that for any test $t = (o_{h:h+W}, a_{h:h+W-1})$, we have $\PP(o_{h:h+W} | \tau; \text{do}(a_{h:h+W-1})) = \langle (\bar M^{\dagger}_h)^{\top} {\bf m'}_t,  {\bf q}_{\tau^a_h} \rangle$.  Thus, 
this concludes $\Tcal^K$ is a core test set. 
\end{proof}
\end{example}

Finally, we present an important property of predictive states, which corresponds to the Bayesian filter in POMDP.

\begin{lemma}[Forward dynamics of predictive states]\label{lem:forward_dynamics}
We have 
\begin{align*}
    \PP(t \mid \tau^a_h ,a,o) = \mathbf{m}^{\top}_{o,a,t} \mathbf{q}_{\tau^a_h }/ \mathbf{m}^{\top}_{o} \mathbf{q}_{\tau^a_h }. 
\end{align*}
When we define $M_{o,a} \in \RR^{|\Tcal|\times |\Tcal|}$ where rows are $\mathbf{m}_{o,a,t}$ for $t\in \Tcal$, we can express the forward update rule of predictive states as follows: 
\begin{align*}
     \mathbf{q}_{\tau^a_h ,a,o} = M_{o,a} \mathbf{q}_{\tau^a_h }/  (\mathbf{m}^{\top}_{o}\mathbf{q}_{\tau^a_h }). 
\end{align*}
\end{lemma} 
\begin{proof}
The proof is an application of Bayes's rule. We denote the observation part of $t$ by $t^{\Ocal}$ and the action part of $t^{\Acal}$, respectively. We have 
\begin{align*}
    \PP(t \mid \tau^a_h ,a,o) &=     \PP(t^{\Ocal} \mid \tau^a_h ,o;  \text{do}(a,t^{\Acal}))  \tag{by definition}\\
    &=    \frac{\PP(o,t^{\Ocal} \mid \tau^a_h ; \text{do}(a,t^{\Acal}) )}{\PP(o;\tau^a_h ) }  \tag{Bayes rule}\\
    & = \mathbf{m}^{\top}_{o,a,t} \mathbf{q}_{\tau^a_h }/ \mathbf{m}^{\top}_{o} \mathbf{q}_{\tau^a_h }. \tag{by definition}
\end{align*} 
This concludes the proof. 
\end{proof}

To further understand that why PSR generalizes POMDP, let us re-visit the undercomplete POMDPs (i.e., $\OO$ being full column rank) again. Set $\Tcal = \Ocal$. 
As we see in the proof of \pref{lem:undercomplete_pomdps}, the belief state ${\bf s}_{\tau}\in \Delta(\Scal)$ together with $\OO$ defines predictive state, i.e., ${\bf q}_{\tau^a_h} = \OO {\bf s}_{\tau^a_h}$, with $M_{o,a} =    \OO   \TT_{a} \diag(\OO(o | \cdot )) \OO^\dagger$, 
and ${\bf m}^\top_{o} =    \one^{\top} \diag(\OO(o|\cdot ))  \OO^\dagger$.
Note that in POMDPs, matrix $M_{o,a}$ and vector $\mathbf{m}_{o}$ all contain non-negative entries. On other hand, in PSRs, $M_{a,o}$ and $\mathbf{m}_{a,o}$ could contain negative entries. This is the intuitive reason why PSRs are more expressive than POMDPs \citep{littman2001predictive}. For the formal instance of a finite-dimensional PSR which cannot be expressed as a finite-dimensional POMDP, refer to \citep{singh2004predictive,jaeger1998discrete}.

\subsection{Existence of link functions}

We discuss the existence and the form of link functions. First, we define general value link functions with multi-step futures. For notational simplicity, we assume here that the tests $t\in \Tcal$ have the same length, i.e., there is a $K\in \NN^+$, such that $\Tcal \subset \Ocal^K \times \Acal^{K-1}$.

\begin{definition}[General value link functions in dynamical systems]
Recall $\Tcal\subset \Ocal^K \times \Acal^{K-1}$ is the set of tests. 
At time step $h$, general value link functions $g^{\pi}_h:\Zcal_{h-1} \times \Ocal^{K}\Acal^{K-1} \to \RR$ %
are defined as solutions to the following: 
\begin{align}\label{eq:new_definition_appe}
\Vcal^{\pi}_h(\tau^a_h ) =  \EE[g^{\pi}_h(z_{h-1}, o_{h:h+K-1},a_{h:h+K-2})  \mid \tau^a_h ; (a_{h:h+K-2})  \sim \rho^{out} ]. 
\end{align} 
where $\rho^{out} $ is some distribution over the action set $\Tcal^{\Acal}$ induced by the test set, i.e., $\{t^{\Acal}: t\in \Tcal \}$. Here, for $t= (o_{h:h+K-1},a_{h:h+K-2} )$, we often denote $o_{h:h+K-1}$ and $a_{h:h+K-2}$   by  $t^{\Ocal}$ and $t^{\Acal}$, respectively.   %
\end{definition}

To show the existence of general value link functions for PSRs, we first study the format of value functions in PSRs. 
The following lemma states that value functions for $M$-memory policies have bilinear forms. 

\begin{lemma}[Bilinear form of value functions for $M$-memory policies]
Let $\phi(\cdot) \in \RR^{|\Zcal_{h-1}| }$ be a one-hot encoding vector over $\Zcal_{h-1}$. Suppose $\Tcal$ is a core test set. Then, for any $M$-memory policy $\pi$, there exists $\mathbb{J}^{\pi}_h \in \RR^{|\Zcal_{h-1}|\times |\Tcal|}$ such that 
\begin{align*}
\Vcal^{\pi}_h(\tau^a_h)= \phi^{\top}(z_{h-1})\mathbb{J}^{\pi}_h {\bf q}_{\tau^a_h}.
\end{align*}\label{lem:value_function_bilinear}
\end{lemma}

\begin{proof}
From \pref{lem:forward_dynamics}, there exists a matrix $M_{o,a} \in \RR^{|\Tcal  |\times |\Tcal |  }$ such that via Bayes rule:
\begin{align}\label{eq:key}
{\bf q}_{\tau^a_h,a,o} = M_{o,a} {\bf q}_{\tau^a_h} / \PP(o | \tau^a_h).
\end{align}

We use induction to prove the claim. Here, the base argument clearly holds. Thus, we assume 
\begin{align*}
\Vcal^{\pi}_{h+1}(\tau^a_{h+1})= \phi^{\top}(z_{h})\mathbb{J}^{\pi}_{h+1} {\bf q}_{\tau^a_{h+1}}.
\end{align*}
We have %

\begin{align*}
\Vcal^{\pi}_h(\tau^a_h)    &= \EE[r_h + \Vcal^{\pi}_{h+1}(\tau^a_h , o_h, a_h) \mid \tau^a_h ;  a_h \sim \pi(\bar z_h)] \\
     &= \underbrace{\sum_{o_h,a_h}\PP(o_h \mid \tau^a_h) \pi_h(a_h\mid o_h, z_{h-1}) r(o_h,a_h)}_{(a)}  \\  
     &+ \underbrace{\sum_{o_h,a_h} \PP(o_h \mid \tau^a_h)\pi_h(a_h\mid o_h, z_{h-1})  \{\phi^{\top}(z_{h})\mathbb{J}^{\pi}_{h+1} {\bf q}_{\tau^a_h,o_h,a_h} \} }_{(b)}. 
\end{align*}

Note we use the assumption that the reward is a function of $o_h,a_h$ conditional on $(\tau^a_h,o_h,a_h)$. 

We first check the first term (a) that contains rewards. Using the fact that $\PP(o | \tau^a_h) = {\bf m}_o^\top {\bf q}_{\tau^a_h}$, this is equal to 
\begin{align*}
    \sum_{o_h,a_h}\langle {\bf m}_{o_h},  \mathbf{q}_{\tau^a_h} \rangle \pi_h(a_h\mid o_h, z_{h-1}) r(o_h,a_h) = \langle \sum_{o_h,a_h} {\bf m}_{o_h}\pi_h(a_h\mid o_h, z_{h-1}) r(o_h,a_h) , \mathbf{q}_{\tau^a_h} \rangle. 
\end{align*}
Thus, it has a bilinear form, i.e., there exists some matrix $\mathbb{J}^{\pi}_1$ such that 
$$ \langle \sum_{o_h,a_h} m_{o_h,a_h}\pi_h(a_h\mid o_h, z_{h-1}) r_h , \mathbf{q}_{\tau^a_h} \rangle= \phi^{\top}(z_{h-1})\mathbb{J}^{\pi}_1\mathbf{q}_{\tau^a_h}$$ 
where $\mathbb{J}^{\pi}_1$ is a matrix whose row indexed by $z_{h-1}$ is equal to $\sum_{o, a}{\bf m}_{o }^{\top} \pi_h(a | o, z_{h-1}) r(o,a)$.

Next, we see the second term (b). Using \eqref{eq:key},  the second term is equal to 
\begin{align*}
    \sum_{o_h,a_h}  \pi_h(a_h\mid o_h, z_{h-1})\phi^{\top}(z_{h-1} \oplus o_h,a_h)\mathbb{J}^{\pi}_{h+1} M_{o_h,a_h} \mathbf{q}_{\tau^a_h}
\end{align*} where we use the notation $z_{h-1} \oplus o, a$ to represent the operation of appending $(o,a)$ pair to the memory while maintaining the proper length of the memory by truncating away the oldest observation-action pair. 
Thus, it has an again bilinear form $ \phi(z_{h-1})^{\top} \mathbb{J}^{\pi}_2 {\bf q}_{\tau^a_h}$ and the matrix $\mathbb{J}^{\pi}_2$ can be defined such that its row indexed by $z_{h-1}$ is equal to $\sum_{o,a} \pi_h(a | o, z_{h-1})  \phi^\top(z_{h-1} \oplus o,a) M^\pi_{h+1} M_{a,o}$. 
This concludes the proof. 
\end{proof}

Next, we check sufficient conditions to ensure the existence of general K-step link functions. Given $\Tcal$, we define the corresponding set of action sequences $\Tcal^{\Acal}$ as $\Tcal^{\Acal} := \{ t^{\Acal}:t\in\Tcal \}$. We set $\rho^{out}$ in \pref{eq:new_definition_appe} to be a uniform distribution over the set $\Tcal^{\Acal}$ denoted by $\Ucal(\Tcal^{\Acal})$. Namely, $\Ucal(\Tcal^{\Acal})$ will uniformly randomly select a sequence of test actions from $\Tcal^{\Acal}$. 

\begin{lemma}[Existence of link functions in PSRs] \label{lem:existence_psr}
Suppose  $\Tcal$ is a core test. 
There exists $g^{\pi}_h:\Zcal_{h-1} \times \Tcal$ such that 
\begin{align*}
\EE[g^{\pi}_h(z_{h-1},  o_{h:h+K-1},  a_{h:h+K-2}) \mid \tau^a_h;  a_{h:h+K-2} \sim \Ucal(\Tcal^{\Acal}) ]      = \Vcal^{\pi}_h(\tau^a_h).
\end{align*}
\end{lemma}
\begin{proof}
We mainly need to design an unbiased estimator of the predictive state ${\bf q}_{\tau^a_h}$. We use importance weighting to do that. Given $a_{h:h+K-2}\sim  \Ucal(\Tcal^{\Acal}) $, and the resulting corresponding random observations $o_{h:h+K-1}$, we define the following estimator $\hat {\bf q}_{\tau^a_h}(o_{h:h+K-1},a_{h:h+K-2}) \in \RR^{ |\Tcal| }$, such that its entry indexed by a test $t\in\Tcal$ is equal to:
\begin{align*}
\hat {\bf q}_{\tau^a_h}(o_{h:h+K-1},a_{h:h+K-2})[t] = \frac{ \one( t^{\Ocal} = o_{h:h+K-1}, t^\Acal = a_{h:h+K-2} )}{ 1/ |\Tcal^\Acal | }.
\end{align*} 
We can verify that 
\begin{align*}
    &\EE[ \hat {\bf q}_{\tau^a_h}(o_{h:h+K-1},a_{h:h+K-2})[t] \mid \tau^a_h; a_{h:h+K-2}\sim  \Ucal(\Tcal^{\Acal}) ]\\
    &= 1/ |\Tcal^\Acal |\EE[\one( t^{\Ocal} = o_{h:h+K-1}, t^\Acal = a_{h:h+K-2} ) \mid \tau^a_h; a_{h:h+K-2}\sim  \Ucal(\Tcal^{\Acal}) ] \\
    &= \EE[\one( t^{\Ocal} = o_{h:h+K-1}, t^\Acal = a_{h:h+K-2} ) \mid \tau^a_h; a_{h:h+K-2}\sim do(t^\Acal) ] = {\bf q}_{\tau^a_h}[t]. 
\end{align*}
Then, 
\begin{align*}
\EE[ \hat {\bf q}_{\tau^a_h}(o_{h:h+K-1},a_{h:h+K-2})\mid \tau^a_h; a_{h:h+K-2}\sim  \Ucal(\Tcal^{\Acal}) ] = {\bf q}_{\tau^a_h}. 
\end{align*}
With this estimator, now we can define the link function using the bilinear form of $\Vcal^\pi_h(\tau)$, i.e., 
\begin{align*}
g^\pi_h(z_{h-1}, o_{h:h+K-1}, a_{h:h+K-2}) = \phi(z_{h-1})^{\top} \mathbb{J}^{\pi}_h {\hat {\bf q}_{\tau^a_h}}(o_{h:h+K-1},a_{h:h+K-2}).
\end{align*} Using the fact that $\hat {\bf q}_{\tau^a_h}(o_{h:h+K-1},a_{h:h+K-2})$ is an unbiased estimate of ${\bf q}_{\tau^a_h}$, we can conclude the proof. 
\end{proof}

Since PSR models capture POMDP models, our above result directly implies the existence of the link functions in observable POMDPs as well by using obtained facts in Example~ \ref{exa:undercomplete_psrs} and \ref{exa:overcomplete_psrs} .

\subsection{PO-Bilinear Rank Decomoposition}

Finally, we calculate the PO-bilinear rank.  Here, $$g_h \in \{\Zcal_{h-1} \times (\Ocal^K\Acal^{K-1}) \ni (z_{h-1}\times t) \mapsto \phi(z_{h-1})^{\top} \mathbb{J}^{\pi}_h  {\hat {\bf q}_{\tau^a_h}}(t) \in \RR: \mathbb{J}^{\pi}_h\in \mathbb{R}^{\Zcal_{h-1}\times  |\Tcal|}\}.$$ 

The Bellman error for $(g,\pi)$ under a roll-in $\pi'$ denoted by $   \mathrm{Br}_h(g,\pi;\pi')$ is defined as 
\begin{align*}
    &-\EE[\EE[g_{h+1}(z_{h},t^{\Acal}_{h+1},t^{\Ocal}_{h+1}) \mid \tau^a_{h+1};  t^{\Acal}_{h+1} \sim U(\Tcal^{\Acal}_{h+1})  ] +r_h ; a_{1:h-1}\sim \pi', a_h \sim \pi  ] \\
    &+ \EE[g_h(z_{h-1},t^{\Acal}_h,t^{\Ocal}_h) \mid \tau^a_h   ;  t^{\Acal}_h \sim U(\Tcal^{\Acal}_h)  ];a_{1:h-1}\sim \pi', a_h \sim \pi  ]. 
\end{align*}
In fact, $\mathrm{Br}_h(g,\pi;\pi')=0$ for any general value link functions $g^{\pi}$.  

Our goal is to design a loss function $l_h(\cdot)$ such that we can estimate the above Bellman error $\mathrm{Br}_h(g, \pi;\pi')$ using data from a \emph{single} policy. To do that, we design the following randomized action selection strategy.  

Given a action sequence $t^\Acal$ from a test $t$, let us denote $\bar t^{\Acal}$ as a copy of $t^\Acal$ but starting from the second action of $t^\Acal$, i.e., if $t^\Acal = \{a_1, a_2, a_3\}$, then $\bar t^\Acal = \{a_2, a_3\}$. Denote $\bar \Tcal^\Acal = \{\bar t^\Acal: t \in \Tcal \}$. Our random action selection strategy first selects $a_h \sim U(\Acal)$ uniformly randomly from $\Acal$, and then select a sequence of actions $\bar{ \bf a}$ uniformly randomly from $\Tcal^\Acal \cup \bar \Tcal^\Acal$. Here, we remark the length of outputs is not fixed (i.e., $\bar {\bf a} \in \Tcal^\Acal$ has length larger than the $\bar {\bf a} \in \bar \Tcal^\Acal$). 

As a first step, we define two unbiased estimators for ${\bf q}_{\tau^a_h}$ and ${\bf q}_{\tau^a_{h+1}}$. Conditioning on history $\tau^a_h $, given actions $a_h \sim U(\Acal)$ followed by action sequence $\bar {\bf a}_{h+1} \sim U(\Tcal^\Acal\cup \bar \Tcal^\Acal)$, denote the corresponding observations as $o_h, o_{h+1}, \dots o_{h + |\bar {\bf a}_{h+1}| + 1}$. We construct unbiased estimators for ${\bf q}_{\tau^a_h}$ and ${\bf q}_{\tau^a_{h+1}}$ as follows. As an unbiased estimator of ${\bf q}_{\tau^a_h}$, we define $\hat {\bf q}_{\tau^a_h}$ with the entry indexed by test $t'\in \Tcal$ as follows: 
\begin{align}\label{eq:length1}
\hat {\bf q}_{\tau^a_h}( a_h, \bar {\bf a}_{h+1}, o_{h:h+|\bar {\bf a}_{h+1}| + 1} )[t'] = \frac{ \one(  \bar{\bf a}_{h+1} \in \bar \Tcal^\Acal, (a_h,\bar {\bf a}_{h+1}) = t'^\Acal, o_{h:h+|\bar {\bf a}_{h+1}| + 1} = t'^\Ocal  )   }{ 1/(2 |\Acal| |\Tcal^\Acal | ) }.
\end{align}
Similarly, as an unbiased estimator of ${\bf q}_{\tau^a_{h+1}}$, we define $\hat {\bf q}_{\tau^a_{h+1}}$ with the entry indexed by test $t' \in\Tcal$ as follows:
\begin{align}\label{eq:length2}
\hat {\bf q}_{\tau^a_{h+1}}(a_h, \bar{\bf a}_{h+1}, o_{h+1:h+|\bar {\bf a}_{h+1}| + 1} )[t'] = \frac{\one( \bar {\bf a}_{h+1} \in \Tcal^\Acal, \bar {\bf a}_{h+1} = t'^\Acal, o_{h+1:h+|\bar {\bf a}_{h+1}| + 1} = t'^{\Ocal} )}{ 1/ (2 |\Tcal^\Acal | ) }
\end{align}
We remark the length of $\bar {\bf a}$ in \pref{eq:length1} and the one of \pref{eq:length2} are different.

Then, by using importance sampling, we can verify  %
\begin{align*}
    \EE[  \hat {\bf q}_{\tau^a_h}( a_h, \bar {\bf a}_{h+1}, o_{h:h+|\bar {\bf a}_{h+1}| + 1} ) |\tau^a_h; a_h \sim U(\Acal), \bar {\bf a}_{h+1} \sim U(\Tcal^{\Acal}  \cup \bar\Tcal^\Acal)   ] & = {\bf q}_{\tau^a_h}, \\
   \EE[  \hat {\bf q}_{\tau^a_{h+1}}( \bar {\bf a}_{h+1}, o_{h+1:h+|\bar {\bf a}_{h+1}| + 1} )  | \tau^a_{h+1}; \bar {\bf a}_{h+1} \sim U(\Tcal^{\Acal}  \cup \bar\Tcal^\Acal)] &= {\bf q}_{\tau^a_{h+1}}. 
\end{align*}

 With the above setup, we can construct the loss function $l$ for estimating the Bellman error.  We set the loss as follows:
 \begin{align}\label{eq:loss_psr}
 &l_h( {z}_{h-1}, a_h, r_h, \bar{\bf a}_{h+1}, o_{h:h+|\bar {\bf a}_{h+1} | + 1} ; \pi, g ) \\
 &=  \phi(z_{h-1})^\top \mathbb{J}_h \hat {\bf q}_{\tau^a_h}(a_h, \bar{\bf a}_{h+1}, o_{h:h+|\bar{\bf a}_{h+1}|+1})  \nonumber \\
 &-\frac{ \one\{ a_h = \pi_h( \bar z_h )\}  }{1/|\Acal|} \left( r_h + \phi(z_h)^{\top} \mathbb{J}_{h+1} \hat {\bf q}_{\tau^a_{h+1}}( \bar {\bf a}_{h+1}, o_{h+1:h+|\bar{\bf a}_{h+1}| + 1})\right). \nonumber
 \end{align}
Since we have shown that $\hat {\bf q}_{\tau^a_h}$ and $\hat {\bf q}_{\tau^a_{h+1}}$ are unbiased estimators of ${\bf q}_\tau$ and ${\bf q}_{\tau^a_{h+1}}$, respectively, we can show that for any roll-in policy $\pi'$:
\begin{align*}
&\mathrm{Br}_h(\pi, g; \pi')\\
&= - \EE[\EE[g_{h+1}(z_{h},t^{\Acal}_{h+1},t^{\Ocal}_{h+1}) \mid \tau^a_{h+1};  t^{\Acal}_{h+1} \sim U(\Tcal^{\Acal}_{h+1})  ] + r_h ;a_{1:h-1}\sim \pi', a_h \sim \pi  ] \\
    &+ \EE[g_h(z_{h-1},t^{\Acal}_h,t^{\Ocal}_h) \mid \tau^a_h   ;  t^{\Acal}_h \sim U(\Tcal^{\Acal}_h)  ];a_{1:h-1}\sim \pi', a_h \sim \pi  ]  \\
&= \EE[ -\phi(z_{h})^{\top} \mathbb{J}^{\pi}_{h+1} {\bf q}_{\tau^a_{h+1}} -r_h +\phi(z_{h-1})^{\top} \mathbb{J}^{\pi}_{h} {\bf q}_{\tau^a_h } ;a_{1:h-1}\sim \pi', a_h \sim \pi   ]  \\ 
&=\EE\left[  l_h( z_{h-1}, a_h, r_h, \bar{\bf a}_{h+1}, o_{h:h+|\bar{\bf a}_{h+1}| +1}; \pi, g)     ; a_{1:h-1} \sim \pi',  a_h \sim U(\Acal), \bar{\bf a}_{h+1} \sim U(\Tcal^\Acal \cup \bar \Tcal^\Acal)    \right]. 
\end{align*}
The above shows that we can use $l_h(\cdot)$ as a loss function. 

\paragraph{Summary} We can use the almost similar algorithm as \pref{alg:OACP}. The sole difference is we need to replace $\sigma^t_h(\pi,g)$ with 
\begin{align*} 
    \EE_{\Dcal^t_h}\left[ l_h( z_{h-1}, a_h, r_h, \bar{\bf a}_{h+1}, o_{h:h+|\bar{\bf a}_{h+1}| +1}; \pi, g)     ; a_{1:h-1} \sim \pi',  a_h \sim U(\Acal), \bar{\bf a}_{h+1} \sim U(\Tcal^\Acal \cup \bar \Tcal^\Acal)    \right] 
\end{align*}
where $\Dcal^t_h$ is an empirical approximation when executing $a_{1:h-1} \sim \pi^t,  a_h \sim U(\Acal), \bar{\bf a}_{h+1} \sim U(\Tcal^\Acal \cup \bar \Tcal^\Acal) $. 

\paragraph{Calculation of PO-bilinear rank}

Finally,  we prove a PSR belongs to the PO-bilinear class. 
\begin{lemma}[PO-bilinear decomposition]
Let $\Qcal$ be a minimum core test set contained in $\Tcal$. The PSR model has PO-bilinear rank at most $|\Ocal|^M |\Acal|^M|\Qcal|$, i.e., there exists two $|\Ocal|^M |\Acal|^M |\Qcal|$-dimensional mappings $W_h:\Pi \times \Gcal \to \RR^{{|\Ocal|^M |\Acal|^M|\Qcal|}}$ and $X_h:\Pi \to  \RR^{{|\Ocal|^M |\Acal|^M|\Qcal|}}$ such that for any tripe $(\pi, g; \pi')$, we have:
\begin{align*}
\mathrm{Br}_h(\pi, g; \pi') &=\EE\left[   \phi(z_{h-1})^\top \mathbb{J}_h {\bf q}_{\tau^a_h} - r_h - \phi(z_h)^\top \mathbb{J}_{h+1}  {\bf q}_{\tau^a_{h+1}}         ; a_{1:h-1} \sim \pi' , a_h \sim \pi \right]  \\ 
&= \left\langle  X_h(\pi'), W_h(\pi, g) \right\rangle. 
\end{align*}
\end{lemma}
\begin{proof}
We first take expectation conditional on $\tau^a_h$. Then, we have
\begin{align*}
 & \phi(z_{h-1})^\top \mathbb{J}_h {\bf q}_{\tau^a_h} -  \EE\left[  r_h +  \phi(z_h)^\top \mathbb{J}_{h+1}  {\bf q}_{\tau^a_{h+1}} \mid \tau^a_h; a_h \sim \pi\right] \\
 & =   \phi(z_{h-1})^\top \mathbb{J}_h {\bf q}_{\tau^a_h} + \left( \phi(z_{h-1})^\top \mathbb{J}^{\pi}_1 {\bf q}_{\tau^a_h} + \phi(z_{h-1})^\top  \mathbb{J}^{\pi}_2 {\bf q}_{\tau^a_h}  \right), 
\end{align*}
where $\mathbb{J}^{\pi}_1$ and $\mathbb{J}^{\pi}_2 $ are some two matrices as defined in the proof of \pref{lem:value_function_bilinear} from where we have already known that the $\pi$-induced Bellman backup on a value function which has a bilinear form  gives back a bilinear form value function.  Rearrange terms, we get:
\begin{align*}
 \phi(z_{h-1})^\top \mathbb{J}_h {\bf q}_{\tau^a_h} -  \EE \left[  r_h +  \phi(z_h)^\top \mathbb{J}_{h+1}  {\bf q}_{\tau^a_{h+1}} \mid \tau^a_h ; a_h \sim \pi \right] =\left\langle  \phi(z_{h-1}), (\mathbb{J}_h + \mathbb{J}^{\pi}_1+ \mathbb{J}^{\pi}_2)  \bf{q}_{\tau^a_h}  \right\rangle.
\end{align*}

Now recall that the minimum core test set is $\Qcal\subset \Tcal$. The final step is to argue that ${\bf q}_{\tau}$ lives in a subspace whose dimension is $|\Qcal|$. Since $\Qcal$ is a core test set, by definition, we can express ${\bf q}_{\tau^a_h}$ using $[ \PP( t | \tau^a_h)]_{t\in \Qcal}$, i.e.,  
\begin{align*}
\exists K \in \RR^{ |\Tcal| \times |\Qcal | }, \quad  {\bf q}_{\tau^a_h} = K [ \PP( t | \tau^a_h)]_{t\in \Qcal},
\end{align*} where the row of $K$ indexed by $t\in \Tcal$ is equal to ${\bf k}_{t}$, where ${\bf k}_{t}$ is the vector that is used to predict $\PP(t | \tau^a_h) = {\bf k}_t^\top [ \PP( t | \tau^a_h)]_{t\in \Qcal}$ whose existences is ensured by the definition of PSRs.   
This implies that 
\begin{align*}
\left\langle  \phi(z_{h-1}), (\mathbb{J}_h + \mathbb{J}^{\pi}_1+ \mathbb{J}^{\pi}_2)  \bf{q}_{\tau^a_h}  \right\rangle & =  \left\langle  \phi(z_{h-1}), (\mathbb{J}_h + \mathbb{J}^{\pi}_1+ \mathbb{J}^{\pi}_2)   K  [ \PP( t | \tau^a_h)]_{t\in \Qcal}   \right\rangle\\
&= (\phi(z_{h-1})\otimes  [ \PP( t | \tau^a_h)]_{t\in \Qcal}), \text{vec}( (\mathbb{J}_h+\mathbb{J}^{\pi}_1+ \mathbb{J}^{\pi}_2) K) \rangle.
\end{align*}

Finally, we take expectation with respect to $\tau^a_h$ then we get $\mathrm{Br}_h(\pi, g; \pi') = \langle X_h(\pi'),  W_h(\pi, g)\rangle $ such that 
\begin{align*}
X_h(\pi') = \phi(z_{h-1}) \otimes  \EE[  [ \PP( t | \tau^a_h)]_{t\in \Qcal} ; a_{1:h-1} \sim \pi' ], \quad W_h(\pi, g) =  \text{vec}( (\mathbb{J}_h +\mathbb{J}^{\pi}_1+ \mathbb{J}^{\pi}_2) K). 
\end{align*}
\end{proof}

The key observation here is that the bilinear rank scales with $|\Qcal|$ but not $|\Tcal|$. This is good news since we often cannot identify exact minimal core test sets; however, it is easy to find core tests including minimal core tests. Thus, even if we do not know the linear dimension of a dynamical system a priori, the resulting bilinear rank is the linear dimension of dynamical systems as long as core sets are large enough so that they include minimal core tests. This will result in the benefit of sample complexity as we will see \pref{sec:sample_complexity_psrs}.

%% file: proof.tex
\section{Proof of \pref{thm:online}}

We fix the parameters as in \pref{thm:online}. Let 
\begin{align*}
    l_h(\bar z_h,a_h,r_h,o_{h+1}) = |\Acal| \pi_h(a_h \mid \bar z_h)\{r_h + g_{h+1}(\bar z_{h+1}) \} - g_h(\bar z_h). 
\end{align*}

We define
\begin{align*}
    \epsilon_{gen} &=\max_h \epsilon_{gen,h}(m,\Pi, \Gcal,\delta/(TH+1)), \quad \epsilon_{ini} =\epsilon_{ini}(\Gcal,\delta/(TH+1)) , \\
    \tilde \epsilon_{gen} &=\max_h \epsilon_{gen}(m,\Pi,\Gcal,\delta/H). 
\end{align*}

Then, by our assumption~\ref{assum:uniform}  with probability $1-\delta$, we
$\forall t \in [T], \forall h \in [H]$
\begin{align}\label{eq:conditional}
 & \sup_{\pi \in \Pi, g  \in \Gcal}|\EE_{\Dcal^t_h}[l_h(\bar z_h,a_h,r_h,o_{h+1};\pi,g) ] - \EE[\EE_{\Dcal^t_h}[l_h(\bar z_h,a_h,r_h,o_{h+1};\pi,g) ] ]|\leq  \epsilon_{gen}, \\
& \sup_{g_1 \in \Gcal_1}|\EE_{\Dcal^0}[g_1(o_1)]-\EE[\EE_{\Dcal^0}[g_1(o_1)] ]  |\leq \epsilon_{ini}. 
\end{align}
Hereafter, we condition on the above events. 

We first show the following lemma. Recall
\begin{align*}
    \pi^{\star} = \argmax_{\pi \in \Pi}J(\pi). 
\end{align*}

\begin{lemma}[Optimism] 
Set $R:=\epsilon^2_{gen}$. For all $t \in [T]$,  $(\pi^{\star},g^{\pi^{\star}})$ is a feasible solution of the constrained program. Furthermore, we have $ J(\pi^{\star}) \leq \EE[g^t_1(o_1)]+ 2\epsilon_{ini}$ for any $t \in [T]$, where $g^t$ is the value link function selected by the algorithm in iteration $t$.
\end{lemma}
\begin{proof}
For any $\pi$, we have 
\begin{align*}
     \EE[\EE_{\Dcal^t_h}[l_h(\bar z_h,a_h,r_h,o_{h+1};\pi,g^{\pi}) ]=0 
\end{align*}
since $g^{\pi} $ is a value link function in $\Gcal$. This is because
\begin{align*}
     &\EE[\EE_{\Dcal^t_h}[l_h(\bar z_h,a_h,r_h,o_{h+1};\pi,g^{\pi}) ]  \\
     &= \EE[g_h(\bar z_h) - r_h -g_{h+1}(\bar z_{h+1})  ;a_{1:h-1} \sim \pi^{t}, a_h \sim \pi ] \tag{IS sampling} \\ 
     &= \langle W_h(\pi, g^{\pi}), X_h(\pi^t)\rangle   \tag{First assumption in \pref{def:simple_bilinear}} \\
     &=0 \tag{Second assumption in \pref{def:simple_bilinear}}. 
\end{align*}

Thus, 
\begin{align*}
    |\EE_{\Dcal^t_h}[l_h(\bar z_h,a_h,r_h,o_{h+1};\pi^{\star},g^{\pi^{\star}}) ]|\leq \epsilon_{gen}. 
\end{align*}
using \eqref{eq:conditional} noting $\pi^{\star} \in \Pi, g^{\pi^{\star}}  \in \Gcal$. This implies $$\forall t \in [T],\forall h \in [H]; ( \EE_{\Dcal^t_h}[l_h(\bar z_h,a_h,r_h,o_{h+1};\pi^{\star},g^{\pi^{\star}}) ])^2 \leq \epsilon^2_{gen}.$$  
Hence, $(\pi^{\star},g^{\pi^{\star}})$ 
is a feasible set for any $t \in [T]$.

Then, we have 
\begin{align*}
    J(\pi^{\star}) &= \EE[g^{\pi^{\star}}_1(o_1) ]    \leq \EE_{\Dcal^0}[ g^{\pi^{\star}}_1(o_1) ]+\epsilon_{ini} \tag{Uniform convergence result} \\ 
      &\leq \EE_{\Dcal^0}[g^t_1(o_1)]+\epsilon_{ini} \tag{Using the construction of algorithm}  \\
      &\leq \EE[g^t_1(o_1)]+2\epsilon_{ini}. \tag{Uniform convergence} 
\end{align*}
\end{proof}

\begin{remark}
Note that 
\begin{align*}
     \EE[\EE_{\Dcal^t_h}[l_h(\bar z_h,a_h,r_h,o_{h+1};\pi,g^{\pi}) ]]=0 
\end{align*}
holds for general link functions $g^{\pi}$ in\pref{def:general_value} . Thus, the statement goes through even if we use \pref{def:general_value}. 

\end{remark}

Next, we prove the following lemma to upper bound the per step regret.  
\begin{lemma} \label{lem:performance_diff}
For any $t \in [T]$, we have 
\begin{align*}
       J(\pi^{\star}) - J(\hat \pi )  \leq \sum_{h=1}^H  |\langle  W_h(\pi^t ,g^t ),X_h(\pi^t )\rangle  |+  2\epsilon_{ini}. 
\end{align*}

\end{lemma}
\begin{proof}
\begin{align*}
    &J(\pi^{\star}) - J(\hat \pi )  & \\
    &\leq   2\epsilon_{ini} + \EE[g^{t}_1(o_1)]- J(\pi^t ) \tag{From optimism} \\ 
     &=  2\epsilon_{ini} + \sum_{h=1}^H \EE[ g^{t}_h(\bar z_h)-\{ r_h  +g^{t}_{h+1}(\bar z_{h+1}) \} ; a_{1:h} \sim \pi^t] \tag{Performance difference lemma} \\ 
     &\leq  2\epsilon_{ini} + \sum_{h=1}^H 
     |\EE[ g^{t}_h(\bar z_h)-\{ r_h  +g^{t}_{h+1}(\bar z_{h+1}) \} ; a_{1:h} \sim \pi^t ]| \\ 
     &= 2\epsilon_{ini} + \sum_{h=1}^H  |\langle  W_h(\pi^t,g^t),X_h(\pi^t)\rangle  |.  \tag{First assumption in \pref{def:simple_bilinear}}
\end{align*}

\end{proof}

\begin{lemma}\label{lem:squared_potential}
Let $\Sigma_{t,h} = \lambda I + \sum_{\tau=0}^{t-1} X_h(\pi^{\tau})X_h(\pi^{\tau})^{\top}$. We have 
\begin{align*}
     \frac{1}{T} \sum_{t=0}^{T-1}\sum_{h=1}^H  \|X_h(\pi^t)\|_{\Sigma^{-1}_{t,h}} \leq  H\sqrt{\frac{d}{T}\ln \left(1 + \frac{TB^2_{X}}{d\lambda} \right)}. 
\end{align*}
\end{lemma}
\begin{proof}
We fix $h \in [H]$. Here, we have
$\Sigma_{t,h} = \lambda I + \sum_{\tau=0}^{t-1} X_h(\pi^{\tau})X_h(\pi^{\tau})^{\top} $.
From the elliptical potential lemma in \citep[Lemma G.2]{agarwal2020pc}, we have 
\begin{align*}
   \frac{1}{T} \sum_{t=0}^{T-1} \|X_h(\pi^t)\|_{\Sigma^{-1}_{t,h}} \leq     \sqrt{\frac{1}{T}  \sum_{t=0}^{T-1}\|X_h(\pi^t)\|^2_{\Sigma^{-1}_{t,h}}} \leq \sqrt{\frac{1}{T}\ln \frac{\det (\Sigma_{t,h}) }{\det (\lambda I)}} \leq \sqrt{\frac{d}{T} \ln \left( 1 + \frac{T B^2_X}{d\lambda} \right)}. 
\end{align*}
Then, 
\begin{align*}
   \frac{1}{T} \sum_{t=0}^{T-1}\sum_{h=0}^H  \|X_h(\pi^t)\|^2_{ \Sigma^{-1}_{t,h}} \leq H \sqrt{\frac{d}{T} \ln \left( 1 + \frac{T B^2_X}{d\lambda} \right)}.
\end{align*}
\end{proof}

\begin{lemma}\label{lem:w_bound}
\begin{align*}
    \|W_h(\pi^t,g^t) \|^2_{\Sigma_{t,h}}\leq 2 \lambda B^2_W + 4T\epsilon^2_{gen}. 
\end{align*}
\end{lemma}
\begin{proof}
We have 
\begin{align*}
     \|W_h(\pi^t,g^t) \|^2_{\Sigma_{t,h}}  = \lambda \|W_h(\pi^t,g^t) \|^2_2 + \sum_{\tau =0}^{t-1} \langle W_h(\pi^{t},g^{t}), X_h(\pi^{\tau}) \rangle^2. 
 \end{align*}  
 The first term is upper-bounded by $\lambda B^2_W$. The second term is upper-bounded by 
\begin{align*}
&  \sum_{\tau =0}^{t-1} \langle W_h(\pi^{t},g^{t}), X_h(\pi^{\tau}) \rangle^2 \\ 
& = \sum_{\tau =0}^{t-1} \prns{ \EE[l_h( \bar z_h,a_h,r_h,o_{h+1} ; \pi^t ,g^t) ; a_{1:h-1} \sim \pi^{\tau}, a_h \sim U(\Acal)]}^2 \tag{First assumption in \pref{def:simple_bilinear}}\\
&\leq 2\sum_{\tau =0}^{t-1} \EE_{\Dcal^{\tau}_h }[l_h( \bar z_h,a_h,r_h,o_{h+1} ; \pi^t ,g^t) ]^2 + 2t\epsilon^2_{gen} \leq  4T\epsilon^2_{gen}. 
\end{align*}
From the first line to the second line, we use the definition of bilinear rank models. From the second line to the third line, we use $(a+b)^2 \leq 2a^2 + 2b^2$. In the last line, we use the constraint on $(\pi^t,g^t)$. 

\end{proof}

Combining lemmas so far, we have 
\begin{align*}
       J(\pi^{\star}) - J(\hat \pi )  & \leq \frac{1}{T}\sum_{t=0}^{T-1}\sum_{h=1}^H  |\langle  W_h(\pi^t ,g^t ),X_h(\pi^t )\rangle  |+  2\epsilon_{ini}  \tag{Use \pref{lem:performance_diff}}\\ 
       & \leq \frac{1}{T}\sum_{t=0}^{T-1}\sum_{h=1}^H  \|  W_h(\pi^t ,g^t )\|_{\Sigma_{t,h}} \|X_h(\pi^t )\|_{\Sigma^{-1}_{t,h}}+  2\epsilon_{ini} \tag{CS inequality} \\
       &\leq  H^{1/2} \bracks{2 \lambda B^2_W + 4T\epsilon^2_{gen}}^{1/2} \prns{ \frac{dH}{T} \ln\left(1 + \frac{T B^2_X}{d\lambda} \right)}^{1/2}        +2\epsilon_{ini}. \tag{Use \pref{lem:squared_potential} and \pref{lem:w_bound} }
\end{align*}
We set $\lambda$ such that  $B^2_X/\lambda = B^2_W B^2_X / \epsilon^2_{gen} +1 $ and $T= \ceil*{ 2Hd \ln(4Hd (B^2_XB^2_W/\tilde \epsilon_{gen} +1))} $. Then, 
\begin{align*}
    \frac{Hd}{T}\ln \left(1 + \frac{TB^2_{X}}{d\lambda} \right) &\leq  \frac{Hd}{T}\ln \left(1 + \frac{T}{d}\prns{ \frac{B^2_W B^2_X }{ \epsilon^2_{gen}} +1 }\right)
    \\   
    &\leq \frac{Hd}{T}\ln \left(1 + \frac{T}{d}\prns{ \frac{B^2_W B^2_X }{\tilde \epsilon^2_{gen}} +1 }\right)\\
    &\leq \frac{Hd}{T}\ln \left(\frac{2T}{d}\prns{\frac{B^2_W B^2_X}{\tilde \epsilon^2_{gen}} +1 }\right)\leq 1
\end{align*}
since $a \ln (bT)/T \leq 1$ when $T = 2a \ln(2ab)$. 

Finally, the following holds
\begin{align*}
    J(\pi^{\star})  - J(\pi^T ) &\leq  H^{1/2} \bracks{4 \lambda B^2_W + 8T\epsilon^2_{gen}}^{1/2} +2\epsilon_{ini} \\ 
       & \leq H^{1/2} \bracks{4 \lambda B^2_W + 16\epsilon^2_{gen}Hd \ln(4Hd (B^2_XB^2_W/\tilde \epsilon_{gen} +1))  }^{1/2} +2\epsilon_{ini}  \tag{Plug in $T$ }\\
        & \leq H^{1/2} \bracks{8 \epsilon^2_{gen} + 16\epsilon^2_{gen}Hd \ln(4Hd (B^2_XB^2_W/\tilde \epsilon_{gen} +1))  }^{1/2} +2\epsilon_{ini} \tag{Plug in $\epsilon_{gen}$ }\\
    & \leq 5\epsilon_{gen} \bracks{H^2 d \ln(4Hd (B^2_XB^2_W/\tilde \epsilon_{gen} +1))  }^{1/2} +2\epsilon_{ini}. 
\end{align*}

\section{Sample Complexity for Finite Function Classes}\label{sec:sample_complexity_finite}

Consider cases where $\Pi$ and $\Gcal$ are finite and the PO-bilinear rank assumption is satisfied. When $\Pi$ and $\Gcal$ are infinite hypothesis classes, $|\Fcal|$ and $|\Gcal|$ are replaced with their $L^{\infty}$-covering numbers, respectively. 

\begin{theorem}[Sample complexity for discrete $\Pi$ and $\Gcal$]\label{thm:finite_sample}
Let $\|\Gcal_h\|_{\infty} \leq C_{\Gcal}, r_h \in [0,1]$ for any $h \in [H]$ and the PO-bilinear rank assumption holds with PO-bilinear rank $d$. By letting $|\Pi_{\max}|=\max_h |\Pi_h|, |\Gcal_{\max}| = \max_h |\Gcal_h|$, with probability $1-\delta$, we can achieve $J(\pi^{\star})-J(\hat \pi)\leq \epsilon$ when we use samples at most 
\begin{align*}
    \tilde O\prns{ d_b H^4 \max(C_{\Gcal},1)^2 |\Acal|^2 \ln(|\Gcal_{\max}| |\Pi_{\max}|  /\delta)\ln^2(B_XB_W)(1/\epsilon)^2  }. 
\end{align*}
Here, $\mathrm{polylog}(d,H,|\Acal|,\ln(|\Gcal_{\max}|),\ln(|\Pi_{\max}|),\ln(1/\delta),\ln(B_X),\ln(B_W),\ln(1/\delta),(1/\epsilon))$ are omitted. 
\end{theorem}
\paragraph{Proof.}
We derive the above result.  First, we check the uniform convergence result. Then, 
\begin{align*}
    \epsilon_{gen}= c\max(C_{\Gcal},1)|\Acal| \sqrt{\ln(|\Gcal_{\max}| |\Pi_{\max}| TH /\delta)/m}. 
\end{align*}
Thus, we need to set $m$ such that
\begin{align*}
   J(\pi^{\star})-J(\hat \pi)\leq  c\max(C_{\Gcal},1)|\Acal| \sqrt{\ln(|\Gcal_{\max}| |\Pi_{\max}| T H  /\delta)/m}\sqrt{d H^2\ln (H^3 d B^2_X B^2_W m+1 ) }\leq \epsilon
\end{align*}
where $c$ is some constant and 
\begin{align*}
    T = c H d \ln(HdB^2_XB^2_W m +1). 
\end{align*}
By organizing the term, the following $m$ is sufficient 
\begin{align*}
     c\sqrt{\frac{d H^2 \max(C_{\Gcal},1)^2 |\Acal|^2\ln(|\Gcal_{\max}| |\Pi_{\max}| H^2d  /\delta) \ln (H^3 d B^2_X B^2_W m) }{m}}\leq \epsilon
\end{align*}
Using \pref{lem:sample_auxiliary}, the following $m$ satisfies the condition: 
\begin{align*}
   m= c \frac{B_1(\ln B_1B_2)^2}{\epsilon^2}, B_1= d H^2 \max(C_{\Gcal},1)^2 |\Acal|^2\ln(|\Gcal_{\max}| |\Pi_{\max}| H^2d  /\delta),B_2 = H^3 d B^2_X B^2_W.
\end{align*}
Combining all together, the sample complexity is $mTH$, i.e., 
\begin{align*}
    \tilde O\prns{ \frac{d^2 H^4 \max(C_{\Gcal},1)^2 |\Acal|^2 \ln(|\Gcal_{\max}| |\Pi_{\max}|  /\delta)\ln^2(B_XB_W) }{\epsilon^2 }   }. \quad  \blacksquare 
\end{align*}

\section{Sample Complexity in Observable HSE POMDPs}\label{sec:sample_hse_pomdps}
 
We revisit the existence of link functions by taking the norm constraint into account. Then, we consider the PO-bilinear decomposition with certain $B_X \in \RR$ and $B_W \in \RR$. Next, we calculate the uniform convergence result. Finally, we show the sample complexity result. 

We use the following assumptions 
\begin{assum}
For any $h \in [H]$, the following holds: 
\begin{enumerate}
    \item $V^{\pi}_h(z_{h-1},s) = \langle \theta^{\pi}_h, \phi_h(z_{h-1},s) \rangle.$
    \item There exists a matrix $K_h$ such that $\EE_{o \sim \OO(s)}[\psi_h(z_{h-1},o)] = K_h \phi_h(z_{h-1},s)$ (i.e., conditional embedding of the omission distribution),
    \item $\|\phi_h(\cdot) \| \leq 1, \|\psi_h(\cdot) \| \leq 1, \|\theta^{\pi}_h\|\leq \Theta_V, 0\leq r_h\leq 1,$ 
    \item There exists a matrix $T_{\pi;h}$ such that $\EE[\phi_h(z_{h},s_{h+1})\mid z_{h-1},s_h;a_h \sim \pi] = T_{\pi;h} \phi_h(z_{h-1},s_h)$ (i.e., conditional embedding of the transition)
    \item $\Pi$ is finite. 
\end{enumerate}
\end{assum}

We define 
\begin{align*}
    &\sigma_{\min}(K) = \min_{h\in [H]} 1/\|K_h ^{\dagger}\|,\  \sigma_{\max}(K) = \max_{h\in [H]} \|K_h\|, \  \sigma_{\max}(T) = \max_{h\in [H]} \|T_{\pi:h}\|, \\
    & d_{\phi} = \max_{h\in [H]} d_{\phi_h},\quad d_{\psi} = \max_{h\in [H]} d_{\psi_h}. 
\end{align*}

\paragraph{Existence of link functions.}

We show value link functions exist. This is proved by noting 
\begin{align*}
    \EE_{o \sim \OO(s)}[\langle (K^{\dagger}_h)^{\top}\theta^{\pi}_h, \psi_h(\bar z_h) \rangle ] = \langle (K^{\dagger}_h)^{\top} \theta^{\pi}_h, K_h  \phi_h(z_{h-1},s_h) \rangle = \langle \theta^{\pi}_h, \phi_h(z_{h-1},s_h) \rangle = V^{\pi}_h(z_{h-1},s_h). 
\end{align*}
Thus, $\langle (K^{\dagger}_h)^{\top}\theta^{\pi}_h, \psi_h(\bar z_h) \rangle $ is a value link function. The radius of the parameter space is upper-bounded by $\Theta_V/\sigma_{\min}(K)$. Hence, we set $$\Gcal_h = \{ \langle \theta, \psi_h(\cdot) \rangle: \|\theta\|\leq \Theta_V/\sigma_{\min}(K)  \}.$$
Then, the realizability holds. 

\paragraph{PO-bilinear decomposition.}

Recall we derive the PO-bilinear decomposition in Section \ref{subsec:hse_pomdps_ape}. Consider a triple $(\pi', \pi, g)$ with $g_h(\cdot) = \theta_h^\top \psi_h(\cdot)$ and $g^{\pi}_h = \langle \theta^{\star}_h,\psi_h(\cdot) \rangle$, we have:
\begin{align*}
& \EE \left[ \theta_h^{\top} \psi_{h+1}(\bar z_h) - r_h - \theta_{h+1}^\top \psi(\bar z_{h+1});a_{1:h-1}\sim \pi',a_h\sim \pi   \right] \\
& = \left\langle   \EE[\phi_h(z_{h-1}, s_h);a_{1:h-1}\sim \pi'], \quad  K_h^{\top} (\theta_h - \theta^{\star}_h) - T_{\pi;h}^\top K_{h+1}^\top (\theta_{h+1}-\theta^{\star}_{h+1})  \right\rangle,
\end{align*} which verifies the PO-bilinear structure, i.e., $$X_h(\pi') = \EE[\phi_h(z_{h-1}, s_h) ;a_{1:h-1}\sim \pi'],\quad W_h(\pi, g) =K_h^{\top} (\theta_h - \theta^{\star}_h) - T_{\pi;h}^\top K_{h+1}^\top (\theta_{h+1}-\theta^{\star}_{h+1}),$$ and shows that the PO-bilinear rank is at most $d_{\phi} = \max_{h} d_{\phi_h}$.
Thus, based on the above PO-bilinear decomposition, we set $\|B_X\| = 1, \|B_W\| = 2(1+ \sigma_{\max}(T))\sigma_{\max}(K)\Theta_V/\sigma_{\min}(K)$.  
This is because
\begin{align*}
   & \|  K_h^{\top} (\theta_h - \theta^{\star}_h) - T_{\pi;h}^\top K_{h+1}^\top (\theta_{h+1}-\theta^{\star}_{h+1})\| \\
   & \leq \|  K_h^{\top}\| (\|\theta_h\|  + \|\theta^{\star}_h\|) +  \|T_{\pi;h}^\top\|  \|K_{h+1}^\top\| (\|\theta_{h+1}\| + \|\theta^{\star}_{h+1})\|)\\
   &\leq  2(1+ \sigma_{\max}(T))\sigma_{\max}(K)\Theta_V/\sigma_{\min}(K). 
\end{align*}
and 
\begin{align*}
    \| \EE[\phi_h(z_{h-1}, s_h) ;a_{1:h-1}\sim \pi']\|\leq  \EE[\| \phi_h(z_{h-1}, s_h) \|;a_{1:h-1}\sim \pi']\leq 1. 
\end{align*}
In the above, we use Jensen's inequality. 

\paragraph{Uniform convergence.}

To invoke \pref{thm:online}, we show the uniform convergence result. 

\begin{lemma}[Uniform convergence of loss functions]\label{lem:uni_hse_pomdp}
Let  $C = \Theta_V/(\sigma_{\min}(K))$. Then, with probability $1-\delta$, 
{ 
\begin{align*}
   & \sup_{\pi \in \Pi, g \in \Gcal} \left| \{\EE_{\Dcal}-\EE\}\bracks{|\Acal| \pi_h(a_h \mid \bar z_h) \braces{ g_h(\bar z_h) - r_h - g_{h+1}(\bar z_{h+1})} } \right|  \\ 
   &\leq  5 |\Acal|\{1+2C\}\sqrt{ \frac{\{2d_{\psi}\ln(1+Cm )+\ln( |\Pi_{\max}| /\delta)\}}{m}} 
\end{align*}
}
and 
\begin{align*}
   \sup_{g_1 \in \Gcal_1}|\{\EE_{\Dcal}-\EE\}[g_1(\bar z_1)| \leq  5 C\sqrt{ \frac{\{ d_{\psi} \ln(1+Cm )+\ln( |\Pi_{\max}| /\delta)\}}{m}}. 
\end{align*}

\end{lemma}
\begin{proof}
Let $C = \Theta_V/\sigma_{\min}(K)$. Define $\Ncal_{\epsilon,h}$ as an $\epsilon$-net for $\Gcal_h$. Then, $|\Ncal_{\epsilon,h}| \leq (1 + C/\epsilon)^d $. Then, 
\begin{align*}
    |l_h(\cdot;\pi,g)-l_h(\cdot;\pi^{\diamond},g^{\diamond})|&\leq |\Acal|  \{\|g_h-g^{\diamond}_{h}\|_{\infty} + \|g_{h+1}-g^{\diamond}_{h+1}\|_{\infty} \}  \\ 
    &\leq |\Acal| \{ \|\theta_h - \theta^{\diamond}_h\|_2 + \|\theta_{h+1} - \theta^{\diamond}_{h+1}\|_2\} \leq 2|\Acal| \epsilon. 
\end{align*}
Besides, for fixed $\pi \in \Pi,\theta_h \in \Ncal_{\epsilon,h}$, $\theta_{h+1} \in \Ncal_{\epsilon,{h+1}}$, we have
{ \small 
\begin{align*}
      \left \lvert \{\EE_{\Dcal}-\EE\}\bracks{|\Acal| \pi_h(a_h \mid \bar z_h) \braces{ g_h(\bar z_h;\theta_h) - r_h - g_{h+1}(\bar z_{h+1};\theta_{h+1}) }}\right \rvert \leq  |\Acal|\prns{1 + 2C} \sqrt{ \frac{\ln(|\Pi_h|/\delta)}{m}}. 
\end{align*}
}
Then, for $\forall \pi \in \Pi, \forall \theta_h \in \Ncal_{\epsilon,h}$, $\forall \theta_{h+1} \in \Ncal_{\epsilon,{h+1}}$, we have
{ \small 
\begin{align*}
    \left \lvert \{\EE_{\Dcal}-\EE\}\bracks{|\Acal| \pi_h(a_h \mid \bar z_h) \braces{ g_h(\bar z_h;\theta_h) - r_h - g_{h+1}(\bar z_{h+1};\theta_{h+1}) }}\right \rvert  \leq  |\Acal|\prns{1 + 2C} \sqrt{ \frac{\ln(|\Pi_h||\Ncal_{\epsilon,h}||\Ncal_{\epsilon,h+1}| /\delta)}{m}}.
\end{align*}
}
Hence, for any $g_h =\langle \theta_h , \psi_h \rangle \in \Gcal_h, g_{h+1} = \langle \theta_{h+1}, \psi_{h+1} \rangle \in \Gcal_{h+1} $, 
{ 
\begin{align*}
    \left \lvert \{\EE_{\Dcal}-\EE\}\bracks{  |\Acal| \pi_h(a_h \mid \bar z_h) \braces{ g_h(\bar z_h;\theta_h) - r_h - g_{h+1}(\bar z_{h+1};\theta_{h+1}) }}\right \rvert\\  \leq  |\Acal|\prns{1 + 2C} \sqrt{ \frac{\ln(|\Pi_h||\Ncal_{\epsilon,h}||\Ncal_{\epsilon,h+1}| /\delta)}{m}} + 4 |\Acal| \epsilon. 
\end{align*}
}
By taking $\epsilon = 1/m$, we have  $\forall \pi \in \Pi,  \forall g_h \in \Gcal_h,\forall g_{h+1}\in \Gcal_{h+1}$:
\begin{align*}
   & |\{\EE_{\Dcal}-\EE\}[|\Acal| \pi_h(a_h \mid \bar z_h) \{g_h(\bar z_h) - r_h - g_{h+1}(\bar z_{h+1})\}]| \\
      &\leq  |\Acal|\{1+2C\}\sqrt{ \frac{\{ 2d\ln(1+Cm )+\ln( |\Pi_h| /\delta)\}}{m}} +  \frac{4|\Acal|}{m} \\ 
           &\leq  5 |\Acal|\{1+2C\}\sqrt{ \frac{\{ 2d\ln(1+Cm )+\ln( |\Pi_h| /\delta)\}}{m}}. 
\end{align*}
Similarly, 
\begin{align*}
   \forall g_1 \in \Gcal_1;  |\{\EE_{\Dcal}-\EE\}[g_1(\bar z_1)]| & \leq   C \sqrt{ \frac{\{ d \ln(1+Cm )+\ln(|\Pi_h| /\delta)\}}{m}} + \frac{4 }{m}\\
         &\leq  5 C\sqrt{ \frac{\{ d \ln(1+Cm )+\ln( |\Pi_h| /\delta)\}}{m}}. 
\end{align*}

\end{proof}

Finally, we obtain the PAC bound, we need to find $m$ such that 
{ 
\begin{align*}
c|\Acal|\max(C,1)\sqrt{\frac{d_{\psi} \ln(\max(C,1)m)+\ln(|\Pi_{\max}|TH/\delta)}{m}}\sqrt{d_{\phi}H^2 \ln \prns{Hd_{\phi} B^2_XB^2_W m +1  } }\leq \epsilon. 
\end{align*}
}
where $c$ is some constant and 
\begin{align*}
    T =c H d_{\phi} \ln(HdB^2_X B^2_Wm+1). 
\end{align*}
By organizing the term, the following $m$ is sufficient: 
{ 
\begin{align*}
c\sqrt{\frac{\{d_{\psi}+\ln(d_{\phi}|\Pi_{\max}|H^2/\delta) \}d_{\phi} H^2|\Acal|^2\max(C,1)^2 \ln(\{C+ Hd_{\phi} B^2_XB^2_W+1)\}m)^2}{m}}\leq \epsilon. 
\end{align*}
}

By using \pref{lem:sample_auxiliary}, we can set 
\begin{align*}
  &m =  \frac{B_1}{\epsilon^2}\ln(mB_1B_2)^2,\\
  &B_1 =\{ d_{\psi}  +\ln(d_{\phi}|\Pi_{\max}|H^2/\delta) \} d_{\phi} H^2|\Acal|^2\max(C,1)^2, B_2 = C+ Hd_{\phi} B^2_XB^2_W+1. 
\end{align*}
Thus, the final sample complexity is 
\begin{align*}
   \tilde O \prns{  \frac{ d^2_{\phi}\{d_{\psi}+\ln(|\Pi_{\max}|/\delta)\} H^4|\Acal|^2\max(C,1) ^2}{\epsilon^2 }  } 
\end{align*}
where  $C = \Theta_V /\sigma_{\min}(K)$. 

\section{Sample Complexity in Observable Undercomplete Tabular POMDPs}\label{sec:sample_complexity_tabular}

We revisit the existence of value link functions. Then, we show the PO-bilinear rank decomposition. After showing the uniform convergence lemma, we calculate the sample complexity. 

\paragraph{Existence of value link functions.}

In the tabular case, by setting
\begin{align*}
    \psi_h(z,o) = \One(z) \otimes \One(o), \phi_h(z,s) = \One(z) \otimes \One(s), K_h = \II_{|\Zcal_{h-1}|}\otimes \OO. 
\end{align*}
where $\One(z)$ is a one-hot encoding vector over $\Zcal_{h-1}$, we can regard the tabular model as an HSE-POMDP. Here is our assumption. 

\begin{assum}
(a) $0 \leq r_h \leq 1$, (b) $\OO$ is full-column rank and $\|\OO^{\dagger} \|_1 \leq 1/\sigma_1 $ for any $h\in [H]$.  
\end{assum}

Note we use the $1$-norm since this choice is more amenable in the tabular setting. However, even if the norm bound is given in terms of $2$-norm, we can still ensure the PAC guarantee  { (this is because $  \|\OO^{\dagger} \|_1 /  \sqrt{|\Scal|}  \leq \| \OO^{\dagger} \|_2 \leq \| \OO^{\dagger} \|_1  \sqrt{|\Ocal|} $). }

{ Here, since we assume the reward lies in $[0,1]$, value functions on the latent state belong to $ \{\langle \theta, \phi_h(\cdot) \rangle : \|\theta\|_{\infty}\leq H \} $. Here, letting $V^{\pi}_h = \langle \theta^{\pi}_h, \phi_h \rangle$,  value link functions exist by taking $\langle \theta^{\pi}_h, \One(z)\times \OO^{\dagger} \One(o) \rangle$. Hence, we take
\begin{align*}
    \Gcal_h = \left\{(z,o) \mapsto \langle \theta, \One(z) \otimes \OO^{\dagger} \One(o) \rangle; \|\theta\|_{\infty}\leq H  \right \}
\end{align*}
so that the realizability holds. Importantly, we can ensure $\|\Gcal_h\|_{\infty}\leq H/\sigma_1$ since 
\begin{align*}
    |\langle \theta, \One(z)\otimes \OO^{\dagger} \One(o) \rangle | \leq \|\theta\|_{\infty} \|\One(z)\otimes \OO^{\dagger} \One(o)\|_1\leq 
    \|\theta\|_{\infty} \|\OO^{\dagger} \One(o)\|_1
    \leq H/\sigma_1 
\end{align*}
for any $(z,o) \in \Zcal_{h-1}\times \Ocal$. Note $\Gcal_h$ is contained in  \begin{align}\label{eq:contain}
     \braces{\langle \theta, \One(z)\otimes \One(o)\rangle  ; \|\theta\|_2 \leq H |\Ocal|^{M+1} |\Acal|^M  / \sigma_1  } 
\end{align}
This is because each $\langle \theta, \One(z)\otimes \OO^{\dagger} \One(o) \rangle $ is equal to $ \langle \theta', \One(z)\otimes \One(o)\rangle $ for some vector $\theta' \in \mathbb{R}^{|\Zcal_{h-1}| \times |\Ocal|}$. Here, denoting the component of $\theta$ corresponding to $z \in \Zcal_{h-1}$ by $\theta_z \in \mathbb{R}^{|\Ocal|}$, $\theta'$ is a vector stacking $\OO^{\dagger}\theta_z$ for each $z \in \Zcal_{h-1}$. Then, we have
\begin{align*}
    \|\OO^{\dagger}\theta_z\|_2 \leq     \|\OO^{\dagger}\|_2 \|\theta_z\|_2 \leq  \|\OO^{\dagger}\|_1\sqrt{|\Ocal|}H\sqrt{|\Ocal|}\leq  
    H|\Ocal|/\sigma_1. 
\end{align*}
Hence,  $\|\theta'\|_2 \leq  |\Ocal|^{M} |\Acal|^M \times H|\Ocal|/\sigma_1 $. 
}

\paragraph{PO-Bilinear decomposition.}

Next, recall we derive the PO-bilinear decomposition:
\begin{align*}
    & \EE[\theta^{\top}_h\phi_h(\bar z_h) - r_h - \theta^{\top}_{h+1}\phi_{h+1}(\bar z_{h+1}); a_{1:h-1}\sim \pi',a_h\sim \pi ]   \\
    &=  \langle K^{\top}_h\{\theta_h- \theta^{\pi}_h \}- \{T_{\pi:h}\}^{\top} K^{\top}_{h+1}\{ \theta_{h+1} - \theta^{\pi}_{h+1}\} , \EE[\phi_{h}(z_{h-1}, s_{h}); a_{1:h-1}\sim \pi'] \rangle.  
\end{align*}
Then, $B_X = 1$ and $B_W =  4 H |\Ocal|^{M+1} |\Acal|^M/\sigma_1  $. We use $\|K^{\top}_h\|_2 = \|\OO_h\|_2 \leq 1,\|T^{\top}_{\pi:h}\|_2\leq 1 $. This is because
\begin{align*}
    &\|K^{\top}_h\{\theta_h- \theta^{\pi}_h \}- \{T_{\pi:h}\}^{\top} K^{\top}_{h+1}\{ \theta_{h+1} - \theta^{\pi}_{h+1}\}\|_2 \\
    &\leq \|\theta_h\|_2+ \|\theta^{\pi}_h\|_2 + \|\theta_{h+1}\|_2 + \|\theta^{\pi}_{h+1}\|_2 \leq 4H  |\Ocal|^{M+1} |\Acal|^M  / \sigma_1. 
\end{align*}
In the last line, we use \pref{eq:contain}. 

\paragraph{Uniform convergence.}

Then, we can obtain the following uniform convergence lemma. 
\begin{lemma}
Let  $C = H /\sigma_1$ and $d_{\psi} =|\Ocal|^{M+1} |\Acal|^M $. Then, with probability $1-\delta$, 
{ 
\begin{align*}
   & \sup_{\pi \in \Pi, g \in \Gcal} \left| \{\EE_{\Dcal}-\EE\}\bracks{|\Acal| \pi_h(a_h \mid \bar z_h) \braces{ g_h(\bar z_h) - r_h - g_{h+1}(\bar z_{h+1})} } \right|  \\ 
   &\leq  5 |\Acal|\{1+2C\}\sqrt{ \frac{\{ d^2_{\psi}\ln(1+Cm )+\ln( |\Pi_h| /\delta)\}}{m}} 
\end{align*}
}
and 
\begin{align*}
   \sup_{g_1 \in \Gcal_1}|\{\EE_{\Dcal}-\EE\}[g_1(\bar z_1)] | \leq  5 C\sqrt{ \frac{\{ d_{\psi} \ln(1+Cm )+\ln( |\Pi_h| /\delta)\}}{m}}. 
\end{align*}

\end{lemma}
\begin{proof}
Let  $d_{\phi} =|\Scal| |\Ocal|^M |\Acal|^M,d_{\psi} =|\Ocal|^{M+1} |\Acal|^M$.

Define $\Ncal_{\epsilon,h}$ as an $\epsilon$-net for $\{\theta:\|\theta\|_2 \leq C \}$ with respect to $L^{2}$-norm. Define $\Ncal'_{\epsilon,h}$ as an $\epsilon$-net for $\Pi_h:\bar \Zcal_h \to \Delta(\Acal)$ with respect to the following norm:
\begin{align*}
    d(\pi,\pi') =\max_{\bar z_{h-1} \in \bar Z_{h-1}}  \|\pi(\cdot \mid \bar z_{h-1})- \pi'(\cdot \mid \bar z_{h-1})\|_1. 
\end{align*}
Then, $|\Ncal_{\epsilon,h}| \leq (1 + C/\epsilon)^d , |\Ncal'_{\epsilon,h}| \leq (1 + 1/\epsilon)^{ d_{\psi} |\Acal| } $.

Let $ g_h = \langle \theta_h, \psi_h \rangle,  g^{\diamond}_h = \langle \theta^{\diamond}_h, \psi_h \rangle$ where $\psi_h$ is a one-hot encoding vector over $\bar \Zcal_h$. Then, when $\|\theta_h - \theta^{\diamond}_h\|_2\leq \epsilon, \|\theta_{h+1} - \theta^{\diamond}_{h+1}\|_2\leq \epsilon, \|\pi_h -\pi^{\diamond}_{h} \|_1 \leq \epsilon$, we have 
\begin{align*}
    |l_h(\cdot;\pi,g)-l_h(\cdot;\pi^{\diamond},g^{\diamond})|&\leq |\Acal|  \{ \|\pi_h- \pi^{\diamond}_h\|_{\infty}C+ \|g_h-g^{\diamond}_{h}\|_{\infty} + \|g_{h+1}-g^{\diamond}_{h+1}\|_{\infty} \}  \\ 
    &\leq |\Acal| \{ \epsilon C+ \|\theta_h - \theta^{\diamond}_h\|_2 + \|\theta_{h+1} - \theta^{\diamond}_{h+1}\|_2\}  \\
    &\leq 3|\Acal| C\epsilon. 
\end{align*}
Besides, for fixed $\pi \in \Ncal'_{\epsilon,h},\theta_h \in \Ncal_{\epsilon,h}$, $\theta_{h+1} \in \Ncal_{\epsilon,{h+1}}$, we have
\begin{align*}
      \left \lvert \{\EE_{\Dcal}-\EE\}\bracks{|\Acal| \pi_h(a_h \mid \bar z_h) \braces{ g_h(\bar z_h;\theta_h) - r_h - g_{h+1}(\bar z_{h+1};\theta_{h+1}) }}\right \rvert  \leq  |\Acal|\prns{1 + 2C} \sqrt{ \frac{\ln(1/\delta)}{m}}. 
\end{align*}
Then, for $\forall \pi \in \Ncal'_{\epsilon,h}, \forall \theta_h \in \Ncal_{\epsilon,h}$, $\forall \theta_{h+1} \in \Ncal_{\epsilon,{h+1}}$, we have
{ \small
\begin{align*}
    \left \lvert \{\EE_{\Dcal}-\EE\}\bracks{|\Acal| \pi_h(a_h \mid \bar z_h) \braces{ g_h(\bar z_h;\theta_h) - r_h - g_{h+1}(\bar z_{h+1};\theta_{h+1}) }}\right \rvert  \leq  |\Acal|\prns{1 + 2C} \sqrt{ \frac{\ln(|\Ncal'_{\epsilon,h}||\Ncal_{\epsilon,h}||\Ncal_{\epsilon,h+1}| /\delta)}{m}}.
\end{align*}
}
Hence, for any $\pi_h \in \Pi_h, g_h =\langle \theta_h , \psi_h \rangle \in \Gcal_h, g_{h+1} = \langle \theta_{h+1}, \psi_{h+1} \rangle \in \Gcal_{h+1} $, 
{ 
\begin{align*}
    &\left \lvert \{\EE_{\Dcal}-\EE\}\bracks{  |\Acal| \pi_h(a_h \mid \bar z_h) \braces{ g_h(\bar z_h;\theta_h) - r_h - g_{h+1}(\bar z_{h+1};\theta_{h+1}) }}\right \rvert  \\
    &\leq  |\Acal|\prns{1 + 2C} \sqrt{ \frac{\ln(|\Ncal'_{\epsilon,h}||\Ncal_{\epsilon,h}||\Ncal_{\epsilon,h+1}| /\delta)}{m}} + 3|\Acal|C\epsilon. 
\end{align*}
}
By taking $\epsilon = 1/m$, we have $ \forall \pi \in \Pi,  \forall g_h \in \Gcal_h,\forall g_{h+1}\in \Gcal_{h+1};$
\begin{align*}
 & |\{\EE_{\Dcal}-\EE\}[|\Acal| \pi_h(a_h \mid \bar z_h) \{g_h(\bar z_h) - r_h - g_{h+1}(\bar z_{h+1})\}]| \\
      &\leq  |\Acal|\{1+2C\}\sqrt{ \frac{\{ 2d_{\psi}\ln(1+Cm )+d_{\psi}|\Acal| \ln(1+m  )+\ln(1/\delta)\}}{m}} +  \frac{3|\Acal|C}{m} \\ 
      &\leq  10 |\Acal|C \sqrt{ \frac{\{ d_{\psi}|\Acal| \ln(1+ C m ) +\ln(1/\delta)\}}{m}}.  
\end{align*}
\end{proof}

\paragraph{Sample Complexity.}

Finally, we obtain the PAC bound. We need to find $m$ such that
{ 
\begin{align*}
    c |\Acal|C\sqrt{ \frac{\{ d_{\psi}|\Acal| \ln(1+C m ) +\ln(TH/\delta)\}}{m}} \sqrt{d_{\phi}H^2 \ln(Hd_{\phi}B^2_XB^2_W m+1 ) }\leq \epsilon. 
\end{align*}
}
where $c$ is some constant and 
\begin{align*}
    T = c H d_{\phi}\ln(Hd B^2_XB^2_W m +1). 
\end{align*}
By organizing terms, we get 
{ 
\begin{align*}
    \sqrt{ \frac{ |\Acal|^3C^2d_{\phi}d_{\psi}H^2 \ln(H^2d_{\phi}/\delta) \ln(\{C+d_{\psi}+Hd_{\phi}B^2_XB^2_W\} m)^2  }{m } } \leq \epsilon. 
\end{align*}
}
Thus, we need to set 
\begin{align*}
    m = \tilde O\prns{\frac{|\Acal|^3 C^2 d_{\phi} d_{\psi} H^2 \ln(1/\delta)}{\epsilon^2}  }
\end{align*}
Hence, the sample complexity is 
\begin{align*}
      \tilde O\prns{\frac{|\Acal|^3 C^2 d^2_{\phi} d_{\psi} H^4 \ln(1/\delta)}{\epsilon^2}  }. 
\end{align*}
By some algebra, it is 
\begin{align*}
       \tilde O\prns{\frac{|\Acal|^{3M+3} |\Ocal|^{3M+1} |\Scal|^2 H^6 \ln(1/\delta)}{\epsilon^2 \sigma^2_1}  }. 
\end{align*}
Later, we prove we can remove $|\Ocal|^M$ using the more refined analysis in \pref{sec:low_rank_sample}.

\paragraph{Global optimality.}

We use a result in the proof of \citep[Theorem 1.2]{golowich2022planning}. We just set $M= C(1/\sigma_1)^4 \ln(SH/\epsilon)$. Note their assumption 1 is satisfied when $\|\OO^{\dagger}\|_1\leq (1/\sigma_1) $. More specifically,  assumption 1 in \citep{golowich2022planning} requires for any $b$ and $b'$, we have 
\begin{align*}
    \|\OO b-  \OO b'\|_1\geq 1/\sigma_1 \|b-b'\|_1. 
\end{align*}
This is proved as follows. Note for any $e,e'$, 
\begin{align*}
    \| \OO^{\dagger} e- \OO^{\dagger} e'\|_1 \leq  \| \OO^{\dagger} \|_1\|e-e'\|_1. 
\end{align*}
Then, by setting $e=\OO b$ and $e'=\OO b'$, the assumption 1 is ensured. 
Here, we use $\OO^{\dagger}\OO = I$.

\section{Sample Complexity in Observable Overcomplete Tabular POMDPs}

To simplify the presentation, we focus on the case when $\pi^{out}= U(\Acal)$. 

\paragraph{Existence of value link functions.}

In the tabular case, by setting
\begin{align*}
    \psi_h(z,t^K) = \One(z) \otimes \One(t^K), \phi_h(z,s) = \One(z) \otimes \One(s), K_h = \II_{|\Zcal_{h-1}|}\otimes \OO^K. 
\end{align*}
where $\One(z)$ is a one-hot encoding vector over $\Zcal_{h-1}$ and $\One(t^K)$ is a one-hot encoding vector over $\Zcal^{K} = \Ocal^{K}\times \Acal^{K-1}$, we can regard the tabular model as an HSE-POMDP. Here is our assumption. 

\begin{assum}
(a) $0\leq r_h \leq 1$, (b) $ \OO^K$ is full-column rank and $\|\{ \OO^K\}^{\dagger}\|_1 \leq 1/\sigma_1 $. 
\end{assum}

Recall we define $ \OO^K$ in \pref{lem:over_complete_bridge}. 
Since we assume the reward lies in $[0,1]$, value functions on the latent state belong to $ \{\langle \theta, \phi_h(\cdot) \rangle : \|\theta\|_{\infty} \leq H \} $. Here, letting $V^{\pi}_h(\cdot) = \langle \theta^{\pi}_h, \phi_h(\cdot) \rangle$,  value link functions exist by taking $\langle \theta^{\pi}_h, \One(z) \otimes \{\OO^K\}^{\dagger} \One(t^K) \rangle$. Hence, we take
\begin{align*}
    \Gcal_h = \{(z,t^K) \mapsto \langle \theta^{\pi}_h, \One(z) \otimes \{\OO^K\}^{\dagger} \One(t^K) \rangle; \| \theta^{\pi}_h\|_{\infty} \leq H \} 
\end{align*}
so that the realizability holds. Importantly, we can ensure $\|\Gcal_h\|_{\infty}\leq H/\sigma_1$ as in \pref{sec:sample_complexity_tabular}. Then, as in \pref{sec:sample_complexity_tabular}, $\Gcal_h$ is contained in 
\begin{align*}
     \braces{\langle \theta, \One(z)\otimes \One(o)\rangle  ; \|\theta\|_2 \leq H |\Ocal|^{M+1} |\Acal|^M  / \sigma_1  }. 
\end{align*}

\paragraph{PO-bilinear decomposition.}
Next, we derive the PO-bilinear decomposition:
\begin{align*}
    & \EE[\theta^{\top}_h\phi(z_{h-1},t^K_h) - r_h - \theta^{\top}_{h+1}\phi(z_{h-1},t^K_h) ;a_{1:h-1}\sim \pi', a_h \sim \pi, a_{h+1:h+K-1} \sim U(\Acal) ]   \\
    &=  \langle \{K_h\}^{\top} \{\theta_h- \theta^{\pi}_h \}- \{T_{\pi:h}\}^{\top} \{K_{h+1}\}^{\top}\{ \theta_{h+1} - \theta^{\pi}_{h+1}\} , \EE[\phi_{h}(z_{h-1}, s_{h});a_{1:h-1}\sim \pi'] \rangle.  
\end{align*}
Then, $B_X = 1$ and $B_W =  4 H |\Ocal|^{M+1} |\Acal|^M/\sigma_1  $. We use $\|K_h\|_2=\|\OO^K \|_2 \leq 1,\|T^{\top}_{\pi:h}\|_2 \leq 1 $. 

\paragraph{Sample Complexity.}

We can follow the same procedure in the proof of \pref{sec:sample_complexity_tabular}. Let $d_{\phi} = |\Scal| |\Ocal|^M |\Acal|^M , d_{\psi} = |\Ocal|^{M+K} |\Acal|^{M+K-1} $.  Hence, the sample complexity is 
\begin{align*}
      \tilde O\prns{\frac{|\Acal|^3 C^2 d^2_{\phi} d_{\psi} H^4 \ln(1/\delta)}{\epsilon^2 \sigma^2_1}  }. 
\end{align*}
By some algebra, the above is 
\begin{align*}
       \tilde O\prns{\frac{|\Acal|^{3M+K+2} |\Ocal|^{3M+K} |\Scal|^2 H^6 \ln(1/\delta)}{\epsilon^2 \sigma^2_1}  }. 
\end{align*}
Using the more refined analysis later, we show we can remove $|\Ocal|^{3M+K}$ in  \pref{sec:low_rank_sample}.

\input{Proof/LQG}

\input{proof_psr}

%% file: Proof/LQG.tex
\section{Sample Complexity in LQG} \label{sec:sample_complexity_lqg}
In this section, we derive the sample complexity in LQG. We first explain the setting. Then, we prove the existence of link functions. \pref{lem:lqg_value} is proved there. Furtheremore, we show the PO-bilinear rank decomposition in LQG. We prove \pref{lem:lqg_value} there. Next, we show the uniform convergence result in LQG. 
Finally, by invoking \pref{thm:online}, we calculate the sample complexity. 

We study a finite-horizon discrete time LQG governed by the following equation: 
\begin{align*}
    s_1 = \epsilon_1, s_{h+1} = A s_h + B a_h + \epsilon_h,  r_h = s^{\top}_h Q s_h + a^{\top}_h R a_h, o_h = O s_h + \tau_h. 
\end{align*}
where $\epsilon_h$ is a Gaussian noise with mean $0$ and noise $\Sigma_{\epsilon}$ and $\tau_h$ is a Gaussian noise with mean $0$ and $\Sigma_{\tau}$. We use a matrix $O$ instead of $C$ to avoid the notational confusion later. 
With a linear policy $\pi_h(a_h \mid o_h, z_{h-1})=\delta(a_h = \Ub_{1h}o_h+ \Ub_{2h} z_{h-1})$, this induces the following system: 
\begin{align*}
   \begin{bmatrix} 
     z'_{h}  \\ o_h \\ a_h \\ s_{h+1}
   \end{bmatrix}= \Xi_{1h}(\pi)   \begin{bmatrix} 
    z_{h-1} \\   s_h
   \end{bmatrix}  + \Xi_{2h}(\pi),   \Xi_{2h}(\pi) = \begin{bmatrix}  0 \\ \tau \\ \Ub_{1h} \tau\\ B\Ub_{1h}\tau + \epsilon \end{bmatrix}, \, \Xi_{1h}(\pi)= \begin{bmatrix} I' & 0  \\ 0  & O \\ \Ub_{2h} & \Ub_{1h}O  \\  B\Ub_{2h} & A + B\Ub_{1h} O  \end{bmatrix}
\end{align*}
where $z'_h$ is the vector removing $(o_h,a_h)$ from $z_h$ and $I'$ is a matrix mapping $z_h$ to $z'_h$. This is derived by 
\begin{align*}
    s_{h+1} &= A s_h + Ba_h + \epsilon = A s_h + B\{ \Ub_{1h}o_h + \Ub_{2h} z_{h-1} \} + \epsilon  \\
     &= (A+B\Ub_{1h}O)s_h  + B\Ub_{2h}z_{h-1}   + \epsilon+B\Ub_{1h}\tau,  \\ 
    a_h  &= \Ub_{1h}o_h  + \Ub_{2h} z_{h-1}   =  \Ub_{1h}Os_h  + \Ub_{2h} z_{h-1}  + \Ub_{1h}\tau, \\
    o_h  &=  Os_h +\tau. 
\end{align*}
We suppose the system is always stable in the sense that the operator norm of $\Xi_{1h}(\pi)$ is upper-bounded by $1$. Here is the assumption we introduce throughout this section. 
\begin{assum}
Suppose $\max(\|A\|,\|B\|,\|O\|,\|Q\|, \|R\|) \leq \mathbb{C} $. Suppose $\| \Xi_{1h}(\pi)\|\leq 1$ for any $\pi$. $O$ is full-column rank. 
\end{assum}

We present the form of linear mean embedding operators in LQGs. 

\begin{lemma}[Linear mean embedding operator]
Let $z \in \Zcal_{h-1},o \in \Ocal,s \in \Scal$. We have 
{ 
\begin{align*}
\EE_{o \sim \OO(s)} \bracks { \begin{bmatrix}  1 \\  \begin{bmatrix}   z \\ o \end{bmatrix} \otimes \begin{bmatrix}  z \\ o \end{bmatrix}  \end{bmatrix}} = K_h \begin{bmatrix} 1  \\ \begin{bmatrix}  z \\ s \end{bmatrix} \otimes \begin{bmatrix} z \\ s \end{bmatrix} \end{bmatrix}, K_h = \begin{bmatrix} 1  &  \mathbf{0}  \\  \mathrm{vec}\prns{ \begin{bmatrix} 0 & 0\\ 0 & \Sigma_{\tau} \end{bmatrix}} & \begin{bmatrix}   \II & 0 \\  0 & O \end{bmatrix}  \otimes  \begin{bmatrix}   \II & 0 \\  0 & O \end{bmatrix}.  \end{bmatrix}  
\end{align*}
}
\end{lemma}
\begin{proof}

Here, we have 
\begin{align*}
    & \EE_{o \sim \OO(s)} \bracks {   \begin{bmatrix}   z \\ o \end{bmatrix} \otimes \begin{bmatrix}  z \\ o \end{bmatrix} } = \mathrm{Vec}\bracks{  \begin{bmatrix} z z^{\top} & z o^{\top} \\ oz^{\top} & oo^{\top}
       \end{bmatrix} }=\mathrm{Vec}\bracks{  \begin{bmatrix} z z^{\top} & z s^{\top}O^{\top} \\ Osz^{\top} & Oss^{\top}O^{\top} + \Sigma_r
       \end{bmatrix} } \\
    & = \mathrm{Vec}\bracks{  \begin{bmatrix} 0  & 0  \\ 0 & \Sigma_r
       \end{bmatrix} } +  \begin{bmatrix}   \II & 0 \\  0 & O \end{bmatrix}  \otimes  \begin{bmatrix}   \II & O \\  0 & O \end{bmatrix} \times \mathrm{Vec}\bracks{  \begin{bmatrix} z z^{\top} & z s^{\top} \\ sz^{\top} & ss^{\top}
       \end{bmatrix} }. 
\end{align*}
From the second line to the third line, we use formula $\mathrm{vec}[A_1 A_2 A_3] = (A^{\top}_3 \otimes A_1) \mathrm{vec}(A_2)$. 
This immediately concludes the result. 

\end{proof}
Thus, the matrix $K_h$ has the left inverse when $O$ is full-column rank as follows: 
\begin{align*}
    K^{\dagger}_h = \begin{bmatrix} 1  & \mathbf{0}  \\ -  \begin{bmatrix}   \II & 0 \\  0 & O^{\dagger} \end{bmatrix}  \otimes  \begin{bmatrix}   \II & 0 \\  0 & O^{\dagger} \end{bmatrix} \mathrm{vec}\prns{ \begin{bmatrix} 0 & 0\\ 0 & \Sigma_{\tau} \end{bmatrix}}  & \begin{bmatrix}   \II & 0 \\  0 & O^{\dagger} \end{bmatrix}  \otimes  \begin{bmatrix}   \II & 0 \\  0 & O^{\dagger} \end{bmatrix}  \end{bmatrix} . 
\end{align*}
We use a block matrix inversion formula:
\begin{align*}
  \begin{bmatrix}
      A^{-1}_1 & 0\\
     -A^{\dagger}_3A_2A^{-1}_1   & A^{\dagger}_3 
    \end{bmatrix}
    \begin{bmatrix}
      A_1 &  0 \\
      A_2 & A_3 
    \end{bmatrix}= I. 
\end{align*}

\subsection{Existence of Link Functions}

\begin{lemma}[Value functions in LQGs]
Let $\pi_h(a \mid o,z) =\delta(a= \Ub_{1h}o + \Ub_{2h}z )$ for $z \in \Zcal_{h-1}, o \in \Ocal$. Then, a value function has a bilinear form: 
\begin{align*}
 V^{\pi}_h(z, s) = [z^{\top}, s^{\top}]   \Lambda_{h}  [z^{\top}, s^{\top}] ^{\top} +  \Gamma_{h}. 
\end{align*}
For any $h\in [H]$, these parameters $\Lambda_h,\Gamma_h$ are recursively defined inductively by   
{ 
\begin{align*}
    &\Lambda _{ H} = \begin{bmatrix} 
    \Ub^{\top}_{2h} R   \Ub_{2h} & \Ub^{\top}_{2h} R \Ub_{1h}O \\
     \{ \Ub^{\top}_{2h} R \Ub_{1h}O \}^{\top}    &    Q + O^{\top} \Ub^{\top}_{1h} R \Ub_{1h} O 
    \end{bmatrix} ,  O_{H} =  \tr(\Ub^{\top}_{1h} R \Ub_{1h} \Sigma_{\tau}),   \\ 
    & \Lambda_h  = \Xi_{1h}(\pi)  \Lambda_{h+1} \Xi^{\top}_{1h}(\pi) + \Sigma_{\Lambda_h}  , 
     \Sigma_{\Lambda_{h1}}= \begin{bmatrix} 
    \Ub^{\top}_{2h} R   \Ub_{2h} & \Ub^{\top}_{2h} R \Ub_{1h}O \\
     \{ \Ub^{\top}_{2h} R \Ub_{1h}O \}^{\top}    &    Q + O^{\top} \Ub^{\top}_{1h} R \Ub_{1h} O 
    \end{bmatrix},   \\
    &\Gamma_h = \tr\left( \Lambda_{h+1} \Sigma_{\Lambda_{h2}}(\pi) \right)  + \Gamma_{h+1}, \quad  \Sigma_{\Lambda_{h2}}(\pi) =\begin{bmatrix} 0 & 0 &  0 & 0  \\ 0 & \Sigma_{\tau} & \Sigma_{\tau} \Ub^{\top}_{1h}   & \Sigma_{\tau} \Ub^{\top}_{1h} B^{\top}  \\ 
      0 & \Ub_{1h} \Sigma_{\tau}   &  \Ub_{1h} \Sigma_{\tau} \Ub^{\top}_{1h}  &  \Ub_{1h} \Sigma_{\tau} \Ub^{\top}_{1h} B^{\top} 
     \\ 
    0 & B \Ub_{1h} \Sigma_{\tau}   &  B \Ub_{1h} \Sigma_{\tau} \Ub^{\top}_{1h}  & B \Ub_{1h} \Sigma_{\tau} \Ub^{\top}_{1h} B^{\top} + \Sigma_{\epsilon} &  \end{bmatrix}. 
\end{align*}
}

\end{lemma}

\begin{proof}  
The proof is completed by backward induction regarding $h$, starting from level $H$. First, we have
\begin{align*}&
    V^{\pi}_H(z,s) =  s^{\top} Q s +\EE_{o \sim O(s) }[\{\Ub_{1h}o + \Ub_{2h}z\}^{\top}  R \{\Ub_{1h}o + K_{2}z\}] \\
    & = s^{\top} Q s +\EE_{o \sim O(s) }[\{\Ub_{1h}O s + \Ub_{1h}\tau + \Ub_{2h}z\}^{\top}  R \{  \Ub_{1h}O s + \Ub_{1h}\tau + \Ub_{2h}z \}] \\ 
    & =  s^{\top}\{ Q + O^{\top} \Ub^{\top}_{1h} R \Ub_{1h} O \}s +  z \Ub^{\top}_{2h} R   \Ub_{2h}z + 2z^{\top}\Ub^{\top}_{2h} R \Ub_{1h}Os + \tr(\Ub^{\top}_{1h} R \Ub_{1h} \Sigma_{\tau}) \\
    &=   [z^{\top}, s^{\top}]  \begin{bmatrix} 
    \Ub^{\top}_{2h} R   \Ub_{2h} & \Ub^{\top}_{2h} R \Ub_{1h}O \\
     \{ \Ub^{\top}_{2h} R \Ub_{1h}O \}^{\top}    &    Q + O^{\top} \Ub^{\top}_{1h} R \Ub_{1h} O 
    \end{bmatrix}  [z^{\top}, s^{\top}] ^{\top} + \tr(\Ub^{\top}_{1h} R \Ub_{1h} \Sigma_{\tau}). 
\end{align*}
Here, we use induction. Thus, supposing the statement is true at horizon $h+1$, we have 
\begin{align*}
     V^{\pi}_h(z,s)& =\Gamma_{h+1} +  s^{\top} Q s + \EE_{o \sim O(s) }[\{\Ub_{1h}o + \Ub_{2h}z\}^{\top}  R \{\Ub_{1h}o + K_{2}z\}] \\ 
       &+\EE_{o \sim O(s), a \sim \pi(o,z), s' \sim \TT(s,a) }[  [z^{\top}_{-1},o^{\top}, a^{\top},s'^{\top}] \Lambda_{h+1}  [z^{\top}_{-1},o^{\top}, a^{\top},s'^{\top}] ^{\top} ] 
\end{align*}
where $z'$ is a vector that removes the last component $(o,a)$ from $z$ and $s'$ is a state at $h+1$. Here, recall we have 
\begin{align*}
  [(z')^{\top},o^{\top}, a^{\top},s'^{\top}]^{\top}= \Xi_{1h}(\pi)  [z^{\top},s^{\top}]^{\top}  + \Xi_{2h}(\pi). 
\end{align*}
Then, the statement is concluded some algebra.

\end{proof}

\begin{lemma}[Norm constraints on value functions]
We can set  $\| \Lambda_h \|\leq  \CC_{\Lambda,h},\| \Gamma_h \|\leq  \CC_{\Gamma,h} $ such that 
\begin{align*}
    \CC_{\Lambda,h} =\mathrm{poly}(\mathbb{C}, H) ,  \CC_{\Gamma,h}= \mathrm{poly}(d_o,d_s,d_a,\mathbb{C},H).  
\end{align*}
\end{lemma}

\begin{proof}
We have
\begin{align*}
    \| \Lambda_H \| \leq  \mathrm{poly}(\mathbb{C}, H) ,\quad \|\Gamma_H\| \leq \mathrm{poly}(\mathbb{C}, H). 
\end{align*}
Then, 
\begin{align*}
    \|\Lambda_h\| \leq \|\Xi_{1h}(\pi)\|  \|\Lambda_{h+1}\| \|\Xi_{1h}(\pi)\|  +  \mathrm{poly}(\mathbb{C}, H). 
\end{align*}
Since we assume $ \|\Xi_{1h}(\pi)\|\leq 1$, this immediately leads to 
\begin{align*}
     \|\Lambda_h\| \leq \mathrm{poly}(\mathbb{C}, H).
\end{align*}
Besides,
\begin{align*}
    \|\Gamma_h \| \leq  \mathrm{poly}(H,d_o,d_s,d_a,\mathbb{C})\|\Lambda_{h+1}\|    + \|\Gamma_{h+1}\|. 
\end{align*}
Thus, 
\begin{align*}
    \|\Gamma_h\| \leq   \mathrm{poly}(H,d_o,d_s,d_a,\mathbb{C}). 
\end{align*}
\end{proof}

Next, we set the norm on the function class $\Gcal_h$.

\begin{lemma}[Realizability on LQGs]
We set 
\begin{align*}
    &\Gcal_h = \braces{ \bar \Gamma_h + (z^{\top}, o^{\top}) \bar \Lambda_h (z^{\top}, o^{\top})^{\top} \mid \|\bar \Lambda_h\|\leq C_{\bar \Lambda,h}, |\bar \Gamma_h| \leq C_{\bar \Gamma,h},z \in Z_{h-1}, o \in \Ocal },  \\
    &C_{\bar \Lambda,h} =  \mathrm{poly}(H,d_o,d_s,d_a,\mathbb{C},\|O^{\dagger}\|) , C_{\bar \Gamma,h} =\mathrm{poly}(H,d_o,d_s,d_a,\mathbb{C},\|O^{\dagger}\|). 
\end{align*}
A function class $\Gcal_h$ includes at least one value link function for any linear policy $\pi=\delta(a=\Ub_{1h}o + \Ub_{2h}z)$ for $\|\Ub_{1h}\|\leq \mathbb{C},\|\Ub_{2h}\|\leq \mathbb{C}$. 
\end{lemma}

\begin{proof}

Here, we have 
\begin{align*}
    V^{\pi}_h(\cdot)  &=  \Gamma_h  + \tr \braces{ \Lambda_h \begin{bmatrix}   z z^{\top}  & zs^{\top} \\ s z^{\top}& s s^{\top} \end{bmatrix}}\\
    &=  \Gamma_h  + \EE_{o \sim \OO(s)} \bracks{ \tr \braces{ \Lambda_h \begin{bmatrix} z z^{\top}  & zo^{\top}\{O^{\dagger}\}^{\top} \\  O^{\dagger}o z^{\top}& O^{\dagger}\braces{o o^{\top}- \Sigma_{\tau}}\{O^{\dagger}\}^{\top} \end{bmatrix}}} \\ 
    &=   \Gamma_h  -\tr \braces{ \Lambda_h \begin{bmatrix} 0  & 0 \\   0 & O^{\dagger}\Sigma_{\tau} \{O^{\dagger}\}^{\top} \end{bmatrix}}  \\
    &+  \EE_{o \sim \OO(s)} \bracks{ [z^{\top}, o^{\top}]  \begin{bmatrix} I & 0 \\ 0 &  \{O^{\dagger}\} ^{\top} \end{bmatrix}  \Lambda_h \begin{bmatrix} I & 0 \\ 0 &  O^{\dagger} \end{bmatrix}  \begin{bmatrix}
     z \\ o 
    \end{bmatrix}  } . 
\end{align*}
The norm constraint on $\bar \Lambda_h$ is decided by the following calculation: 
\begin{align*}
    \left \|\begin{bmatrix} I & 0 \\ 0 &  \{O^{\dagger}\} ^{\top} \end{bmatrix}  \Lambda_h \begin{bmatrix} I & 0 \\ 0 &  O^{\dagger} \end{bmatrix}  \right \|\leq \|O^{\dagger}\|^2_2 \| \Lambda_h\|=\mathrm{poly}(H,d_o,d_s,d_a,\mathbb{C},\|O^{\dagger}\|). 
\end{align*}
Then, the norm on $\bar \Gamma_h$ is decided by the following calculation: 
\begin{align*}
   \left \lvert   \Gamma_h  -   \tr \braces{ \Lambda_h \begin{bmatrix} 0  & 0 \\   0 & O^{\dagger}\Sigma_{\tau} \{O^{\dagger}\}^{\top} \end{bmatrix}} \right \rvert  &\leq   | \Gamma_h |  +  \left \lvert \tr \braces{ \Lambda_h \begin{bmatrix} 0  & 0 \\   0 & O^{\dagger}\Sigma_{\tau} \{O^{\dagger}\}^{\top} \end{bmatrix}}  \right \rvert \\
   &\leq |\Gamma_h |  + \|\Sigma_h \|_2 \|O^{\dagger}\|^2_2 \mathrm{Tr}(\Sigma_{\tau})\\
   &\leq |\Gamma_h |  + \|\Sigma_h \|_2 \|O^{\dagger}\|^2_2 \mathbb{C} d_o \\
   &= \mathrm{poly}(H,d_o,d_s,d_a,\mathbb{C},\|O^{\dagger}\|). 
\end{align*}
From the first line to the second line, we use \pref{lem:useful}. 

\end{proof}

\subsection{PO-bilinear Rank Decomposition}

\begin{lemma}[Bilinear rank decomposition for LQG]\label{lem:decomposition}
For any $g_{h+1} \in \Gcal_{h+1}, g_h \in \Gcal_h$, we have the following bilinear rank decomposition: 
\begin{align*}
    \EE[ g_{h+1}(z_{h},o_{h+1}) + r_h -g_h(z_{h-1},o_h)  ; a_{1:h-1}\sim \pi',a_h \sim \pi  ] =  \langle X(\pi') , W(\pi) \rangle
\end{align*}
where
\begin{align*}
    &X_h(\pi') = (1, \EE[  [z^{\top}_{h-1},s^{\top}_h ] \otimes  [z^{\top}_{h-1},s^{\top}_h ] ; a_{1:h-1}\sim \pi'] )^{\top},  \\
    & W_h(\pi) =  \begin{bmatrix}
\tr\prns{ \{\bar \Lambda_h - \bar \Lambda^{\star}_h\} \begin{bmatrix}  0 & 0 \\ 0 & \Sigma_{\tau} \end{bmatrix} +  \begin{bmatrix} I & 0 \\ 0 & O^{\top} \end{bmatrix} \{\bar \Lambda_{h+1}-\bar \Lambda^{\star}_{h+1}\} \begin{bmatrix} I & 0 \\ 0 & O \end{bmatrix}\Sigma_{\Lambda_{h2}}(\pi) } 
    \\
    \mathrm{vec}\bracks{  \begin{bmatrix} I & 0 \\ 0 & O^{\top} \end{bmatrix}  \{\bar \Lambda_{h}-\bar \Lambda^{\star}_h \} \begin{bmatrix} I & 0 \\ 0 & O \end{bmatrix}   +   \Xi^{\top}_{1h}(\pi) \begin{bmatrix} I & 0 \\ 0 & O^{\top} \end{bmatrix}   \{\bar \Lambda^{\star}_{h+1} - \bar \Lambda_{h+1}\} \begin{bmatrix} I & 0 \\ 0 & O \end{bmatrix}  \Xi_{1h}(\pi)     }
        \end{bmatrix}.
\end{align*}
Here, $\Xi_{1h}(\pi)$ and $\Sigma_{\Lambda_{h2}}(\pi)$ depend on a policy $\pi$. The following norm constraints hold: 
\begin{align*}
     \|X_h(\pi')\|_2 \leq  \mathrm{poly}(H,d_o,d_s,d_a,\mathbb{C}, \|O^{\dagger}\|), \|W_h(\pi)\|_2 \leq \mathrm{poly}(H,d_o,d_s,d_a,\mathbb{C}, \|O^{\dagger}\|). 
\end{align*}
\end{lemma}

\begin{proof}
We have
\begin{align}
    &\EE[g_h(z_{h-1},o_h) - r_h(z_{h-1},s_h)- g_{h+1}(z_{h},o_{h+1}) ; a_{1:h-1} \sim \pi', a_h \sim \pi] \nonumber \\
    &=- \EE[r_h(z_{h-1},s_h); a_{1:h-1} \sim \pi', a_h \sim \pi ]+   \nonumber \\
    &+ \EE \bracks{ \bar \Gamma_h + (z_{h-1}^{\top}, o^{\top}_h) \bar \Lambda_h (z_{h-1}^{\top}, o^{\top}_{h})^{\top} -\bar \Gamma_{h+1} - (z_{h}^{\top}, o^{\top}_{h+1}) \bar \Lambda_{h+1} (z_{h}^{\top}, o^{\top}_{h+1})^{\top} ; a_{1:h-1} \sim \pi', a_h \sim \pi  }.  \label{eq:LQG_focus}
\end{align}
Since we have 
\begin{align*}
    &\EE[r_h(z_{h-1},s_h); a_{1:h-1} \sim \pi', a_h \sim \pi ]\\
    &=- \EE \bracks{ \bar \Gamma^{\star}_h  + (z_{h-1}^{\top}, o^{\top}_h) \bar \Lambda^{\star}_h (z_{h-1}^{\top}, o^{\top}_{h})^{\top} -\bar \Gamma^{\star}_{h+1} - (z_{h}^{\top}, o^{\top}_{h+1}) \bar \Lambda^{\star}_{h+1} (z_{h}^{\top}, o^{\top}_{h+1})^{\top} ; a_{1:h-1} \sim \pi', a_h \sim \pi  }. 
\end{align*}
we focus on the term \pref{eq:LQG_focus}. 

Hereafter, we suppose the expectation is always taken under $a_{1:h-1}\sim \pi',a\sim \pi$. We also denote $z=z_{h-1},o_h=o,o_{h+1}=o',s_h=s,s_{h+1}=s'$ to simplify the presentation. Using this simplified notation, we get 
\begin{align*}
    & \EE \bracks{  (z ^{\top}, o^{\top}) \bar \Lambda_h (z^{\top}, o^{\top})^{\top}} = \EE\bracks{   (z ^{\top}, (Os+\tau)^{\top}) \bar \Lambda_h (z^{\top}, (Os+\tau)^{\top})^{\top}   } \\
    &=\EE\bracks{ [z^{\top},s^{\top}]  \begin{bmatrix} I & 0 \\ 0 & O^{\top} \end{bmatrix}  \bar \Lambda_h  \begin{bmatrix} I & 0 \\ 0 & O \end{bmatrix} \begin{bmatrix} z \\ s \end{bmatrix}  } + \tr\prns{ \bar \Lambda_h \begin{bmatrix}  0 & 0 \\ 0 & \Sigma_{\tau}
    \end{bmatrix} }. 
\end{align*}
Besides, 
\begin{align*}
    & \EE \bracks{ (z'^{\top}, o'^{\top}) \bar \Lambda_{h+1} (z'^{\top}, o'^{\top})^{\top}   }  \\ 
    & =\EE\bracks{ [z'^{\top},s'^{\top}]  \begin{bmatrix} I & 0 \\ 0 & O^{\top} \end{bmatrix}  \bar \Lambda_{h+1}  \begin{bmatrix} I & 0 \\ 0 & O \end{bmatrix} \begin{bmatrix} z' \\ s' \end{bmatrix}  } + \tr\prns{ \bar \Lambda_h \begin{bmatrix}  0 & 0 \\ 0 & \Sigma_{\tau}
    \end{bmatrix} }\\ 
    &= \EE \bracks{  [z^{\top},s^{\top}]  \Xi^{\top}_{1h}(\pi) \begin{bmatrix} I & 0 \\ 0 & O^{\top} \end{bmatrix}  \bar \Lambda_{h+1} \begin{bmatrix} I & 0 \\ 0 & O \end{bmatrix}  \Xi_{1h}(\pi)  \begin{bmatrix} z \\ s \end{bmatrix}   } +   \tr\prns{ \bar \Lambda_h \begin{bmatrix}  0 & 0 \\ 0 & \Sigma_{\tau}
    \end{bmatrix} } + \\
    &+ \tr \prns{ \begin{bmatrix} I & 0 \\ 0 & O^{\top} \end{bmatrix} \bar \Lambda_{h+1} \begin{bmatrix} I & 0 \\ 0 & O \end{bmatrix}\Sigma_{\Lambda_{h2}}(\pi)}. 
\end{align*}

Then, the bilinear decomposition is clear by using 
\begin{align*}
   A^{\top}_2 A_1 A^{}_2 =\mathrm{tr}(A_1 A^{}_2 A^{\top}_2)=\mathrm{vec}(A^{\top}_1)^{\top}\mathrm{vec}( A^{}_2 A^{\top}_2) = \langle \mathrm{vec}(A^{\top}_1), A_2 \otimes A_2 \rangle. 
\end{align*}
where $A_2$ is any vector and $A_1$ is any matrix. 

First, we calculate the upper bounds of the norms. 
\begin{align*}
   \|X_h(\pi') \|^2_2 &=  1+ \left \| \EE_{(z,s) \sim d^{\pi'}_h } \bracks{ \begin{bmatrix}
     z z^{\top} & zs^{\top} \\
     s z^{\top}  & ss^{\top}
    \end{bmatrix}  } \right\|^2_{F}   \\
    &= 1+ \left \| \EE_{(z,s) \sim d^{\pi'}_{h-1} } \bracks{ 
    \Xi_{1h}(\pi)
     \begin{bmatrix}
     z z^{\top} & zs^{\top} \\
     s z^{\top}  & ss^{\top}
    \end{bmatrix} \Xi^{\top}_{1h}(\pi)
} + \Sigma_{\Lambda_{h2}}(\pi)\right\|^2_{F} \\
    & \leq 1 + \|\Xi_{1h}(\pi)\|^4_2 \left \| \EE_{(z,s) \sim d^{\pi'}_{h-1} } \bracks{ \begin{bmatrix}
     z z^{\top} & zs^{\top} \\
     s z^{\top}  & ss^{\top}
    \end{bmatrix}  } \right\|^2_{F}  + \|\Sigma_{\Lambda_{h2}}(\pi)\|^2_F  \\ 
    &\leq   1 + \|X_{h-1}(\pi') \|^2_2 + \|\Sigma_{\Lambda_{h2}}(\pi)\|^2_F. 
\end{align*}
From the third line to the fourth line, we use $\|\Xi_{1h}(\pi)\|_2\leq 1$. Thus,  $\|X_h(\pi')\|_2 \leq  \mathrm{poly}(H,d_o,d_s,d_a,\|O^{\dagger}\|,\mathbb{C})$. 

Next, we consider $W(\pi)$. By some algebra, we can see 
\begin{align*}
     \|W(\pi)\|_2 &\leq   \mathrm{poly}(\|\bar \Lambda_h\|,\|\bar \Lambda_{h+1}\|,\mathbb{C},d_o,d_s,d_a,\|\Xi_{1h}(\pi)\|) ) \\
     &\leq \mathrm{poly}(H,d_o,d_s,d_a,\|O^{\dagger}\|,\mathbb{C})
\end{align*}
\end{proof}

\begin{lemma}[Variance of marginal distribution] \label{lem:variance}
Recall $d^{\pi}_h(z_{h-1},s_h)$ is a marginal distribution over $\Zcal_{h-1}\times \Scal$ at $h$ when we execute $a_{1:h-1}\sim \pi$. 
The distribution $d^{\pi}_h(z_{h-1},s_h)$ is a Gaussian distribution with mean $0$. The operator norm on the variance of $d^{\pi}_h(z_{h-1},s_h)$ is upper-bounded by  $\mathrm{poly}(H,d_o,d_a,d_s,\mathbb{C}) $. 
\end{lemma} 

\begin{proof}
We first calculate the operator norm of the variance of $d^{\pi}_h(z_{h-1},o_h)$.  
The variance is 
\begin{align*}
    \sum_{i=1}^h 
    \prns{\prod_{t=i+1}^h \Xi_{1t}(\pi) }\Sigma_{\Lambda_{i2} }(\pi)  \prns{\prod_{t=i+1}^h \Xi^{\top}_{1t}(\pi) }. 
\end{align*}
The statement is immediately concluded.
\end{proof}

Let $u_h(\bar z_h,r_h,a_h,o_{h+1};\theta) = \theta^{\top}_h\psi_h(\bar z_h)-r_h-\theta^{\top}_{h+1} \psi_{h+1}(\bar z_{h+1}) $. Recall $\psi_h(\bar z_h) = [1,\bar z^{\top}_h \otimes \bar z^{\top}_h]^{\top}$. We define 
\begin{align*}
 \hat y_h(a^{[i]}) &=  \EE_{\Dcal}\{\alpha_i(\pi(\bar z_h )) \II(\|\bar z_h\| \leq Z_1) \II(\|r_h\|\leq Z_2)  \II(\|o_{h+1}\|\leq Z_3) \II(a_h=a^{[i]})(1+d^{\diamond}) \\
  &u_h(\bar z_h,r_h,a_h,o_{h+1}; \theta)  ; a_{1:h-1}\sim \pi',  a_h \sim U(1+d^{\diamond})  \}  \\
  \hat y_h(a^{[0]}) &=  \EE_{\Dcal}\{\{1- \sum_i \alpha_i(\pi(\bar z_h ))\}  \II(a=0) \II(\|\bar z_h\|\leq Z_1) \II(\|r_h\|\leq Z_2)  \II(\|o_{h+1}\|\leq Z_3)  \\ 
  &(1+d^{\diamond})u_h(\bar z_h,r_h,a_h,o_{h+1}; \theta) ; a_{1:h-1}\sim \pi',  a_h \sim U(1+d^{\diamond})  \} . 
\end{align*}
Then, the final estimator is constructed by 
\begin{align*}
      \hat y_h(a^{[0]})+ \sum_{i=1}^{d^{\diamond}} \hat y_h(a^{[i]}). 
\end{align*}
This is equal to 
\begin{align*}
    \EE_{\Dcal}[l_h(\bar z_h,a_h,r_h,o_{h+1};\theta,\pi) ]
\end{align*}
where
\begin{align*}
    &l_h(\bar z_h,a_h,r_h,o_{h+1};\theta,\pi)=\bracks{ \sum_i \alpha_i(\pi(\bar z_h )) \II(a_h = a^{[i]})  +  \{1- \sum_i \alpha_i(\pi(\bar z_h ))\}  \II(a_h = 0)}  \times  \\
    &\II(\|\bar z_h\|\leq Z_1 ) \II(\|r_h\|\leq Z_2)  \II(\|o_{h+1}\|\leq Z_3) (1+d^{\diamond}) u_h(\bar z_h,r_h,a_h,o_{h+1}; \theta). 
\end{align*}
We set 
\begin{align*}
    Z_i  = \mathrm{poly}(\ln(m), d_s,d_o,d_a, \mathbb{C}, H,\|O^{\dagger}\|). 
\end{align*}
for any $i \in [3]$. 

\subsection{Uniform Convergence}
Recall that 
\begin{align*}
    \Pi = \{\delta(a=\Ub_{1h}z + \Ub_{2h}o) \mid \|\Ub_{1h}\|\leq \CC,\|\Ub_{2h}\|\leq \CC\}. 
\end{align*}
Besides, $\Gcal_h$ is included in 
\begin{align*}
    \{\langle \theta, \psi_h(\cdot) \rangle \mid \|\theta\|\leq \mathrm{poly}(H,d_o,d_s,d_a,\CC,\|O^{\dagger}\|) \}. 
\end{align*}
\begin{lemma}[Concentration of loss functions]
With probability $1-\delta$, 
\begin{align*}
    \sup_{\pi \in \Pi, \theta \in \Theta   }| (\EE_{\Dcal}-\EE)\{ l_h(\bar z_h,a_h,r_h,o_{h+1};\theta,\pi)   \}  |
\end{align*}
is upper-bounded by 
\begin{align*}
     \mathrm{poly}(\ln(m), d_s,d_o,d_a, \mathbb{C}, H,\|O^{\dagger}\|)\times \sqrt{\ln(1/\delta)/m}. 
\end{align*}
\end{lemma}

\begin{proof}
Due to indicator functions, $l_h(\bar z_h,a_h,o_h,o_{h+1};\theta,\pi)$ is bounded for any $\pi,\theta$ by 
\begin{align*}
    \mathrm{poly}(\ln(m), d_s,d_a,d_o,\mathbb{C}, H, \|O^{\dagger}\| ). 
\end{align*}
Thus, for fixed $\pi$ and $\theta$, we can say that with high probability $1-\delta$
\begin{align*}
      \mathrm{poly}(\ln(m),  d_s,d_o,d_a, \mathbb{C}, H, \|O^{\dagger}\|,\ln(1/\delta))\times \sqrt{1/m}. 
\end{align*}
Besides, we can consider a covering number with respect to $l^{\infty}$-norm for the space of $K$ and $\theta$ since  both are bounded. The radius of each space is upper-bounded by 
\begin{align*}
    \mathrm{poly}(\ln(m), d_s,d_o,d_a, \mathbb{C}, H, \|O^{\dagger}\|). 
\end{align*}
Thus, by taking uniform bound and considering the bias term due to the discretization as in the proof of \pref{lem:uni_hse_pomdp}, the statement is concluded. 
\end{proof}

\begin{lemma}[Bias terms 1]\label{lem:bias_term1}
Expectation of $\hat y_h(a^{[i]}) $  and $\hat y_h(a^{[0]}) $ are equal to 
\begin{align*}
    y_h(a^{[i]}) + \mathrm{Error}_1,\quad  y_h(a^{[0]}) + \mathrm{Error}_2. 
\end{align*}
where 
\begin{align*}
   y_h(a^{[i]}) &= \EE\bracks{\alpha_i(\pi(\bar z_h ))\II(\|\bar z_h\|\leq Z_1 )  u_h(\bar z_h,a_h,r_h,o_{h+1};\theta)   ; a_{1:h-1}\sim \pi',  a_h \sim do(a^{[i]}) },\\
 y_h(a^{[0]}) &= \EE\bracks{ \{1- \sum_i \alpha_i(\pi(\bar z_h ))\} \II(\|\bar z_h\|\leq Z_1 )  u_h(\bar z_h,a_h,r_h,o_{h+1};\theta) ; a_{1:h-1}\sim \pi',  a_h \sim do(0) } ,\\
 \mathrm{Error}_1 &= m^{-1}    \mathrm{poly}(\ln(m), d_s,d_o,d_a, \mathbb{C}, H, \|O^{\dagger}\|),\quad \mathrm{Error}_2 = m^{-1}    \mathrm{poly}(\ln(m), d_s,d_o,d_a, \mathbb{C}, H, \|O^{\dagger}\|). 
\end{align*}
\end{lemma}
\begin{proof}
 We want to upper bound the difference of 
 \begin{align*}
      \EE\bracks{\alpha_i(\pi(\bar z_h ))\II(\|\bar z_h\|\leq Z_1 )  u_h(\bar z_h,a_h,r_h,o_{h+1};\theta)   ; a_{1:h-1}\sim \pi',  a_h \sim do(a^{[i]}) }
 \end{align*}
 and 
 \begin{align*}
      \EE\bracks{\alpha_i(\pi(\bar z_h ))\II(\|\bar z_h\|\leq Z_1 )\II(\|r_h\|\leq Z_2)\II(\|o_{h+1}\|\leq Z_3)  u_h(\bar z_h,a_h,r_h,o_{h+1};\theta)   ; a_{1:h-1}\sim \pi',  a_h \sim do(a^{[i]}) }. 
 \end{align*}
By CS inequality, we have 
 {\small 
 \begin{align*}
     & | \EE\bracks{\alpha_i(\pi(\bar z_h ))\II(\|\bar z_h\|\leq Z_1 )\{\II(\|r_h\|\leq Z_2)\II(\|o_{h+1}\|\leq Z_3)-1\}  u_h(\bar z_h,a_h,r_h,o_{h+1};\theta)   ; a_{1:h-1}\sim \pi',  a_h \sim do(a^{[i]}) } |  \\
      &\leq \underbrace{\left \lvert \EE\bracks{\{\II(\|r_h\|\leq Z_2)\II(\|o_{h+1}\|\leq Z_3)-1\}^2   ; a_{1:h-1}\sim \pi',  a_h \sim do(a^{[i]}) } \right \rvert}_{(a)} \\ 
      & \times \underbrace{\left \lvert  \EE\bracks{\alpha^2_i(\pi(\bar z_h ))  u^2_h(\bar z_h,a_h,r_h,o_{h+1})   ; a_{1:h-1}\sim \pi',  a_h \sim do(a^{[i]}) } \right \rvert^{1/2}}_{(b)}. 
 \end{align*}
}
We analyze the term (a) and the term (b). Before starting analysis, note $(\bar z^{\top}_h,a^{\top}_h,r^{\top}_h,o^{\top}_{h+1})$ follows Gaussian distribution with mean $0$ and variance upper-bounded by
\begin{align*}
    \mathrm{poly}(\mathbb{C},d_s,d_o,d_a,H) 
\end{align*}
using \pref{lem:variance}. Besides,  $\alpha^2_i(\pi_h(\bar z_h)) \leq  \mathrm{poly}(d_s,d_o,d_a,H) $ from \pref{lem:norm_alpha}. Note we can use a G-optimal design since we have a norm constraint on $\bar z_1$. 

Regarding the term (a), by setting $Z_2=\mathrm{poly}(\mathbb{C},d_s,d_o,d_a,\ln(m),H,\|O^{\dagger}\|) $ and $Z_3=\mathrm{poly}(\mathbb{C},d_s,d_o,d_a,\ln(m),H ,\|O^{\dagger}\|) $ properly, we can ensure it is upper-bounded by 
\begin{align*}
    \frac{\mathrm{poly}(\mathbb{C},d_s,d_o,d_a,H,\|O^{\dagger}\|,\ln(m))}{m}. 
\end{align*}
Regarding the term (b), noting high order moments of Gaussian distributions can be always upper-bounded, the term (b) is upper-bounded by $ \mathrm{poly}(\mathbb{C},d_s,d_o,d_a,H,\|O^{\dagger}\|,\ln(m))$. This concludes the statement. 
\end{proof}

\begin{lemma}[Bias terms 2]
Recall we define $y_h(a^{[i]})$ and $y_h(a^{[0]})$ in \pref{lem:bias_term1}. Then, we have 
\begin{align*}
    \EE[\II(\|\bar z_h\|\leq Z_1 ) u_h(\bar z_h,a_h,r_h,o_{h+1};\theta) ; a_{1:h-1}\sim \pi',  a_h \sim \pi(\bar z_j)  ] =   y_h(a^{[0]})+ \sum_i y_h(a^{[i]}). 
\end{align*}
Thus, 
\begin{align*}
& \EE[ l_h(\bar z_h,a_h,r_h,o_{h+1};\theta,\pi)  a_{1:h-1}\sim \pi', a_h \sim U(1+d^{\diamond})] \\
    &=\EE[u_h(\bar z_h,a_h,r_h,o_{h+1};\theta) ; a_{1:h-1}\sim \pi',  a_h \sim \pi(\bar z_j)  ]+   \frac{\mathrm{poly}(\mathbb{C},d_s,d_a,d_o,H,\|O^{\dagger}\| ) }{m}. 
\end{align*}
\end{lemma}
\begin{proof}
~
\paragraph{First Statement}
We have 
\begin{align*}
     & \EE[\II(\|\bar z_h\|\leq Z_1 ) u_h(\bar z_h,a_h,r_h,o_{h+1};\theta) ; a_{1:h-1}\sim \pi',  a_h \sim \pi(\bar z_j)  ] \\
     &= \EE[ \II(\|\bar z_h\|\leq Z_1 ) \EE[ u_h(\bar z_h,a_h,r_h,o_{h+1};\theta)\mid  \bar z_h,s_h,a_h ] ; a_{1:h-1}\sim \pi',  a_h \sim \pi(\bar z_j)  ]\\
     &= \EE[ \II(\|\bar z_h\|\leq Z_1 ) \EE[ u_h(\bar z_h,\pi_h(\bar z_h),r_h,o_{h+1};\theta)\mid  \bar z_h,s_h,a_h = \pi_h(\bar z_h) ] ; a_{1:h-1}\sim \pi' ]. 
\end{align*}
Here, by some algebra, there exists a vector $c_2$
\begin{align*}
    \EE[ u_h(\bar z_h,a_h,r_h,o_{h+1};\theta)\mid  \bar z_h,s_h,a_h ] = \langle c_2, [1,[\bar z^{\top}_h,s^{\top}_h,a^{\top}_h]\otimes [\bar z^{\top}_h,s^{\top}_h,a^{\top}_h]  ]^{\top} \rangle . 
\end{align*}
Thus,  there exists $c_0$ and a vector $c_1$ such that 
\begin{align*}
  \EE[ u_h(\bar z_h,a_h,r_h,o_{h+1};\theta)\mid  \bar z_h,s_h,a_h ]=  c_0(\bar z_h,s_h) + c^{\top}_1(\bar z_h,s_h) \kappa(a_h) 
\end{align*}
Recall we can write 
\begin{align*}
 \kappa(\pi_h(\bar z_h)) =  \sum_{i=1}^{d^{\diamond}}  \alpha_i(\pi_h(\bar z_h))\kappa(a^{[i]})
\end{align*}
Using the above, 
{\small 
\begin{align*}
    & \EE[\II(\|\bar z_h\|\leq Z_1)u_h(\bar z_h,a_h,r_h,o_{h+1};\theta) ; a_{1:h-1}\sim \pi',  a_h \sim \pi(\bar z_j)  ] \\
    &=\EE[\II(\|\bar z_h\|\leq Z_1) \{c_0(\bar z_h,s_h) + c^{\top}_1(\bar z_h,s_h)\} \kappa(\pi_h(\bar z_h)) ; a_{1:h-1}\sim \pi' ]\\
    &=\EE\II(\|\bar z_h\|\leq Z_1) \{c_0(\bar z_h,s_h) + \sum_i c^{\top}_1(\bar z_h,s_h) \alpha_i(\pi_h(\bar z_h)) \kappa(a^{[i]})\} ; a_{1:h-1}\sim \pi' ]\\
    &=\EE[\II(\|\bar z_h\|\leq Z_1) [c_0(\bar z_h,s_h) + \\
    &+\sum_i \alpha_i(\pi_h(\bar z_h))\{ \EE[ \kappa(\bar z_h,\pi_h(\bar z_h),r_h,o_{h+1})\mid  \bar z_h,s_h, a_h=a^{[i]}] -c_0(\bar z_h,s_h)\}] ; a_{1:h-1}\sim \pi' ] \\ 
   &= \EE\bracks{\II(\|\bar z_h\|\leq Z_1)\bracks{c_0(\bar z_h,s_h)-\sum_i \alpha_i(\pi_h(\bar z_h)) c_0(\bar z_h,s_h)} } + \sum_i y_h(a^{[i]}).
\end{align*}
}
Besides, 
\begin{align*}
    c_0(\bar z_h,s_h) = \EE[u_h(\bar z_h,a_h,r_h,o_{h+1};\theta) \mid \bar z_h,s_h,a_h=do(0)]. 
\end{align*}
Thus, 
\begin{align*}
   & \EE\bracks{\II(\|\bar z_h\|\leq Z_1) \bracks{c_0(\bar z_h,s_h)-\sum_i \alpha_i(\pi_h(\bar z_h)) c_0(\bar z_h,s_h)} } \\ 
     &= \EE[\II(\|\bar z_h\|\leq Z_1)\{1 - \sum_i \alpha_i(\pi_h(\bar z_h)) \} \EE[u_h(\bar z_h,a_h,r_h,o_{h+1};\theta) \mid \bar z_h,s_h,a_h=0]; a_{1:h-1}\sim \pi' ] \\
     &= \EE[ \II(\|\bar z_h\|\leq Z_1)\{1 - \sum_i \alpha_i(\pi_h(\bar z_h)) \}u_h(\bar z_h,a_h,r_h,o_{h+1};\theta) ;  a_{1:h-1}\sim \pi', a_h = do(0)  ]\\
     &= y_h(a^{[0]}). 
\end{align*}
In conclusion, 
\begin{align*}
    \EE[\II(\|\bar z_h\|\leq Z_1)u_h(\bar z_h,a_h,r_h,o_{h+1};\theta) ; a_{1:h-1}\sim \pi',  a_h \sim \pi(\bar z_j)  ] =  y_h(a^{[0]})+ \sum_i y_h(a^{[i]}). 
\end{align*}

\paragraph{Second Statement} 

As we see in the proof of \pref{lem:bias_term1}, the following term 
\begin{align*}
    \EE[\{\II(\|\bar z_h\|\leq Z_1) - 1  \}u_h(\bar z_h,a_h,r_h,o_{h+1};\theta) ; a_{1:h-1}\sim \pi',  a_h \sim \pi(\bar z_j)  ] 
\end{align*}
is upper-bounded by  $    \mathrm{poly}(\mathbb{C},d_s,d_o,d_a,H,\ln(m) )/m. $
Hence, 
\begin{align*}
    & \EE[l_h(\bar z_h,a_h,r_h,o_{h+1};\theta,\pi) ]\\
    &=\EE[\hat y_h(a^{[0]})] + \sum_{i=1}^{d^{\diamond}}\EE[\hat y_h(a^{[i]})]  \tag{Definition} \\
    &= y_h(a^{[0]}) + \sum_{i=1}^{d^{\diamond}} y_h(a^{[i]})  +  \mathrm{poly}(\mathbb{C},d_s,d_o,d_a,H,\ln(m) )/m \tag{Statement of  \pref{lem:bias_term1}} \\ 
    &= \EE[\II(\|\bar z_h\|\leq Z_1)  u_h(\bar z_h,a_h,r_h,o_{h+1};\theta); a_{1:h-1}\sim \pi', a_h \sim \pi_h(\bar z_h) ]  +  \mathrm{poly}(\mathbb{C},d_s,d_o,d_a,H ,\ln(m))/m \tag{First statement}\\
    &= \EE[u_h(\bar z_h,a_h,r_h,o_{h+1};\theta); a_{1:h-1}\sim \pi', a_h \sim \pi_h(\bar z_h) ]+\mathrm{poly}(\mathbb{C},d_s,d_o,d_a,H,\ln(m) )/m \\
    &=  \mathrm{Br}_h(\pi,\theta;\pi')+\mathrm{poly}(\mathbb{C},d_s,d_o,d_a,H,\ln(m) )/m.
\end{align*}

\end{proof}

\subsection{Sample Complexity}

Summarizing results so far, we have 
\begin{align*}
    &\sup_{\pi \in \Pi, \theta \in \Theta   }| \EE_{\Dcal}[l_h(\bar z_h,a_h,r_h,o_{h+1};\theta,\pi)   \} ] - \mathrm{Br}_h(\pi,\theta;\pi') | \\
    &\leq
     \mathrm{poly}(\ln(m), d_s,d_o,d_a, \mathbb{C}, H,\|O^{\dagger}\|)\times \sqrt{\ln(1/\delta)/m}. 
\end{align*}
This is enough to invoke \pref{thm:online}. Here, recall we have 
\begin{align*}
    \|X_h(\pi)\|\leq \mathrm{poly}(H,d_o,d_a,d_s,\CC,\Theta,\|O^{\dagger}\|), \quad \|W_h(\pi)\|\leq \mathrm{poly}(H,d_o,d_a,d_s,\CC,\Theta,\|O^{\dagger}\|). 
\end{align*}
for any $\pi \in \Pi$ using \pref{lem:decomposition}. In addition, we showed the PO-bilinear rank is 
\begin{align*}
    \mathrm{poly}(H,d_o,d_a,d_s).
\end{align*}
Then, using \pref{thm:online}, the sample complexity is 
\begin{align*}
    \tilde O\prns{ \mathrm{poly}(\ln(m),d_s,d_o,d_a,\CC,\Theta,H,\|O^{\dagger}\|,\ln(1/\delta) ) \times \frac{1}{\epsilon^2} }.  
\end{align*}

%% file: proof_psr.tex
\section{Sample Complexity in PSRs}\label{sec:sample_complexity_psrs}

To focus on the main point, we just use a one-step future. We first show the form of link functions to set a proper class for $\Gcal_h$. Next, we show the PO-bilinear decomposition.  

We assume the following assumptions. 

\begin{assum}
(a) $\Tcal \subset \Ocal$ is a core test and $\Qcal$ is a minimum core rest, (b) $\| \mathrm{vec}(\mathbb{J}^{\pi}_h)\|\leq \Theta$ for any $\pi \in \Pi$ where $\mathbb{J}^{\pi}_h$ is in $\Vcal^\pi_h(\tau_h) = \one(z_{h-1})^\top \mathbb{J}_h^\pi {\bf q}_{\tau_h}$. 
\end{assum}

\subsection{Existence of Link Functions}

Recall $V^\pi_h(\tau_h) = \one(z_{h-1})^\top \mathbb{J}_h^\pi {\bf q}_{\tau_h}$, where we use $\one(z) \in \RR^{| \Ocal|^M | \Acal |^M}$ to denote the one-hot encoding vector over $\Zcal_{h-1}$, and $\mathbb{J}_h^\pi$ is a matrix in $\RR^{| \Ocal|^M | \Acal |^M \times |\Tcal|}$. 

Then, $g^{\pi}_h(z_{h-1}, o):= \one( z_{h-1} )^\top \mathbb{J}^{\pi}_h  [ \one(t = o) ]_{t\in\Tcal} $ is a value link function. This is because
\begin{align*}
   \EE[ g_h(z_{h-1}, o) \mid \tau_h ] &= \EE[ \one( z_{h-1} )^\top \mathbb{J}^{\pi}_h  [ \one(t = o) ]_{t\in\Tcal}  \mid \tau_h ]  \\ 
    & = \one(z_{h-1})^\top \mathbb{J}_h^\pi {\bf q}_{\tau_h}. 
\end{align*}

Hence, we set $\Gcal_h$ to be 
\begin{align*}
     \{ (z_{h-1},o)\mapsto \one( z_{h-1} )^\top \mathbb{J}  [ \one(t = o) ]_{t\in\Tcal}: \|\mathrm{vec}( \mathbb{J})\| \leq \Theta \} 
\end{align*}
so that the realizability holds.

\subsection{PO-bilinear Rank Decomposition}

We show that PSR admits PO-bilinear rank decomposition (\pref{def:general_value}). Here is the Bellman loss: 
\begin{align*}
    \EE[ \{g_{h+1}(z_{h},o_{h+1} )  + r_h \} - g_h(z_{h-1},o_h )  ; a_{1:h-1} \sim \pi', a_h \sim \pi ]. 
\end{align*}
To analyze the above, we decompose the above into three terms: 
{\small 
\begin{align*}
    \underbrace{\EE[g_{h+1}(z_{h},o_{h+1} ) ; a_{1:h-1} \sim \pi', a_h \sim \pi  ]  }_{(a)} + \underbrace{\EE[r_h; a_{1:h-1} \sim \pi', a_h \sim \pi  ]  }_{(b)}  +\underbrace{\EE[- g_h(z_{h-1},o_h ) ; a_{1:h-1} \sim \pi', a_h \sim \pi  ]}_{(c)}. 
\end{align*}
}
Let $\Qcal$ be a minimum core test. Here, for any future $t$, there exists $\tilde m_t$ such that $\PP(t\mid \tau_h) = \langle \tilde m_t, \tilde {\bf q}_{\tau_h} \rangle$ where $[\PP(\cdot \mid \tau_h)]_{|\Qcal|}$ is a $|\Qcal|$-dimensional predictive state $\tilde {\bf q}_{\tau_h}$. This satisfies 
\begin{align}\label{eq:future}
    \PP(o_h\mid \tau_h; a_h) \tilde {\bf q}_{\tau_{h},a_h,o_h} = \tilde M_{o_h,a_h} \tilde {\bf q}_{\tau_{h}}. 
\end{align}
where $\tilde M_{o_h,a_h}$ is a matrix whose $i$-th row is $\tilde m^{\top}_{o_h,a_h}$ as we see in \pref{sec:psrs}. 

\paragraph{Term (c).}
We have 
\begin{align*}
    \EE[ g_h(z_{h-1},o_h )\mid \tau_h  ]& = \one(z_{h-1})^\top \mathbb{J}  \EE[ [ \one(t = o_h) ]_{t\in\Tcal}\mid \tau_h ] \\
    & = \one(z_{h-1})^\top \mathbb{J} \mathbb{J}_1\tilde {\bf q}_{\tau_h}
\end{align*}
where $\mathbb{J}_1 \in \RR^{|\Tcal|\times |\Qcal|}$ is a matrix whose $i$-th row is $\tilde m^{\top}_{t}$. The existence of $\mathbb{J}_1$ is ensured since $\Qcal$ is a core test. 

\paragraph{Term (b).}
We have 
\begin{align*}
    \EE[r_h \mid \tau_h; a_h \sim \pi] &= \sum_{o_h,a_h} \pi(a_h \mid o_h,z_{h-1})r_h(a_h,o_h)\PP(o_h \mid \tau_h; a_h )  \\ 
     &= \sum_{o_h,a_h} \pi(a_h \mid o_h,z_{h-1})r_h(a_h,o_h)\langle \tilde m_{o_h,a_h} ,\tilde {\bf q}_{\tau_h} \rangle  \\ 
    &=  \one(z_{h-1})^\top \mathbb{J}^{\pi}_2 \tilde {\bf q}_{\tau_h} 
\end{align*}
for some matrix $\mathbb{J}^{\pi}_2$. In the first inequality, we use the reward is a function of $o_h,a_h$ conditioning on the whole history. 
From the first line to the second line, we use a property of core tests. 

\paragraph{Term (a).}

We have
\begin{align*}
 &  \EE[g_{h+1}(z_h,o_{h+1}) \mid \tau_h ; a_h \sim \pi] 
= \EE[ \one(z_h)^{\top} \mathbb{J} [\one(t=o_{h+1})]_{t\in \Tcal}\mid \tau_h ; a_h \sim \pi] \\
&=\EE[ \one(z_h)^{\top} \mathbb{J} \mathbb{J}_3 \tilde q_{\tau_{h},a_h,o_h} \mid \tau_h ; a_h \sim \pi] 
\end{align*}
for some matrix $\mathbb{J}_3$. Then, the above is further equal to 
\begin{align*}
   & \sum_{a_h, o_h}\one(z_h)^{\top} \mathbb{J} \mathbb{J}_3  \pi(a_h\mid z_{h-1},o_h) \PP(o_h\mid \tau_h; a_h) \tilde {\bf q}_{\tau_{h},a_h,o_h} \\
     &= \sum_{a_h, o_h}\one(z_h)^{\top} \mathbb{J} \mathbb{J}_3  \pi(a_h\mid z_{h-1},o_h) \tilde M_{o_h,a_h}\tilde {\bf q}_{\tau_{h}} \\
     &= \one(z_{h-1})^\top \mathbb{J}^{\pi}_4 \tilde {\bf q}_{\tau_h} 
\end{align*}
for some matrix $\mathbb{J}^{\pi}_4$. From the first line to the second line, we use $\PP(o_h\mid \tau_h; a_h) \tilde {\bf q}_{\tau_{h},a_h,o_h}=\tilde M_{o_h,a_h}\tilde {\bf q}_{\tau_{h}} $ in \pref{eq:future}. 

\paragraph{Summary.}

Combining all terms, there exists a matrix $\mathbb{J}^{\pi}_5$ such that 
\begin{align*}
   & \EE[ \{g_{h+1}(z_{h},o_{h+1} )  + r_h \} - g_h(z_{h-1},o_h )  ; a_{1:h-1} \sim \pi', a_h \sim \pi ] \\ 
    &= \one(z_{h-1})^\top \mathbb{J}^{\pi}_5 \EE[\tilde {\bf q}_{\tau_h} ; a_{1:h-1}\sim \pi' ] \\ 
    &= \langle \mathrm{Vec}(\mathbb{J}^{\pi}_5),  \one(z_{h-1})\otimes \EE[\tilde {\bf q}_{\tau_h} ; a_{1:h-1}\sim \pi' ]\rangle
\end{align*}
Here, we suppose $\|\mathrm{Vec}(\mathbb{J}^{\pi}_5)\|\leq \Theta_W$ for any $\pi$. Besides, 
\begin{align*}
  & \| \one(z_{h-1})\otimes \EE[\tilde {\bf q}_{\tau_h} ; a_{1:h-1}\sim \pi' ]\|_2 \leq \|\EE[\tilde {\bf q}_{\tau_h} ; a_{1:h-1}\sim \pi' ]\|_2\leq \EE[\|\tilde {\bf q}_{\tau_h}\|_2; a_{1:h-1}\sim \pi' ] \\
  &\leq \EE[\|\tilde {\bf q}_{\tau_h}\|_1; a_{1:h-1}\sim \pi' ]=1. 
\end{align*}
Thus, we can set $B_X=1$. 

\subsection{Sample Complexity}

Suppose $\Pi, \Gcal$ are finite and rewards at $h$ lie in $[0,1]$. Assume the realizability holds. Then, 
\begin{align*}
    \epsilon_{gen} = c \max(\Theta,1) |\Acal| \sqrt{\ln (|\Gcal_{\max} | |\Pi_{\max}| TH/\delta)/m }. 
\end{align*}
Following the calculation in \pref{sec:sample_complexity_finite}, the sample complexity is 
\begin{align*}
    \tilde O \prns{ \frac{|\Ocal|^{2(M-1)} |\Acal|^{2(M-1)}|\Qcal|^2 \max(\Theta,1) H^4 |\Acal|^2 \ln(|\Gcal_{\max} | |\Pi_{\max} | /\delta) \ln(\Theta_W)^2 }{\epsilon^2} }. 
\end{align*}
Here, there is no explicit dependence on $|\Tcal|$. Note the worst-case sample complexity of $\ln |\Gcal_{\max} |$ is  $O( |\Zcal_{h-1}||\Tcal| )$ and the worse-case sample complexity of $\ln |\Pi_{\max}|$ is $O( |\Zcal_{h-1}||\Ocal||\Acal| )$.

\subsection{Most General Case}

We consider the general case in  \pref{sec:psrs}. 
Let $\Gcal_h$ be a function class consisting of $\One(z_{h-1})^{\top}\mathbb{J}_h\One(t) $ where $\mathbb{J}_h$ satisfies $\mathbb{J}_h \in \RR^{\Zcal_{h-1} \times |\Tcal|}$ and  $\|\mathrm{vec}(\mathbb{J}_h)\| \leq \Theta$. When the realizability holds, we would get 
\begin{align*}
    \tilde O \prns{ \frac{|\Ocal|^{2(M-1)} |\Acal|^{2(M-1)} |\Qcal|^2  |\Tcal^{\Acal}|^2\max(\Theta,1) H^4 |\Acal|^2 \ln(|\Gcal_{\max} | |\Pi_{\max} | /\delta) \ln(B_X B_W)^2 }{\epsilon^2} }. 
\end{align*}
Here, there is no explicit sample complexity of $|\Tcal^{\Ocal}|$. Note the worse-case sample complexity of $\ln |\Gcal_{\max}|$ is $O(|\Zcal_{h-1}| |\Tcal| )$ and the worst-case sample complexity of $\ln |\Pi_{\max}|$ is $O(|\Zcal_{h-1}| |\Ocal| |\Acal| )$.

%% file: main_document/proof_more_general.tex
\section{Proof of \pref{thm:online_general}} 

We fix the parameters as in \pref{thm:online_general}. Let 
\begin{align*}
    l_h(\tau_h,a_h,r_h,o_{h+1};f ,\pi, g) = |\Acal| \pi_h(a_h \mid \bar z_h)\{r_h + g_{h+1}(\bar z_{h+1})-g_h(\bar z_h) \}f(\tau_h) - 0.5 f(\tau_h)^2. 
\end{align*}
From the assumption, 
Then,  with probability $1-\delta$, we have 
$\forall t \in [T], \forall h \in [H]$
\begin{align} 
 & \sup_{\pi \in \Pi, g  \in \Gcal, f \in \Fcal}|\EE_{\Dcal^t_h}[l_h(\tau_h,a_h,r_h,o_{h+1};f,\pi,g) ] - \EE[\EE_{\Dcal^t_h}[l_h(\tau_h,a_h,r_h,o_{h+1};f,\pi,g) ] ]|\leq  \epsilon_{gen}, \label{eq:uniform_general} \\
& \sup_{g_1 \in \Gcal_1}|\EE_{\Dcal^0}[g_1(o_1)]-\EE[\EE_{\Dcal^0}[g_1(o_1)] ]  |\leq \epsilon_{ini}.\label{eq:uniform_general2} 
\end{align}
We first show the following lemma. Recall $    \pi^{\star} = \argmax_{\pi \in \Pi}J(\pi). $

\begin{lemma}[Optimism] 
Set $R:=\epsilon_{gen}$. For all $t \in [T]$,  $(\pi^{\star},g^{\pi^{\star}})$ is a feasible solution of the constrained program. Furthermore, we have $ J(\pi^{\star}) \leq \EE[g^t_1(o_1)]+ 2\epsilon_{ini}$ for any $t \in [T]$, where $g^t$ is the value link function selected by the algorithm in iteration $t$.
\end{lemma}
\begin{proof}
For any $\pi$, we have 
\begin{align*}
     \max_{f \in \Fcal_h}|\EE[\EE_{\Dcal^t_h}[l_h(\tau_h,a_h,r_h,o_{h+1};f,\pi,g^{\pi}) ]]|=0 
\end{align*}
since $g^{\pi} $ is a value link function in $\Gcal$ noting the condition (c) in \pref{def:bilinear_minimax}. Thus, 
\begin{align*}
 \max_{f \in \Fcal_h} |\EE_{\Dcal^t_h}[l_h(\tau_h,a_h,r_h,o_{h+1};f,\pi^{\star},g^{\pi^{\star}}) ]|\leq \epsilon_{gen} 
\end{align*}
using \eqref{eq:uniform_general} noting $\pi^{\star} \in \Pi, g^{\pi^{\star}}  \in \Gcal$. 
Hence, $(\pi^{\star},g^{\pi^{\star}})$ 
is a feasible set for any $t \in [T]$ and any $h \in [H]$ .

Then, we have 
\begin{align*}
    J(\pi^{\star}) &= \EE[g^{\pi^{\star}}_1(o_1) ]    \leq \EE_{\Dcal^0}[ g^{\pi^{\star}}_1(o_1) ]+ \epsilon_{ini} \tag{Uniform convergence result} \\ 
      &\leq \EE_{\Dcal^0}[g^t_1(o_1)]+ \epsilon_{ini} \tag{Using the construction of algorithm}  \\
      &\leq \EE[g^t_1(o_1)]+2 \epsilon_{ini}. \tag{Uniform convergence} 
\end{align*}
\end{proof}

Next, we prove the following lemma to upper bound the per step regret.  
\begin{lemma} \label{lem:performance_diff_general}
For any $t \in [T]$, we have 
\begin{align*}
       J(\pi^{\star}) - J(\hat \pi )  \leq \sum_{h=1}^H  \prns{ |\langle  W_h(\pi^t ,g^t ),X_h(\pi^t )\rangle  |}+  2\epsilon_{ini}. 
\end{align*}

\end{lemma}
\begin{proof}
\begin{align*}
    &J(\pi^{\star}) - J(\hat \pi )  & \\
    &\leq   2\epsilon_{ini} + \EE[g^{t}_1(o_1)]- J(\pi^t ) \tag{From optimism} \\ 
     &=  2\epsilon_{ini} + \sum_{h=1}^H \EE[ g^{t}_h(\bar z_h)-\{ r_h  +g^{t}_{h+1}(\bar z_{h+1}) \} ; a_{1:h} \sim \pi^t] \tag{Performance difference lemma} \\ 
     &\leq  2\epsilon_{ini} + \sum_{h=1}^H 
     |\EE[ g^{t}_h(\bar z_h)-\{ r_h  +g^{t}_{h+1}(\bar z_{h+1}) \} ; a_{1:h} \sim \pi^t ]| \\ 
     &\leq 2\epsilon_{ini} + \sum_{h=1}^H  |\langle  W_h(\pi^t,g^t),X_h(\pi^t)\rangle  |.  \tag{From (a) in \pref{def:simple_bilinear}}
\end{align*}

\end{proof}

From \pref{lem:squared_potential}, we have
\begin{align*}
     \frac{1}{T} \sum_{t=0}^{T-1}\sum_{h=1}^H  \|X_h(\pi^t)\|_{\Sigma^{-1}_{t,h}} \leq  H\sqrt{\frac{d}{T}\ln \left(1 + \frac{TB^2_{X}}{d\lambda} \right)}. 
\end{align*}

\begin{lemma}
\begin{align*}
    \|W_h(\pi^t,g^t) \|^2_{\Sigma_{t,h}}\leq 2 \lambda B^2_W + T\zeta(2\epsilon_{gen}). 
\end{align*}
\end{lemma}
\begin{proof}
We have 
\begin{align*}
     \|W_h(\pi^t,g^t) \|^2_{\Sigma_{t,h}}  = \lambda \|W_h(\pi^t,g^t) \|^2_2 + \sum_{\tau =0}^{t-1} \langle W_h(\pi^{t},g^{t}), X_h(\pi^{\tau}) \rangle^2. 
 \end{align*}  
 The first term is upper-bounded by $\lambda B^2_W$. The second term is upper-bounded by 
\begin{align*}
&  \sum_{\tau =0}^{t-1} \langle W_h(\pi^{t},g^{t}), X_h(\pi^{\tau}) \rangle^2 \\ 
& \leq  \sum_{k =0}^{t-1} \zeta\prns{ \max_{f \in \Fcal_h} \left|  \EE[l_h( \tau_h,a_h,r_h,o_{h+1} ;f,\pi^t ,g^t) ; a_{1:M(h)-1} \sim \pi^{k}, a_{M(h):h} \sim \pi^e(\pi) ] \right| }^2 \\
&\leq   \sum_{k =0}^{t-1}  \zeta \prns{\max_{f \in \Fcal_h}\left | \EE_{\Dcal^{k}_h }[l_h( \bar z_h,a_h,r_h,o_{h+1} ;f, \pi^t ,g^t) ]\right| + \epsilon_{gen} }^2 \\
&\leq  t\zeta(2\epsilon_{gen})^2.  
\end{align*}
From the first line to the second line, we use (b) in \pref{def:bilinear_minimax}. From the second line to the third line, we use $\xi$ is a non-decreasing function. In the last line, we use the constraint on $(\pi^t,g^t)$. 

\end{proof}

Combining lemmas so far, we have 
\begin{align*}
       J(\pi^{\star}) - J(\hat \pi )  & \leq \frac{1}{T}\sum_{t=0}^{T-1}\sum_{h=1}^H  |\langle  W_h(\pi^t ,g^t ),X_h(\pi^t )\rangle  |+  2\epsilon_{ini} \\ 
       & \leq \frac{1}{T}\sum_{t=0}^{T-1}\sum_{h=1}^H  \|  W_h(\pi^t ,g^t )\|_{\Sigma_{t,h}} \|X_h(\pi^t )\|_{\Sigma^{-1}_{t,h}}+  2\epsilon_{ini} \tag{CS inequality} \\
       &\leq  H^{1/2} \bracks{2 \lambda B^2_W +  T\zeta^2(2\epsilon_{gen} ) }^{1/2} \prns{ \frac{dH}{T} \ln\left(1 + \frac{T B^2_X}{d\lambda} \right)}^{1/2}        +2\epsilon_{ini}.
\end{align*} 
We set $\lambda$ such that  $B^2_X/\lambda = B^2_W B^2_X / \zeta^2(\epsilon_{gen}) +1 $ and $T= \ceil*{ 2Hd \ln(4Hd (B^2_XB^2_W/\zeta^2(\tilde \epsilon_{gen}) +1))} $. Then, 
\begin{align*}
    \frac{Hd}{T}\ln \left(1 + \frac{TB^2_{X}}{d\lambda} \right) &\leq  \frac{Hd}{T}\ln \left(1 + \frac{T}{d}\prns{ \frac{B^2_W B^2_X }{ \zeta^2(\epsilon_{gen})} +1 }\right)
    \\   
    &\leq \frac{Hd}{T}\ln \left(1 + \frac{T}{d}\prns{ \frac{B^2_W B^2_X }{\zeta^2(\tilde \epsilon_{gen})} +1 }\right)\\
    &\leq \frac{Hd}{T}\ln \left(\frac{2T}{d}\prns{\frac{B^2_W B^2_X}{\zeta^2(\tilde \epsilon_{gen}) } +1 }\right)\leq 1
\end{align*}
since $a \ln (bT)/T \leq 1$ when $T = 2a \ln(2ab)$. 

Finally, the following holds
\begin{align*}
       &J(\pi^{\star})  - J(\pi^T ) \\
       &\leq  H^{1/2} \bracks{2 \lambda B^2_W + T \zeta^2(2\epsilon_{gen}) }^{1/2} +2\epsilon_{ini} \\ 
       & \leq H^{1/2} \bracks{2 \lambda B^2_W + 2 \zeta^2(2\epsilon_{gen})Hd \ln(4Hd (B^2_XB^2_W/\zeta^2(\tilde \epsilon_{gen}) +1))  }^{1/2} +2\epsilon_{ini}  \tag{Plug in $T$ }\\
        & \leq H^{1/2} \bracks{4 \zeta^2(\epsilon_{gen}) +  2 \zeta^2(2\epsilon_{gen})Hd \ln(4Hd (B^2_XB^2_W/\zeta^2(\tilde \epsilon_{gen}) +1))  }^{1/2} +2\epsilon_{ini} \tag{Plug in $\lambda$ }.
\end{align*}

%% file: proof2.tex
\section{Sample Complexity in $M$-step Decodable POMDPs} \label{sec:m_step_decodable_sample}

We first give a summary of our results. Then, we show that an $M$-step decodable POMDP is a PO-bilinear rank model. After showing the uniform convergence of the loss function with fast rates, we calculate the sample complexity. Since we use squared loss functions, we need to modify the proof of \pref{thm:online}.

\subsection{PO-bilinear Rank Decomposition (Proof of  \pref{lem:decodable}) }\label{subsec:proof_po_bilinear}

In this section, we derive the PO-bilinear decomposition of $M$-step decodable POMDPs (\pref{lem:decodable} ). 

First, we define moment matching policies following \citep{efroni2022provable}. We denote $M(h)=h-M$. 

\begin{definition}[Moment Matching Policies]
For $h' \in [M(h),h]$, we define
\begin{align*}
    x_{h'} = (s_{M(h):{h'}}, o_{M(h):{h'}},a_{M(h):h'-1}) \in \Xcal_l
\end{align*}
where $\Xcal_l = S^{l} \times \Ocal^{l} \times \Acal^{l-1}$ and $l = h'-M(h)+1$. For an $M$-step policy $\pi$ and $h \in [H]$, we define the moment matching policy $\mu^{\pi,h}=\{\mu^{\pi,h}_{h'}:\Xcal_{h'-M(h)+1} \to \Delta(\Acal) \}^{h-1}_{h'=M(h)}$:
\begin{align*}
    \mu^{\pi,h}_{h'}(a_{h'} \mid x_{h'}) :=\EE[\pi_{h'}(a_{h'} \mid \bar z_{h'})  \mid x_{h'} ; \pi ]. 
\end{align*}
Note the expectation in the right hand side is taken under a policy $\pi$. 
\end{definition}

Using \citep[Lemma B.2]{efroni2022provable}, we have
\begin{align*}
 \mathrm{Br}(\pi,g; \pi') & = \EE[\{g_h(\bar z_h) - r_h-g_{h+1}(\bar z_{h+1})\} ;   a_{1:M(h)-1} \sim  \pi',a_{M(h):h}\sim \pi  ]  \\
&= \EE[\{g_h(\bar z_h) - r_h-g_{h+1}(\bar z_{h+1})\} ; a_{1:M(h)-1} \sim  \pi',a_{M(h):h-1}\sim \mu^{\pi,h},a_{h} \sim \pi  ]  \\
&= \EE[ \EE[\{g_h(\bar z_h) - r_h-g_{h+1}(\bar z_{h+1})\} \mid s_{M(h)}  ; a_{M(h):h-1}\sim \mu^{\pi,h},a_{h} \sim \pi  ] ; a_{1:M(h)-1} \sim  \pi' ]  \\
&= \left \langle  X_h(\pi'), W_h(\pi,g)  \right\rangle . 
\end{align*}
where
\begin{align*}
   W_h(\pi,g) &=  \int \EE[\{g_h(\bar z_h) - r_h-g_{h+1}(\bar z_{h+1})\} \mid s_{M(h)}  ; a_{M(h):h-1}\sim \mu^{\pi,h},a_{h} \sim \pi  ]\mu(s_{M(h)})\mathrm{d}(s_{M(h)}), \\
   X_h(\pi') &= \EE[\phi(s_{M(h)-1},a_{M(h)-1} )  ; a_{1:M(h)-1} \sim  \pi' ]. 
\end{align*}

Thus, the first condition in \pref{def:bilinear_minimax} (\pref{eq:first_condition}) is satisfied 

Next, we show the second condition in \pref{def:bilinear_minimax} (\pref{eq:second_condition}). This is proved as follows

{\small 
\begin{align}
     & \frac{0.5}{|\Acal|^{M}}  \left \langle  X_h(\pi'), W_h(\pi,g)  \right\rangle^2 \label{eq:decodable_bilinear_transform}\\
    & = \frac{0.5}{|\Acal|^{M}}\left( \EE\left[\left( g_h(\bar z_h) -   (\Bcal^{\pi}_h g_{h+1})( \bar z_h )  \right) ; a_{1:M(h)-1} \sim  \pi',a_{M(h):h-1}\sim \mu^{\pi,h} \right]\right)^2\nonumber \\ 
    & \leq \frac{0.5}{|\Acal|^{M}} \EE\left[\left( g_h(\bar z_h) -   (\Bcal^{\pi}_h g_{h+1})( \bar z_h )  \right)^2 ; a_{1:M(h)-1} \sim  \pi',a_{M(h):h-1}\sim \mu^{\pi,h} \right] \nonumber \tag{Jensen's inequality} \\ 
    & \leq \frac{1}{|\Acal|^{M}} \max_{f \in \Fcal_h} \EE\left[\left( g_h(\bar z_h) -   (\Bcal^{\pi}_h g_{h+1})( \bar z_h )  \right) f(\bar z_h) - 0.5 f(\bar z_h)^2 ; a_{1:M(h)-1} \sim  \pi',a_{M(h):h-1}\sim \mu^{\pi,h} \right] \nonumber\\
    & = \frac{1}{|\Acal|^{M}} \max_{f \in \Fcal_h} \EE\left[ |\Acal| \pi_h (a_h | \bar z_h) \left( g_h(\bar z_h) -   r_h - g_{h+1}(\bar z_{h+1}) \right) f(\bar z_h) - 0.5 f(\bar z_h)^2 ; a_{1:M(h)-1} \sim  \pi',a_{M(h):h-1}\sim \mu^{\pi,h}, a_h\sim \Ucal(\Acal) \right] \nonumber\\
    &\leq \max_{f \in \Fcal_h} \EE\left[ |\Acal| \pi_h (a_h | \bar z_h) \left( g_h(\bar z_h) -   r_h - g_{h+1}(\bar z_{h+1}) \right) f(\bar z_h) - 0.5 f(\bar z_h)^2 ; a_{1:M(h)-1} \sim  \pi',a_{M(h):h}\sim  \Ucal(\Acal) \right] \nonumber \\
    & = \max_{f \in \Fcal_h}  \EE\left[ l_h( \bar z_h, a_h, r_h o_{h+1}; f, \pi, g ) ; a_{1:M(h)-1} \sim  \pi',a_{M(h):h}\sim  \Ucal(\Acal) \right]. 
\end{align}
}
From the first line to the second line, we use \citep[Lemma B.2]{efroni2022provable}. From the third to the fourth line, we use the Bellman completeness assumption: $-(\Bcal^{\pi}_h\Gcal) + \Gcal_h\subset \Fcal_h$. From the fourth line to the fifth line, we use importance sampling.

Finally, we show the third condition in \pref{def:bilinear_minimax} \pref{eq:third_condition}:
\begin{align}
    \left\lvert  \max_{f \in \Fcal_h}\EE[l_h( \tau_h, a_h, r_h, o_{h+1}; f, \pi, g^{\pi});a_{1:M(h)-1}\sim \pi', a_{M(h):h} \sim \Ucal(\Acal) ] \right\rvert = 0.
\end{align}
This follows since 
\begin{align*}
    & \EE[l_h( \tau_h, a_h, r_h, o_{h+1}; f, \pi, g^{\pi});a_{1:M(h)-1}\sim \pi', a_{M(h):h} \sim \Ucal(\Acal) ] \\
    &= \EE\left[ |\Acal| \pi_h (a_h | \bar z_h) \left( g^{\pi}_h(\bar z_h) -   r_h - g^{\pi}_{h+1}(\bar z_{h+1}) \right) f(\bar z_h) - 0.5 f(\bar z_h)^2 ; a_{1:M(h)-1} \sim  \pi',a_{M(h):h}\sim  \Ucal(\Acal) \right]\\
&= \EE\left[ \EE[|\Acal| \pi_h (a_h | \bar z_h) \left( g^{\pi}_h(\bar z_h) -   r_h - g^{\pi}_{h+1}(\bar z_{h+1}) \right)\mid \bar z_h ]f(\bar z_h) - 0.5 f(\bar z_h)^2 ; a_{1:M(h)-1} \sim  \pi',a_{M(h):h}\sim  \Ucal(\Acal) \right]\\
&= \EE \left[  - 0.5 f(\bar z_h)^2 ; a_{1:M(h)-1} \sim  \pi',a_{M(h):h}\sim  \Ucal(\Acal) \right].
\end{align*}

\subsection{Uniform Convergence}

We define the operator
\begin{align*}
      (\Bcal^{\pi}_h g)(\bar z_h):=\EE[ r_h + g_{h+1}(\bar z_{h+1}) \mid \bar z_h; a_h \sim \pi  ].
\end{align*}
and 
\begin{align*}
      (\bar \Bcal^{\pi}_h g)(\bar z_h):= - (\Bcal^{\pi}_{h} g)(\bar z_h) + g_h. 
\end{align*}

\begin{lemma}[Uniform Convergence] \label{lem:decodable_key_lemma}
Let $|\Dcal| = m$. Suppose $\|\Fcal_h \|_{\infty} \leq 3H$ for $h\in [H]$. Fix $\pi' \in \Pi$. 
\begin{enumerate}
    \item Take a true link function $g^{\pi}\in \Gcal$. Then, it satisfies 
    { 
\begin{align*}
   & \max_{f_h \in \Fcal_h}|\EE_{\Dcal}[|\Acal|\pi_h(a_h \mid \bar z_h)\{g^{\pi}_h(\bar z_h) - r_h-g^{\pi}_{h+1}(\bar z_{h+1})\}f_h(\bar z_h) - 0.5 f_h(\bar z_h)^2 ;   a_{1:M(h)-1} \sim  \pi',a_{M(h):h}\sim U(\Acal) ] \\
    &\leq   c_1\frac{(H |\Acal|)^2\ln(|\Pi_{\max}|  |\Fcal_{\max}| |\Gcal_{\max}|  /\delta)}{m} . 
\end{align*}
}
    \item Suppose $g(\pi)$ satisfies 
{\small
\begin{align*}
    &\max_{f_h \in \Fcal_h}|\EE_{\Dcal}[|\Acal|\pi_h(a_h \mid \bar z_h)\{g_{h}(\pi)(\bar z_h) - r_h-g_{h+1}(\pi)(\bar z_{h+1})\}f_h(\bar z_h) - 0.5 f_h(\bar z_h)^2 ;   a_{1:M(h)-1} \sim  \pi',a_{M(h):h}\sim U(\Acal) ]| \\
    &\leq \Lambda, 
\end{align*}
}
and the Bellman completeness $\bar \Bcal^{\pi}_h  \Gcal \subset \Fcal_h (\forall \pi \in \Pi)$ holds. Then, with probability $1-\delta$, we have
\begin{align*}
  & \EE[   (\bar \Bcal^{\pi}_hg(\pi))^2(\bar z_h);   a_{1:M(h)-1} \sim  \pi',a_{M(h):h-1}\sim U(\Acal) ] \\
  &\leq \Lambda +  c_2\frac{(H |\Acal|)^2\ln(|\Pi_{\max}|  |\Fcal_{\max}| |\Gcal_{\max}|  /\delta)}{m} . 
\end{align*}
\end{enumerate}
\end{lemma}
\begin{proof}
To simplify the notation, we define
$$\alpha_h(\bar z_h,a_h,r_h,o_{h+1};g) =\pi_h(a_h \mid \bar z_h)|\Acal| \{g_h(\bar z_h) - r_h-g_{h+1}(\bar z_{h+1})\}.$$ 
Given $g \in \Gcal$, we define $\hat f_h(\cdot ; g)$ as the maximizer: 
\begin{align*}
    \argmax_{f_h \in \Fcal_h}|\EE_{\Dcal}[|\Acal|\pi_h(a_h \mid \bar z_h)\{g_h(\bar z_h) - r_h-g_{h+1}(\bar z_{h+1})\}f_h(\bar z_h) - 0.5 f_h(\bar z_h)^2 ;   a_{1:M(h)-1} \sim  \pi',a_{M(h):h}\sim U(\Acal) ]|. 
\end{align*}
In this proof, the expectation is always taken for the data generating process $\Dcal$. 
We first observe
\begin{align*}
    &\EE_{\Dcal}[\alpha_h(\bar z_h,a_h,r_h,o_{h+1};g)  f_h(\bar z_h) - 0.5 f_h(\bar z_h)^2  ]\\
    &=0.5\EE_{\Dcal}[ \alpha_h(\bar z_h,a_h,r_h,o_{h+1};g)^2- \{\alpha_h(\bar z_h,a_h,r_h,o_{h+1};g) -  f_h(\bar z_h)\} ^2 ]. 
\end{align*}
Then, we define 
\begin{align*}
    \mathrm{Er}_h(f,g):= 0.5\{ \alpha_h(\bar z_h,a_h,r_h,o_{h+1};g) -  f_h(\bar z_h)\}^2-0.5 \{ \alpha_h(\bar z_h,a_h,r_h,o_{h+1};g) -   (\bar \Bcal^{\pi}_h g)(\bar z_h)  \}^2 ]. 
\end{align*}

As the first step, we prove with probability $1-\delta$
\begin{align}\label{eq:first_step}
  \forall g;   |\EE_{\Dcal}[ \mathrm{Er}_h( \hat f_h(\cdot;g),g )]| \leq \frac{12H|\Acal| \ln(2|\Fcal_h|||\Gcal_h||\Gcal_{h+1}| /\delta)}{m}.
\end{align}

We first fix $g$. Then, from the definition of $\hat f_h(\cdot ; g)$ and the Bellman completeness $\bar \Bcal^{\pi}_h\Gcal \subset \Fcal_h$, we have 
\begin{align}\label{eq:help1}
    \EE_{\Dcal}[ \mathrm{Er}_h(\hat f_h(\cdot; g) , g)] \leq 0. 
\end{align}
Here, we invoke Bernstein's inequality: 
\begin{align}\label{eq:help2}
    \forall f\in \Fcal_h; |(\EE-\EE_{\Dcal})\mathrm{Er}_h(f , g))| \leq \sqrt{ \EE[ \mathrm{Er}_h(f , g)] \frac{\ln(2|\Fcal_h|/\delta)}{m} }+\frac{ (6H|\Acal|)^2 \ln(2|\Fcal_h| /\delta)}{m}.
\end{align}
Hereafter, we condition on the above event. Then, combining \eqref{eq:help1} and  \eqref{eq:help2}, we have  
\begin{align*}
   \EE[ \mathrm{Er}_h(\hat f_h(\cdot;g) , g )] &\leq    \EE_{\Dcal}[ \mathrm{Er}_h(\hat f_h(\cdot; g), g )]  + |(\EE-\EE_{\Dcal})\mathrm{Er}_h(\hat f_h(\cdot;g),g)|   \\
   &\leq \sqrt{ \frac{\EE[ \mathrm{Er}^2_h(\hat f_h(\cdot ; g), g)] \ln(2|\Fcal_h|/\delta) (6H |\Acal|)^2}{m}}+\frac{(6H|\Acal|)^2 \ln(2|\Fcal_h| /\delta)}{m}.
\end{align*}
Here, we use 
\begin{align*}
      \EE[ \mathrm{Er}_h(\hat f_h(\cdot;g) , g )] &= 0.5 \EE[ \{f_h(\bar z_h) - (\bar \Bcal^{\pi}_h g)(\bar z_h)\}^2 ],  \\ 
      \EE[ \mathrm{Er}_h(\hat f_h(\cdot ; g), g )^2] &\leq  \EE[ \{f_h(\bar z_h) - (\bar \Bcal^{\pi}_h g)(\bar z_h)\}^2 ] (6H |\Acal|)^2= \EE[ \mathrm{Er}_h(\hat f_h(\cdot;g))] (6H |\Acal|)^2. 
\end{align*}
Thus, by some algebra, 
\begin{align*}
    \EE[ \mathrm{Er}_h(\hat f_h(\cdot;g), g)]\leq \frac{(12H|\Acal|)^2 \ln(2|\Fcal_h| /\delta)}{m}.
\end{align*}
Besides, 
\begin{align*}
    &|\EE_{\Dcal}[ \mathrm{Er}_h( \hat f_h(\cdot;g) , g )]| \\
    &\leq  \EE[ \mathrm{Er}_h( \hat f_h(\cdot;g))]+ |(\EE-\EE_{\Dcal})[\mathrm{Er}_h ( \hat f_h(\cdot;g), g )]|  \\ 
        &\leq  \EE[ \mathrm{Er}_h( \hat f_h(\cdot;g), g )]+ \sqrt{ \frac{\EE[ \mathrm{Er}_h(\hat f_h(\cdot ; g), g )] (6H|\Acal|)^2\ln(2|\Fcal_h|/\delta)}{m}}+\frac{(6H|\Acal|)^2  \ln(2|\Fcal_h| /\delta)}{m}  \\ 
    &\leq \frac{3(12H|\Acal|)^2 \ln(2|\Fcal_h| /\delta)}{m}+ \frac{27H|\Acal| \ln(2|\Fcal_h| /\delta)}{m}.
\end{align*}
Lastly, by union bounds over $\Gcal_h,\Gcal_{h+1}$, the statement \pref{eq:first_step} is proved. Note $\bar \Bcal^{\pi}_h g^{\pi}=0$. 

\paragraph{First Statement.}

\begin{align*}
    &|\EE_{\Dcal}[\alpha_h(\bar z_h,a_h,r_h,o_{h+1};g^{\pi} )\hat f_h(\bar z_h;g^{\pi})-0.5 \hat f_h(\bar z_h;g^{\pi})^2 ]| \\
    &=|0.5\EE_{\Dcal}[\alpha_h(\bar z_h,a_h,r_h,o_{h+1};g^{\pi} )]-0.5\EE_{\Dcal}[ \{\alpha_h(\bar z_h,a_h,r_h,o_{h+1};g^{\pi} )-f_h(\bar z_h) \}^2] |\\
    &\leq c \frac{H |\Acal| \ln(|\Fcal_h | \Gcal_h | |\Gcal_{h+1}/\delta)}{m}. 
\end{align*}
From the second line to the third line, we use \pref{eq:first_step}.

\paragraph{Second Statement.}

Now, we use the assumption on $g(\pi)$: 
\begin{align*}
    \EE_{\Dcal}[\alpha_h(\bar z_h,a_h,r_h,o_{h+1};g(\pi))\hat f_h(\bar z_h;g(\pi))-0.5 \hat f_h(\bar z_h;g(\pi))^2 ] \leq \Lambda. 
\end{align*}
From what we showed in \pref{eq:first_step}, this implies
\begin{align*}
    \EE_{\Dcal}[\alpha_h(\bar z_h,a_h,r_h,o_{h+1};g(\pi)) (\bar \Bcal^{\pi}_h g(\pi))(\bar z_h)-0.5 (\bar \Bcal^{\pi}_h g(\pi))^2(\bar z_h) ] \leq \Lambda + \frac{ 3(12H|\Acal|)^2\ln(|\Fcal_h||\Gcal_h||\Gcal_{h+1}| /\delta)}{m}. 
\end{align*}
Recall  we want to upper-bound the error for $\EE[0.5 (\bar \Bcal^{\pi}_h g(\pi))^2(\bar z_h)]. $ Here, we use the following observation later: 
\begin{align*}
 \EE[\alpha_h(\bar z_h,a_h,r_h,o_{h+1};g(\pi)) (\bar \Bcal^{\pi}_h g(\pi))(\bar z_h)-0.5 (\bar \Bcal^{\pi}_h g(\pi))^2(\bar z_h)]= \EE[0.5 (\bar \Bcal^{\pi}_h g(\pi))^2(\bar z_h)]. 
\end{align*}
We use Bernstein's inequality: with probability $1-\delta$, for any $g \in \Gcal$, 
\begin{align*}
   & |(\EE-\EE_{\Dcal})[\alpha_h(\bar z_h,a_h,r_h,o_{h+1};g) (\Bcal^{\pi}_h g)(\bar z_h)-0.5 (\Bcal^{\pi}_h g)^2(\bar z_h) ]| \\
    &\leq \sqrt{ \frac{\EE[(3|\Acal| H)^2(\Bcal^{\pi}_h g)^2(\bar z_h) ]\ln(2|\Gcal_h| |\Gcal_{h+1}|/\delta)  }{m} } + \frac{  (3|\Acal| H)^2 \ln(|\Gcal_h||\Gcal_{h+1}|/\delta) }{m}. 
\end{align*}
Here, we use 
\begin{align*}
    &\EE[\{\alpha_h(\bar z_h,a_h,r_h,o_{h+1};g) (\Bcal^{\pi}_h g)(\bar z_h)-0.5 (\Bcal^{\pi}_h g)^2(\bar z_h)\}^2 ] \\
    &\leq \EE[\{\alpha_h(\bar z_h,a_h,r_h,o_{h+1};g) (\Bcal^{\pi}_h g)(\bar z_h)-0.5 (\Bcal^{\pi}_h g)^2(\bar z_h)\} ](6|\Acal|H)^2. 
\end{align*}
Hereafter, we condition on the above event.

Finally, we have 
{ 
\begin{align*}
    &\EE[0.5 (\bar \Bcal^{\pi}_h g(\pi))^2(\bar z_h) ]\\
     &\leq  \EE_{\Dcal}[\alpha_h(\bar z_h,a_h,r_h,o_{h+1};g(\pi)) (\bar \Bcal^{\pi}_h g(\pi))(\bar z_h)-0.5 (\bar \Bcal^{\pi}_h g(\pi))^2(\bar z_h) ]+ \\ &+|(\EE-\EE_{\Dcal})[\alpha_h(\bar z_h,a_h,r_h,o_{h+1};g(\pi)) (\bar \Bcal^{\pi}_h g(\pi))(\bar z_h)-0.5 (\bar \Bcal^{\pi}_h g(\pi))^2(\bar z_h) ]|\\
    &\leq \Lambda + \frac{3(12 H |\Acal|)^2 \ln(4|\Fcal_h| |\Gcal_h| |\Gcal_{h+1}| /\delta)}{m}  \\
    &+ |(\EE-\EE_{\Dcal})[\alpha_h(\bar z_h,a_h,r_h,o_{h+1};g(\pi)) (\Bcal^{\pi}_h g(\pi))(\bar z_h)-0.5 (\Bcal^{\pi}_h g(\pi))^2(\bar z_h) ]| \\ 
       &\leq \Lambda + \frac{3(12 H |\Acal|)^2\ln(4|\Fcal_h| |\Gcal_h| |\Gcal_{h+1}| /\delta)}{m}  + \sqrt{\frac{\EE[0.5 (\Bcal^{\pi}_h g(\pi))^2(\bar z_h)  ]  \ln(4|\Gcal_h| |\Gcal_{h+1}|/\delta) }{m}} +\frac{  \ln(|\Gcal_h||\Gcal_{h+1}| /\delta) }{m}. 
\end{align*}
}
Hence, 
\begin{align*}
   \forall \pi \in \Pi, \forall g(\pi); \EE[0.5 (\Bcal^{\pi}_h g(\pi))^2(\bar z_h) ]\leq  \Lambda + c\frac{(H |\Acal|)^2\ln(|\Fcal_h| |\Gcal_h| |\Gcal_{h+1}| /\delta)}{m} . 
\end{align*}

\end{proof}

\subsection{Proof of Main Statement}

We define 
\begin{align*}
    |\Fcal_{\max}|= \max_{h \in [H]}|\Fcal_h|,\,    |\Pi_{\max}|= \max_{h \in [H]}|\Pi_h|,   |\Gcal_{\max}|= \max_{h \in [H]}|\Gcal_h|. 
\end{align*}

Let 
\begin{align*}
    \epsilon^2_{gen} &=  c_1  \frac{(H |\Acal|)^2\ln(|\Pi_{\max}|  |\Fcal_{\max}| |\Gcal_{\max}|  T(H+1) /\delta)}{m},\\
     \tilde \epsilon^2_{gen} &=  c_1 \frac{(H |\Acal|)^2\ln(|\Pi_{\max}|  |\Fcal_{\max}| |\Gcal_{\max}|  (H+1)/\delta)}{m},\\
    \epsilon_{ini} &=  c_3 \sqrt{\frac{(H |\Acal|)^2\ln( |\Gcal_1|  T(H+1) /\delta)}{m}}, \\
    T &=2Hd \ln\left(4Hd \left (\frac{B^2_XB^2_W}{\tilde \epsilon_{gen} } +1 \right) \right),\quad R= \epsilon^2_{gen}. 
\end{align*}
Then, from  the first statement in \pref{lem:decodable_key_lemma}, with probability $1-\delta$, $\forall t\in [T], \forall h \in [H],\forall \pi \in \Pi$
  { \small 
\begin{align}\label{eq:conditional_decodable}
   & \max_{f_h \in \Fcal_h} |\EE_{\Dcal^t_h}[|\Acal|\pi_h(a_h \mid \bar z_h)\{g^{\pi}_h(\bar z_h) - r_h-g^{\pi}_{h+1}(\bar z_{h+1})\}f_h(\bar z_h) - 0.5 f_h(\bar z_h)^2 ;   a_{1:M(h)-1} \sim  \pi^t,a_{M(h):h}\sim U(\Acal) ] | \\
    &\leq   c_1\frac{(H |\Acal|)^2\ln(|\Pi_{\max}|  |\Fcal_{\max}| |\Gcal_{\max}|  T(H+1)/\delta)}{m} .  \nonumber 
\end{align}
}
Besides, from the second statement in \pref{lem:decodable_key_lemma}, for $\pi \in \Pi , \forall t \in [T], \forall h \in [H]$, when $g(\pi)$ satisfies 
{ \small
\begin{align*}
    &\max_{f_h \in \Fcal_h} |\EE_{\Dcal^t_h}[|\Acal|\pi_h(a_h \mid \bar z_h)\{g_{h}(\pi)(\bar z_h) - r_h-g_{h+1}(\pi)(\bar z_{h+1})\}f_h(\bar z_h) - 0.5 f_h(\bar z_h)^2 ;   a_{1:M(h)-1} \sim  \pi^t,a_{M(h):h}\sim U(\Acal) ]| \\ 
    & \leq c_1\frac{(H |\Acal|)^2\ln(|\Pi_{\max}|  |\Fcal_{\max}| |\Gcal_{\max}|  T(H+1)/\delta)}{m}, 
\end{align*}
}
we have  
 { \small
\begin{align}\label{eq:conditional_decodable2}
 \EE[   (\bar \Bcal^{\pi}_h g(\pi))^2(\bar z_h);   a_{1:M(h)-1} \sim  \pi^t,a_{M(h):h-1}\sim U(\Acal) ]\leq (c_1 +  c_2)\frac{(H |\Acal|)^2\ln(|\Pi_{\max}|  |\Fcal_{\max}| |\Gcal_{\max}| TH /\delta)}{m} . 
\end{align}
}

We first show the optimism. Recall $    \pi^{\star} = \argmax_{\pi \in \Pi}J(\pi). $
\begin{lemma}[Optimism]
Set $R=\epsilon^2_{gen}$. For all $t \in [T]$,  $(\pi^{\star},g^{\pi^{\star}})$ is a feasible solution of the constrained program. Furthermore, we have $ J(\pi^{\star}) \leq \EE[g^t_1(o_1)]+ 2\epsilon_{ini}$ for any $t \in [T]$. 
\end{lemma}
\begin{proof}
For any $\pi \in \Pi$, letting $g^{\pi}\in \Gcal$ be a corresponding value link function, we have 
\begin{align*}
   \max_{f\in \Fcal_h} |\EE_{\Dcal^t_h}[l_h(\bar z_h,a_h,r_h,o_{h+1};f,\pi,g^{\pi}) ]|\leq \epsilon^2_{gen}. 
\end{align*}
using \eqref{eq:conditional_decodable}.  This implies $$\forall t\in[T],\forall h \in [H],  \max_{f\in \Fcal_h} |\EE_{\Dcal^t_h}[l_h(\bar z_h,a_h,r_h,o_{h+1};f,\pi^{\star},g^{\pi^{\star}}) ]|\leq \epsilon^2_{gen}.$$  
Hence, $(\pi^{\star},g^{\pi^{\star}})$ is a feasible set for any $t \in [T]$. Then, we have 
\begin{align*}
    J(\pi^{\star}) &= \EE[g^{\pi^{\star}}_1(o_1) ]    \leq \EE_{\Dcal^t_1}[ g^{\pi^{\star}}_1(o_1) ]+\epsilon_{ini} \\ 
      &\leq \EE_{\Dcal^t_1}[g^t_1(o_1)]+\epsilon_{ini} \leq \EE[g^t_1(o_1)]+2\epsilon_{ini} . 
\end{align*}
\end{proof}

Next, recall the following two statements. The following statements are proved as before in the proof of \pref{thm:online}.  
\begin{itemize}
    \item For any $t\in [T]$, 
    \begin{align*}
       J(\pi^{\star}) - J(\pi^t )  \leq \sum_{h=1}^H  |\langle  W_h(\pi^t ,g^t ),X_h(\pi^t )\rangle  |+  2\epsilon_{ini}. 
       \end{align*}
    \item Let $\Sigma_{t,h} = \lambda I + \sum_{\tau=0}^{t-1} X_h(\pi^{\tau})X_h(\pi^{\tau})^{\top}$. We have 
\begin{align*}
    \frac{1}{T}\sum_{t=0}^{T-1}\sum_{h=1}^H \|X_h(\pi^t)\|^2_{\Sigma^{-1}_{t,h}} \leq H\sqrt{\frac{d}{T}\ln \prns{1 + \frac{T B^2_X}{d\lambda}}}. 
\end{align*}

\end{itemize}

\begin{lemma}
\begin{align*}
    \|W_h(\pi^t,g^t) \|^2_{\Sigma_{t,h}}\leq 2 \lambda B^2_W + T |\Acal|^M \epsilon^2_{gen}.
\end{align*}
\end{lemma}
\begin{proof}
We have 
\begin{align*}
     \|W_h(\pi^t,g^t) \|^2_{\Sigma_{t,h}}  = \lambda \|W_h(\pi^t,g^t) \|^2_2 + \sum_{\tau =0}^{t-1} \langle W_h(\pi^{t},g^{t}), X_h(\pi^{\tau}) \rangle^2. 
 \end{align*}  
 The first term is upper-bounded by $\lambda B^2_W$. The second term is upper-bounded by 
\begin{align*}
&  \sum_{\tau =0}^{t-1} \langle W_h(\pi^{t},g^{t}), X_h(\pi^{\tau}) \rangle^2 \\ 
& \leq |\Acal|^M\sum_{\tau =0}^{t-1} \EE[\EE[ |\Acal|\pi^t_h(a_h\mid \bar z_h) g_h(\bar z_h) - r_h-g_{h+1}(\bar z_{h+1}) \mid \bar z_h; a_h \sim U(\Acal)  ]^2 ;   a_{1:M(h)-1} \sim  \pi^{\tau},a_{M(h):h-1}\sim U(\Acal)  ] \\
&= |\Acal|^M\sum_{\tau =0}^{t-1} \EE[   (\bar \Bcal^{\pi^t}_{h}g(\pi))^2(\bar z_h);   a_{1:M(h)-1} \sim  \pi^{\tau},a_{M(h):h-1}\sim U(\Acal) ]\\
&\leq |\Acal|^MT(c_1 +  c_2)\frac{(H |\Acal|)^2\ln(|\Pi_{\max}|  |\Fcal_{\max}| |\Gcal_{\max}| TH /\delta)}{m} \leq T |\Acal|^M \epsilon^2_{gen}. 
\end{align*}
From the first line to the second line, we use \pref{eq:decodable_bilinear_transform}. 
Here, from the third line to the fourth line, we use \pref{eq:conditional_decodable2}.
\end{proof}

The rest of the argument is the same as the proof in \pref{thm:online}. Finally, the following holds
\begin{align*}
       J(\pi^{\star}) - J(\hat \pi )   \leq 5\epsilon_{gen}|\Acal|^{M/2} \bracks{H^2 d \ln(4Hd (B^2_XB^2_W/\tilde \epsilon_{gen} +1))  }^{1/2} +2\epsilon_{ini}. 
\end{align*}

\paragraph{Sample Complexity Result.}

We want to find $m$ such that 
\begin{align*}
    \sqrt{\frac{H^2|\Acal|^{2+M}\ln(|\Pi_{\max} ||\Fcal_{\max}||\Gcal_{\max}|TH/\delta) }{m}  }[H^2 d \ln(HdB^2_X B^2_W m) ]^{1/2}\leq \epsilon. 
\end{align*}
where 
\begin{align*}
    T =  Hd\ln(Hd B^2_X B^2_W  m). 
\end{align*}
By organizing terms, we have
\begin{align*}
        \sqrt{\frac{H^4 d |\Acal|^{2+M}\ln(|\Pi_{\max}||\Fcal_{\max}||\Gcal_{\max}|Hd/\delta) \ln(HdB^2_X B^2_W m) }{m}  } \leq \epsilon. 
\end{align*}
Thus, setting the following $m$ is enough: 
\begin{align*}
    m =\tilde O\prns{ \frac{H^4d |\Acal|^{2+M}\ln(|\Pi_{\max}||\Fcal_{\max}||\Gcal_{\max}|/\delta)}{\epsilon^2} }. 
\end{align*}
The total sample we use $mTH$ is 
\begin{align*}
\tilde O\prns{ \frac{d^2 H^6 |\Acal|^{2+M}\ln(|\Pi_{\max}||\Fcal_{\max}||\Gcal_{\max}|/\delta)}{\epsilon^2} }.     
\end{align*}

\section{Sample Complexity in Observable POMDPs with Latent Low-rank Transition} \label{sec:low_rank_sample}

This section largely follows the one in \pref{sec:m_step_decodable_sample}. 

\subsection{Existence of Value Link Functions}

Since we consider the discrete setting, we can set the value link function class as \pref{sec:sample_complexity_tabular}. Hence, we set 
\begin{align*}
    \Gcal_h = \left\{ \langle \theta, \One(z)\otimes \OO^{\dagger} \One(o) \rangle; \|\theta\|_{\infty}\leq H  \right \}. 
\end{align*}
Then, we can ensure $\|\Gcal_h\|\leq H/\sigma_1$. Then, from the construction of $\Fcal_h$, we can also ensure $ \|\Fcal_h\|\leq 4H/\sigma_1$.

\subsection{PO-bilinear Rank Decomposition (Proof of \pref{lem:low_rank_po_bilinear})}\label{subsec:proof_po_bilinear}

In this section, we derive the PO-bilinear decomposition of observable POMDPs with the latent low-rank transition. We want to prove \pref{lem:low_rank_po_bilinear}.  Recall $M(h)=\max(h-M,1)$.

Using \citep[Lemma B.2]{efroni2022provable}, we have
\begin{align*}
& \mathrm{Br}(\pi,g; \pi') = \EE[\{g_h(\bar z_h) - r_h-g_{h+1}(\bar z_{h+1})\} ;   a_{1:M(h)-1} \sim  \pi',a_{M(h):h}\sim \pi  ]  \\
&= \EE[\{g_h(\bar z_h) - r_h-g_{h+1}(\bar z_{h+1})\} ; a_{1:M(h)-1} \sim  \pi',a_{M(h):h-1}\sim \mu^{\pi,h},a_{h} \sim \pi  ]  \\
&= \EE[ \EE[\{g_h(\bar z_h) - r_h-g_{h+1}(\bar z_{h+1})\} \mid s_{M(h)}  ; a_{M(h):h-1}\sim \mu^{\pi,h},a_{h} \sim \pi  ] ; a_{1:M(h)-1} \sim  \pi' ]  \\
&= \left \langle  X_h(\pi'), W_h(\pi,g)  \right\rangle  
\end{align*}
where
\begin{align*}
   W_h(\pi,g) &=  \int \EE[\{g_h(\bar z_h) - r_h-g_{h+1}(\bar z_{h+1})\} \mid s_{M(h)}  ; a_{M(h):h-1}\sim \mu^{\pi,h},a_{h} \sim \pi  ]\mu(s_{M(h)})\mathrm{d}(s_{M(h)}), \\
   X_h(\pi') &= \EE[\phi(s_{M(h)-1},a_{M(h)-1} )  ; a_{1:M(h)-1} \sim  \pi' ]. 
\end{align*}
Thus, the first condition in \pref{def:bilinear_minimax} is satisfied 

Next, we show the second condition in \pref{def:bilinear_minimax}. This is proved as follows: 
{\small 
\begin{align}
     & \frac{0.5}{|\Acal|^{M}}  \left \langle  X_h(\pi'), W_h(\pi,g)  \right\rangle^2 \nonumber \\
    & = \frac{0.5}{|\Acal|^{M}} \EE\left[\left( g_h(\bar z_h) -   (\Bcal^{\pi}_h g_{h+1})( \tau_h)  \right) ; a_{1:M(h)-1} \sim  \pi',a_{M(h):h-1}\sim \mu^{\pi,h} \right]^2\nonumber \\ 
    & \leq \frac{0.5}{|\Acal|^{M}} \EE\left[\left( g_h(\bar z_h) -   (\Bcal^{\pi}_h g_{h+1})( \tau_h )  \right)^2 ; a_{1:M(h)-1} \sim  \pi',a_{M(h):h-1}\sim \mu^{\pi,h} \right] \nonumber \\ 
    & \leq \frac{1}{|\Acal|^{M}} \max_{f \in \Fcal_h} \EE\left[\left( g_h(\bar z_h) -   (\Bcal^{\pi}_h g_{h+1})( \tau_h )  \right) f(\tau_h) - 0.5 f(\tau_h)^2 ; a_{1:M(h)-1} \sim  \pi',a_{M(h):h-1}\sim \mu^{\pi,h} \right] \nonumber\\
    & = \frac{1}{|\Acal|^{M}} \max_{f \in \Fcal_h} \EE\left[ |\Acal| \pi_h (a_h | \bar z_h) \left( g_h(\bar z_h) -   r_h - g_{h+1}(\bar z_{h+1}) \right) f(\tau_h) - 0.5 f(\tau_h)^2 ; a_{1:M(h)-1} \sim  \pi',a_{M(h):h-1}\sim \mu^{\pi,h}, a_h\sim \Ucal(\Acal) \right] \nonumber\\
    &\leq \max_{f \in \Fcal_h} \EE\left[ |\Acal| \pi_h (a_h | \bar z_h) \left( g_h(\bar z_h) -   r_h - g_{h+1}(\bar z_{h+1}) \right) f(\tau_h) - 0.5 f(\tau_h)^2 ; a_{1:M(h)-1} \sim  \pi',a_{M(h):h}\sim  \Ucal(\Acal) \right] \nonumber \\
    & = \max_{f \in \Fcal_h}  \EE\left[ l_h( \tau_h, a_h, r_h, o_{h+1}; f, \pi, g ) ; a_{1:M(h)-1} \sim  \pi',a_{M(h):h}\sim  \Ucal(\Acal) \right].  \nonumber 
\end{align}
}
From the first line to the second line, we use \citep[Lemma B.2]{efroni2022provable}. From the third to the fourth line, we use the Bellman completeness assumption: $-(\Bcal^{\pi}_h\Gcal) + \Gcal_h\subset \Fcal_h$. From the fourth line to the fifth line, we use importance sampling. 

The third condition 
\begin{align*}
    \left\lvert  \max_{f \in \Fcal_h}\EE[l_h(\tau_h, a_h, r_h, o_{h+1}; f, \pi, g^{\pi});a_{1:M(h)-1}\sim \pi', a_{M(h):h} \sim U(\Acal) ] \right\rvert = 0. 
\end{align*}
is easily proved. 

Finally, the following norm constraints hold: 
\begin{align*}
    \|W_h(\pi,g)\|\leq 3C_{\Gcal} \sqrt{d},\quad \|X_h(\pi')\|\leq 1.
\end{align*}

\paragraph{Sample Complexity Result.}

Following the same procedure as \pref{sec:m_step_decodable_sample}, here, we want to find $m$ such that 
\begin{align*}
    \sqrt{\frac{C^2_{\Gcal}|\Acal|^{2+M}\ln(|\Pi_{\max} ||\Fcal_{\max}||\Gcal_{\max}|TH/\delta) }{m}  }[H^2 d \ln(HdB^2_X B^2_W m) ]^{1/2}\leq \epsilon. 
\end{align*}
where 
\begin{align*}
    T =  Hd\ln(Hd B^2_X B^2_W  m). 
\end{align*}
By organizing terms, we have
\begin{align*}
        \sqrt{\frac{C^2_{\Gcal} H^2d |\Acal|^{2+M}\ln(|\Pi_{\max}||\Fcal_{\max}||\Gcal_{\max}|Hd/\delta) \ln(HdB^2_X B^2_W m) }{m}  } \leq \epsilon. 
\end{align*}
Thus, setting the following $m$ is enough
\begin{align*}
    m =\tilde O\prns{ \frac{H^4d |\Acal|^{2+M}\ln(|\Pi_{\max}||\Fcal_{\max}||\Gcal_{\max}|/\delta)}{\epsilon^2 \sigma^2_1} }. 
\end{align*}
The total sample we use $mTH$ is 
\begin{align*}
\tilde O\prns{ \frac{d^2 H^6 |\Acal|^{2+M}\ln(|\Pi_{\max}||\Fcal_{\max}||\Gcal_{\max}|/\delta)}{\epsilon^2 \sigma^2_1} }.    
\end{align*}
Finally, we plug-in $\ln(|\Pi_{\max}||\Fcal_{\max}||\Gcal_{\max}|/\delta) = \ln(|\Mcal|)$.

%% file: main_document/exponential_stability.tex
\section{Exponential Stability for POMDPs with Low-rank Transition}\label{sec:exponential_stability}

 In this section, we prove that the short memory policy is a globall near optimla policy in low-rank MDPs. We first introduce several notation. Next, we prove the exponential stability of Bayesian fileters, which immediately leads to the main statement. 

\paragraph{Notation.}

Given a belief $b \in \Delta(\Scal)$, an action and observation pair $(a,o)$,  we define the Bayesian update as follows. We define $B(b, o) \in \Delta(\Scal)$ as the operation that incorporates observation $o$, i.e., $b' = B(b,o)$ with $b'(s) =   O(o | s) b(s) / (\sum_{\bar s} O(o | \bar s) b(\bar s) )$, and $\TT_{a} b$ as the operation that incorporates the transition, i.e., $(\TT_a b)(s') = \sum_{s} b(s) \TT(s' | s,a)$. Finally, we denote $U(b, a, o)$ as the full Bayesian filter, i.e., $$U(b, a, o) = B( \TT_{a} b, o ).$$
Let us denote $b_0 \in \Delta(\Scal)$ as the initial latent state distribution. Given the first observation $o_1 \sim \OO(\cdot | s), s\sim b_0$, we denote $b_1 = B(b_0, o_1)$ as the initial belief of the system conditioned on the first observation $o_1$.  Given two beliefs $b,b'$, we define the distance $D_2(b,b'):= \log \EE_{s \sim b}[b(s)/b(s')]$

Consider a POMDP whose latent transition is low rank, i.e., $\TT(s'|s,a) = \mu(s')^\top \phi(s,a)$. For notation simplicity, we still consider discrete state, action, and observation space to avoid using measure theory languages.

\paragraph{Design of initial distribution.}

We want to design a good distribution for the initial distribution in an artificial Bayesian filter ignoring the history other than the short history.

The following lemma is from \citep[Lemma 4.9]{golowich2022planning} that quantifies the contraction of a Bayesian map. 
\begin{lemma}[Contraction propery of beliefs]\label{lem:contraction}
Suppose $b, b' \in \Delta(\Scal)$ and $\| b / b'\|_{\infty} <\infty$. Then we have:
\begin{align*}
\EE_{ s\sim b, o\sim \OO(s)} \left[ \sqrt{ \exp\left( \frac{ D_2( B(b, o), B(b',o) ) }{4}   - 1 \right) } \right] \leq \left( 1-\sigma^4_1 /2^{40}  \right) \sqrt{ \exp\left( \frac{ D_2( b, b' ) }{4} \right) -1 } 
\end{align*} 
\end{lemma}

Next, we compute the G-optimal design using feature $\phi(s,a):\Scal \times \Acal \to \RR$. Denote the G-optimal design as $\rho \in \Delta(\Scal\times\Acal)$. Here, we use assumption $\|\phi(s,a)\|\leq 1$ for any $(s,a)$ in Assumption \pref{assum:regular}, which ensures that $\phi(s,a)$ lives in a compact space for any $(s,a)$. The property is given as in \pref{thm:g_optimal}. In summary, the support of $\rho$ (denoted by $S_{\rho}$) is at most $d(d+1)/2$ points and for any $\phi(s,a)$, there exists $\alpha(s,a)$ such that 
\begin{align}\label{eq:g_optimal_design}
    \phi(s,a) = \sum_{i=1}^{|S_{\rho}|} \alpha_i(s,a)\phi(s^i,a^i)\rho^{1/2}(s^i,a^i),\quad  \alpha_i(s,a)/\rho^{1/2}(s^i,a^i) \leq d
\end{align} where we denote the points on the support $S_{\rho}$ as $\{s^i,a^i\}_{i=1}^{|S_{\rho}|}$.

We set our ``empty" belief as follows:
\begin{align*}
\tilde{b}_0(\cdot) := \sum_{\tilde s,\tilde a} \rho(\tilde s,\tilde a) \TT(\cdot | \tilde s,\tilde a) = \sum_{i=1}^{|S_{\rho}|} \rho(s^i,a^i) \TT(\cdot | s^i, a^i).
\end{align*} Note that this belief $\tilde b_0$ does not depend on any history. 
We aim to bound $D_2(b, \tilde b_0)$ using the following lemma where $b$ is some belief resulting from applying $\TT_{a}$ for any $a$ to a belief $\tilde b\in \Delta(\Scal)$. This is a newly introduce lemma. 

\begin{lemma}[Distance between the actual belief and the designed initial distribution] \label{lem:ini_distance}
For any distribution $b\in \Delta(\Scal)$ that results from a previous belief $\tilde b$ and a one-step latent transition under action $a$, i.e., $b(s) = \mathbb{T}_{a}\tilde b(\tilde s) $, we have:
\begin{align*}
D_2( b, \tilde b_0 ) \leq  \ln (d^3).
\end{align*}
\end{lemma}

\begin{proof}
For any $b \in \Delta(\Scal)$, using its definition, we have:
\begin{align*}
b(s) & = \sum_{\tilde s} \tilde b(\tilde s)  \phi(\tilde s, a)^\top \mu(s) \tag{Definition} \\
 &= \sum_{\tilde s} \tilde b(\tilde s) \sum_{i=1}^{S_{\rho}} \alpha_i(\tilde s,a) \rho^{1/2}(s^i,a^i) \phi(s^i,a^i)^\top \mu(s) \tag{Property of G-optimal design} \\
& =  \sum_{i=1}^{S_{\rho}} \underbrace{ \left( \sum_{\tilde s} \tilde b(\tilde s)  \alpha_i(\tilde s, a) \rho^{1/2}(s^i,a^i)  \right)}_{ := \beta_i } \phi(s^i, a^i)^\top \mu(s) 
\end{align*} 
Similarly, the construction of $\tilde b_0$ implies that $\tilde b_0(s) =  \sum_{i=1}^{S_{\rho}} \rho(s^i, a^i) \phi(s^i, a^i)^\top \mu(s)$, thus, we have:
\begin{align*}
 b(s) / \tilde b_0 (s)  & =  \sum_{i=1}^{S_{\rho}}  \frac{   \beta_i \phi(s^i,a^i)^\top \mu(s) }{ \sum_{j=1}^{S_{\rho}} \rho(s^j,a^j) \phi(s^j, a^j)^\top \mu(s)   }  \leq \sum_{i=1}^{S_{\rho}} \frac{  \beta_i \phi(s^i, a^i)^\top \mu(s)    }{  \rho(s^i, a^i) \phi(s^i, a^i)^\top \mu(s)     } \\
 & = \sum_{i=1}^{S_{\rho}} \frac{ \beta_i }{ \rho(s^i, a^i) } = \sum_{i=1}^{S_{\rho}} {  \sum_{\tilde s} \tilde b(\tilde s) \frac{\alpha_i(\tilde s, a)}{\rho^{1/2}(s^i, a^i)}  }  = \sum_{\tilde s} \tilde b(\tilde s) \sum_{i=1}^{S_{\rho}} \frac{\alpha_i(\tilde s, a)}{\rho^{1/2}(s^i, a^i)} \\
 &\leq \sum_{\tilde s} \tilde b(\tilde s) d^3 = d^3. \tag{Use propety of G-optimal design\pref{eq:g_optimal_design}}
\end{align*}
Thus, $D_2( b , \tilde b_0) = \ln\left( \EE_{s\sim b} \frac{ b(s) }{\tilde b_0(s)}  \right) \leq \ln d^3$.
\end{proof}

{
Now we prove the exponential stability by leveraging \pref{lem:ini_distance} and \pref{lem:contraction}. 

\begin{theorem}[Exponential stability for POMDPs with Low-rank Latent Transition]\label{thm:exponential_stability}
Consider a $t \geq C \gamma^{-4} \ln( d  / \epsilon)$.  Consider  any policy (full history dependent) $\pi$ and a trajectory $a_{1:h+t-1}, o_{1:h+t} \sim \pi$ for $h \geq 1 $. Denote $b_{h+t}$ as the (true) belief conditioned on $a_{1:h+t-1}, o_{1:h+t}$. 
For approximated belief, first for $h = 1$, we define $\bar b_{h+t}$ as:
\begin{align*}
\bar b_{1} = b_1, \quad \bar b_{1+\tau}(o_{1:1+\tau}, a_{1:1+\tau-1}) = U( \bar b_{n}( o_{1:\tau}, a_{1:\tau-1}), o_{1+\tau}, a_{1+\tau-1} ), 1\leq \tau  \leq t; 
\end{align*}
for $ h \geq 2$, we define $\bar b_{h+t}$ as:
\begin{align*}
\bar b_{h}  = B(\tilde b_0, o_h), \quad \bar b_{h+\tau}(o_{h:h+\tau}, a_{h:h+\tau-1}) = U(\bar b_{h+\tau - 1}(o_{h:h+\tau-1}, a_{h:h+\tau-2}), o_{h+\tau}, a_{h+\tau-1}), 1\leq \tau \leq t; 
\end{align*} 
Then we have:
\begin{align*}
\forall h \geq 1: \quad \EE [\left\| b_{h+t}(o_{1:h+t}, a_{1:h+t-1}) - \bar b_{h+t} (o_{h:h+t}, a_{h:h+t-1})   \right \|_1; a_{1:h+t-1} \sim \pi] \leq \epsilon.
\end{align*}
\end{theorem}

\begin{proof}

We define $$Y_{h+n}(o_{1:h+n}, a_{1:h+n-1}) = \sqrt{ \exp( D_2( b_{h+n}(o_{1:h+n}, a_{1:h+n-1}), \bar b_{h+n}(o_{h:h+n}, a_{h:h+n-1}) ) / 4 ) -1 }.$$ 
Hereafter, we omit $(o_{1:h+n}, a_{1:h+n-1})$ to simplify the notation.

We start from the base case $Y_h$ (i.e., $n = 0$). 

First case, consider $h > 1$, $b_{h} = U(b_{h-1}, o_h, a_{h-1})$. Denote $b'_h = \TT_{a_{h-1}} b_{h-1}$. From \pref{lem:ini_distance}, we know that:
\begin{align*}
\EE[D_2( b'_h, \tilde b_0  ) \mid o_{1:h-1}, a_{h-1}; a_{1:h-1}\sim \pi ]\leq \ln(d^3).
\end{align*}  
Thus, noting $b_h = B(b'_h,o_h)$ and $\bar b_h=B(\tilde b_h,o_h)$, we have: %
\begin{align*}
&\EE_{o_h \sim \OO b'_h}\bracks{ \sqrt{ \exp( D_2(  b_h,   B( \tilde b_0, o_h ) ) / 4  ) - 1  } \mid o_{1:h-1}, a_{1:h-1}; a_{1:h-1}\sim \pi  }  \\
&\leq \EE_{o_h \sim \OO b'_h}\bracks{ \sqrt{ \exp( D_2(  b'_h ,    \tilde b_0 ) / 4  ) - 1  } \mid o_{1:h-1}, a_{1:h-1}; a_{1:h-1}\sim \pi  } \tag{From \pref{lem:contraction}} \\
&\leq (1-\sigma^4_1 / 2^{40})  d^{3/2}
\end{align*} which implies the base case:
\begin{align*}
\EE [ Y_h  \mid o_{1:h-1}, a_{h-1}; a_{1:h-1}\sim \pi   ]  \leq (1-\sigma^4_1 / 2^{40})   d^{3/2}.
\end{align*}

Now for any $n \geq 1$, we have:
\begin{align*}
&\EE[ Y_{h+n }\mid o_{1:h-1}, a_{1:h-1}; a_{1:h+n-1}\sim \pi ] \\
& = \EE\left[ \sqrt { \exp\left( D_2( b_{h+n}, \bar b_{h+n} ) / 4\right)  -1 } \mid o_{1:h-1}, a_{1:h-1}; a_{1:h+n-1}\sim \pi \right]  \\
& \leq (1-\sigma^4_1 / 2^{40}) \EE \left[ \sqrt{  \exp\left( D_2\left(  ( \TT_{a_{h+n-1}} b_{h+n-1})    ,   (\TT_{a_{h+n-1}} \bar b_{h+n-1}  ) \right) / 4\right)  -1       } \mid o_{1:h-1}, a_{1:h-1}; a_{1:h+n-1}\sim \pi \right] \\
& \leq   (1-\sigma^4_1 / 2^{40}) \EE \left[ \sqrt{  \exp\left( D_2\left(   b_{h+n-1}  ,   \bar b_{h+n-1} ) \right) / 4\right)  -1       } \mid o_{1:h-1}, a_{1:h-1}; a_{1:h+n-1}\sim \pi \right] \tag{Data processing inequality from \citep[Lemma 2.7]{golowich2022planning}} \\
& = (1-\sigma^4_1 / 2^{40}) \EE[ Y_{h+n-1} \mid o_{1:h-1}, a_{1:h-1}; a_{1:h+n-1}\sim \pi ]. 
\end{align*}This completes the induction step. Adding expectation with respect to the history $a_{1:h-1}, o_{1:h-1}$ back, we conclude the proof.

When $h = 1$, we simply start with the original belief $b_1$. For any $0\leq n \leq t$, we simply set $\bar b_{1+n} = b_{1+n}$, thus the conclusion still holds.

\end{proof}

The above \pref{thm:exponential_stability} indicates that in order to approximate the ground truth belief $b_{h+t}$ that is conditioned on the entire history, we only need to apply the Bayesian filter on the M memory $\bar z_{h+t}$ starting from a fixed distribution $\tilde b_0$. The existence of such $\tilde b_0$ is proven by construction where we rely on the low-rankness of the latent transition and a D-optimal design over $\Scal\times\Acal$ using the feature $\phi$.

The above \pref{thm:exponential_stability}  together with the proof of Theorem 1.2 in \cite{golowich2022planning} immediately implies for $M  = \Theta( C (\sigma_1)^{-4} \ln(d H / \epsilon))$ (with $C$ being some absolute constant), there must exists an M-memory policy $\pi^\star$, such that $   J(\pi^\star_{gl})-J(\pi^{\star})  \leq \epsilon$ -- thus a globally optimal policy can be approximated by a policy that only relies on short memories. 
}

%% file: main_document/ape_auxi.tex
\section{Auxiliary Lemmas}

We use the following in \pref{sec:lqg}. 
\begin{lemma}[Useful inequalities]\label{lem:useful}
~
\begin{itemize}
    \item \begin{align*}
    & \|AB\|\leq \|A\| \|B \|, \|AB\|_{F} \leq \|A\| \|B\|_{F} \\
    & \mathrm{vec}(aa^{\top}) = a \otimes a, \|\mathrm{vec}(A)\|_2 =\|A\|_F, \mathrm{Tr}(AB) = \mathrm{vec}(A^{\top})^{\top} \mathrm{vec}(B). 
\end{align*}
    \item When $A$ and $B$ are semi positive definite matrices, we have 
\begin{align*}
\mathrm{Tr}(AB) \leq \|A\|\mathrm{Tr}(B).   
\end{align*}
\end{itemize}

\end{lemma}

The following lemma is useful when we calculate the sample complexity.

\begin{lemma}\label{lem:sample_auxiliary}
The following is satisfied 
\begin{align*}
    \sqrt{\frac{B_1}{m}\ln^2(B_2m+B_3)}\leq c\epsilon 
\end{align*}
when \begin{align*}
    m = c\frac{B_1}{\epsilon^2}\{\ln(m(B_2+B_3+1))\}^2. 
\end{align*}
for some constant $c$.  
\end{lemma}

%% file: arxiv_draft.bbl
\newcommand{\etalchar}[1]{$^{#1}$}
\begin{thebibliography}{WCYW22}

\bibitem[AHKS20]{agarwal2020pc}
Alekh Agarwal, Mikael Henaff, Sham Kakade, and Wen Sun.
\newblock Pc-pg: Policy cover directed exploration for provable policy gradient
  learning.
\newblock {\em Advances in Neural Information Processing Systems},
  33:13399--13412, 2020.

\bibitem[AKKS20]{agarwal2020flambe}
Alekh Agarwal, Sham Kakade, Akshay Krishnamurthy, and Wen Sun.
\newblock Flambe: Structural complexity and representation learning of low rank
  mdps.
\newblock {\em Advances in neural information processing systems},
  33:20095--20107, 2020.

\bibitem[ALA16]{azizzadenesheli2016reinforcement}
Kamyar Azizzadenesheli, Alessandro Lazaric, and Animashree Anandkumar.
\newblock Reinforcement learning of pomdps using spectral methods.
\newblock In {\em Conference on Learning Theory}, pages 193--256. PMLR, 2016.

\bibitem[BBC{\etalchar{+}}19]{berner2019dota}
Christopher Berner, Greg Brockman, Brooke Chan, Vicki Cheung, Przemys{\l}aw
  D{\k{e}}biak, Christy Dennison, David Farhi, Quirin Fischer, Shariq Hashme,
  Chris Hesse, et~al.
\newblock Dota 2 with large scale deep reinforcement learning.
\newblock {\em arXiv preprint arXiv:1912.06680}, 2019.

\bibitem[Ber12]{bertsekas2012dynamic}
Dimitri Bertsekas.
\newblock {\em Dynamic programming and optimal control: Volume I}, volume~1.
\newblock Athena scientific, 2012.

\bibitem[BGG13]{boots2013hilbert}
Byron Boots, Geoffrey Gordon, and Arthur Gretton.
\newblock Hilbert space embeddings of predictive state representations.
\newblock {\em arXiv preprint arXiv:1309.6819}, 2013.

\bibitem[BK21]{bennett2021}
Andrew Bennett and Nathan Kallus.
\newblock Proximal reinforcement learning: Efficient off-policy evaluation in
  partially observed markov decision processes.
\newblock 2021.

\bibitem[BSG11]{boots2011closing}
Byron Boots, Sajid~M Siddiqi, and Geoffrey~J Gordon.
\newblock Closing the learning-planning loop with predictive state
  representations.
\newblock {\em The International Journal of Robotics Research}, 30(7):954--966,
  2011.

\bibitem[CJ19]{chen2019information}
Jinglin Chen and Nan Jiang.
\newblock Information-theoretic considerations in batch reinforcement learning.
\newblock In {\em International Conference on Machine Learning}, pages
  1042--1051. PMLR, 2019.

\bibitem[CO20]{chowdhury2020no}
Sayak~Ray Chowdhury and Rafael Oliveira.
\newblock No-regret reinforcement learning with value function approximation: a
  kernel embedding approach.
\newblock {\em arXiv preprint arXiv:2011.07881}, 2020.

\bibitem[CPS{\etalchar{+}}20]{cui2020semiparametric}
Yifan Cui, Hongming Pu, Xu~Shi, Wang Miao, and Eric~Tchetgen Tchetgen.
\newblock Semiparametric proximal causal inference.
\newblock {\em arXiv preprint arXiv:2011.08411}, 2020.

\bibitem[CYW22]{cai2022sample}
Qi~Cai, Zhuoran Yang, and Zhaoran Wang.
\newblock Sample-efficient reinforcement learning for pomdps with linear
  function approximations.
\newblock {\em arXiv preprint arXiv:2204.09787}, 2022.

\bibitem[Dea18]{deaner2018proxy}
Ben Deaner.
\newblock Proxy controls and panel data.
\newblock {\em arXiv preprint arXiv:1810.00283}, 2018.

\bibitem[DHB{\etalchar{+}}17]{downey2017predictive}
Carlton Downey, Ahmed Hefny, Byron Boots, Geoffrey~J Gordon, and Boyue Li.
\newblock Predictive state recurrent neural networks.
\newblock {\em Advances in Neural Information Processing Systems}, 30, 2017.

\bibitem[DJW20]{duan2020minimax}
Yaqi Duan, Zeyu Jia, and Mengdi Wang.
\newblock Minimax-optimal off-policy evaluation with linear function
  approximation.
\newblock In {\em International Conference on Machine Learning}, pages
  2701--2709. PMLR, 2020.

\bibitem[DKJ{\etalchar{+}}19]{du2019provably}
Simon Du, Akshay Krishnamurthy, Nan Jiang, Alekh Agarwal, Miroslav Dudik, and
  John Langford.
\newblock Provably efficient rl with rich observations via latent state
  decoding.
\newblock In {\em International Conference on Machine Learning}, pages
  1665--1674. PMLR, 2019.

\bibitem[DKL{\etalchar{+}}21]{du2021bilinear}
Simon Du, Sham Kakade, Jason Lee, Shachar Lovett, Gaurav Mahajan, Wen Sun, and
  Ruosong Wang.
\newblock Bilinear classes: A structural framework for provable generalization
  in rl.
\newblock In {\em International Conference on Machine Learning}, pages
  2826--2836. PMLR, 2021.

\bibitem[DLMS20]{DikkalaNishanth2020MEoC}
Nishanth Dikkala, Greg Lewis, Lester Mackey, and Vasilis Syrgkanis.
\newblock Minimax estimation of conditional moment models.
\newblock In {\em Advances in Neural Information Processing Systems},
  volume~33, pages 12248--12262, 2020.

\bibitem[EDKM05]{even2005reinforcement}
Eyal Even-Dar, Sham~M Kakade, and Yishay Mansour.
\newblock Reinforcement learning in pomdps without resets.
\newblock 2005.

\bibitem[EJKM22]{efroni2022provable}
Yonathan Efroni, Chi Jin, Akshay Krishnamurthy, and Sobhan Miryoosefi.
\newblock Provable reinforcement learning with a short-term memory.
\newblock {\em arXiv preprint arXiv:2202.03983}, 2022.

\bibitem[FKQR21]{foster2021statistical}
Dylan~J Foster, Sham~M Kakade, Jian Qian, and Alexander Rakhlin.
\newblock The statistical complexity of interactive decision making.
\newblock {\em arXiv preprint arXiv:2112.13487}, 2021.

\bibitem[GDB16]{guo2016pac}
Zhaohan~Daniel Guo, Shayan Doroudi, and Emma Brunskill.
\newblock A pac rl algorithm for episodic pomdps.
\newblock In {\em Artificial Intelligence and Statistics}, pages 510--518.
  PMLR, 2016.

\bibitem[GMR22a]{golowich2022learning}
Noah Golowich, Ankur Moitra, and Dhruv Rohatgi.
\newblock Learning in observable pomdps, without computationally intractable
  oracles.
\newblock {\em arXiv preprint arXiv:2206.03446}, 2022.

\bibitem[GMR22b]{golowich2022planning}
Noah Golowich, Ankur Moitra, and Dhruv Rohatgi.
\newblock Planning in observable pomdps in quasipolynomial time.
\newblock {\em arXiv preprint arXiv:2201.04735}, 2022.

\bibitem[HDG15]{hefny2015supervised}
Ahmed Hefny, Carlton Downey, and Geoffrey~J Gordon.
\newblock Supervised learning for dynamical system learning.
\newblock {\em Advances in neural information processing systems}, 28, 2015.

\bibitem[HDL{\etalchar{+}}21]{hao2021sparse}
Botao Hao, Yaqi Duan, Tor Lattimore, Csaba Szepesv{\'a}ri, and Mengdi Wang.
\newblock Sparse feature selection makes batch reinforcement learning more
  sample efficient.
\newblock In {\em International Conference on Machine Learning}, pages
  4063--4073. PMLR, 2021.

\bibitem[Hes18]{hespanha2018linear}
Joao~P Hespanha.
\newblock {\em Linear systems theory}.
\newblock Princeton university press, 2018.

\bibitem[HFP13]{hamilton2013modelling}
William~L Hamilton, Mahdi~Milani Fard, and Joelle Pineau.
\newblock Modelling sparse dynamical systems with compressed predictive state
  representations.
\newblock In {\em International Conference on Machine Learning}, pages
  178--186. PMLR, 2013.

\bibitem[HKZ12]{hsu2012spectral}
Daniel Hsu, Sham~M Kakade, and Tong Zhang.
\newblock A spectral algorithm for learning hidden markov models.
\newblock {\em Journal of Computer and System Sciences}, 78(5):1460--1480,
  2012.

\bibitem[IP08]{izadi2008point}
Masoumeh~T Izadi and Doina Precup.
\newblock Point-based planning for predictive state representations.
\newblock In {\em Conference of the Canadian Society for Computational Studies
  of Intelligence}, pages 126--137. Springer, 2008.

\bibitem[Jae98]{jaeger1998discrete}
Herbert Jaeger.
\newblock {\em Discrete-time, discrete-valued observable operator models: a
  tutorial}.
\newblock GMD-Forschungszentrum Informationstechnik Darmstadt, Germany, 1998.

\bibitem[Jae00]{jaeger2000observable}
Herbert Jaeger.
\newblock Observable operator models for discrete stochastic time series.
\newblock {\em Neural computation}, 12(6):1371--1398, 2000.

\bibitem[JKA{\etalchar{+}}17]{jiang2017contextual}
Nan Jiang, Akshay Krishnamurthy, Alekh Agarwal, John Langford, and Robert~E
  Schapire.
\newblock Contextual decision processes with low bellman rank are
  pac-learnable.
\newblock In {\em International Conference on Machine Learning}, pages
  1704--1713. PMLR, 2017.

\bibitem[JKKL20]{jin2020sample}
Chi Jin, Sham Kakade, Akshay Krishnamurthy, and Qinghua Liu.
\newblock Sample-efficient reinforcement learning of undercomplete pomdps.
\newblock {\em Advances in Neural Information Processing Systems},
  33:18530--18539, 2020.

\bibitem[JLM21]{jin2021bellman}
Chi Jin, Qinghua Liu, and Sobhan Miryoosefi.
\newblock Bellman eluder dimension: New rich classes of rl problems, and
  sample-efficient algorithms.
\newblock {\em Advances in Neural Information Processing Systems}, 34, 2021.

\bibitem[JYWJ20]{jin2020provably}
Chi Jin, Zhuoran Yang, Zhaoran Wang, and Michael~I Jordan.
\newblock Provably efficient reinforcement learning with linear function
  approximation.
\newblock In {\em Conference on Learning Theory}, pages 2137--2143. PMLR, 2020.

\bibitem[KECM21]{kwon2021rl}
Jeongyeol Kwon, Yonathan Efroni, Constantine Caramanis, and Shie Mannor.
\newblock Rl for latent mdps: Regret guarantees and a lower bound.
\newblock {\em Advances in Neural Information Processing Systems}, 34, 2021.

\bibitem[KJS15]{kulesza2015spectral}
Alex Kulesza, Nan Jiang, and Satinder Singh.
\newblock Spectral learning of predictive state representations with
  insufficient statistics.
\newblock In {\em Twenty-Ninth AAAI Conference on Artificial Intelligence},
  2015.

\bibitem[KMN99]{kearns1999approximate}
Michael Kearns, Yishay Mansour, and Andrew Ng.
\newblock Approximate planning in large pomdps via reusable trajectories.
\newblock {\em Advances in Neural Information Processing Systems}, 12, 1999.

\bibitem[KMU21]{kallus2021causal}
Nathan Kallus, Xiaojie Mao, and Masatoshi Uehara.
\newblock Causal inference under unmeasured confounding with negative controls:
  A minimax learning approach.
\newblock {\em arXiv preprint arXiv:2103.14029}, 2021.

\bibitem[KW60]{kiefer1960equivalence}
Jack Kiefer and Jacob Wolfowitz.
\newblock The equivalence of two extremum problems.
\newblock {\em Canadian Journal of Mathematics}, 12:363--366, 1960.

\bibitem[LAHA20]{lale2020regret}
Sahin Lale, Kamyar Azizzadenesheli, Babak Hassibi, and Anima Anandkumar.
\newblock Regret minimization in partially observable linear quadratic control.
\newblock {\em arXiv preprint arXiv:2002.00082}, 2020.

\bibitem[LCSJ22]{liu2022partially}
Qinghua Liu, Alan Chung, Csaba Szepesv{\'a}ri, and Chi Jin.
\newblock When is partially observable reinforcement learning not scary?
\newblock {\em arXiv preprint arXiv:2204.08967}, 2022.

\bibitem[Lit96]{littman1996algorithms}
Michael~Lederman Littman.
\newblock {\em Algorithms for sequential decision-making}.
\newblock Brown University, 1996.

\bibitem[LMPR20]{li2020efficient}
Tianyu Li, Bogdan Mazoure, Doina Precup, and Guillaume Rabusseau.
\newblock Efficient planning under partial observability with unnormalized q
  functions and spectral learning.
\newblock In {\em International Conference on Artificial Intelligence and
  Statistics}, pages 2852--2862. PMLR, 2020.

\bibitem[LS01]{littman2001predictive}
Michael Littman and Richard~S Sutton.
\newblock Predictive representations of state.
\newblock {\em Advances in neural information processing systems}, 14, 2001.

\bibitem[LS20]{lattimore2020bandit}
Tor Lattimore and Csaba Szepesv{\'a}ri.
\newblock {\em Bandit algorithms}.
\newblock Cambridge University Press, 2020.

\bibitem[MHKL20]{misra2020kinematic}
Dipendra Misra, Mikael Henaff, Akshay Krishnamurthy, and John Langford.
\newblock Kinematic state abstraction and provably efficient rich-observation
  reinforcement learning.
\newblock In {\em International conference on machine learning}, pages
  6961--6971. PMLR, 2020.

\bibitem[MKS{\etalchar{+}}13]{mnih2013playing}
Volodymyr Mnih, Koray Kavukcuoglu, David Silver, Alex Graves, Ioannis
  Antonoglou, Daan Wierstra, and Martin Riedmiller.
\newblock Playing atari with deep reinforcement learning.
\newblock {\em arXiv preprint arXiv:1312.5602}, 2013.

\bibitem[MST18]{miao2018confounding}
Wang Miao, Xu~Shi, and Eric~Tchetgen Tchetgen.
\newblock A confounding bridge approach for double negative control inference
  on causal effects.
\newblock {\em arXiv preprint arXiv:1808.04945}, 2018.

\bibitem[MTR19]{mania2019certainty}
Horia Mania, Stephen Tu, and Benjamin Recht.
\newblock Certainty equivalent control of lqr is efficient.
\newblock {\em arXiv preprint arXiv:1902.07826}, 2019.

\bibitem[Mur00]{murphy2000survey}
Kevin~P Murphy.
\newblock A survey of pomdp solution techniques.
\newblock {\em environment}, 2(10), 2000.

\bibitem[MZG{\etalchar{+}}21]{mastouri2021proximal}
Afsaneh Mastouri, Yuchen Zhu, Limor Gultchin, Anna Korba, Ricardo Silva, Matt~J
  Kusner, Arthur Gretton, and Krikamol Muandet.
\newblock Proximal causal learning with kernels: Two-stage estimation and
  moment restriction.
\newblock {\em arXiv preprint arXiv:2105.04544}, 2021.

\bibitem[NBGF12]{nishiyama2012hilbert}
Yu~Nishiyama, Abdeslam Boularias, Arthur Gretton, and Kenji Fukumizu.
\newblock Hilbert space embeddings of pomdps.
\newblock {\em arXiv preprint arXiv:1210.4887}, 2012.

\bibitem[PT87]{papadimitriou1987complexity}
Christos~H Papadimitriou and John~N Tsitsiklis.
\newblock The complexity of markov decision processes.
\newblock {\em Mathematics of operations research}, 12(3):441--450, 1987.

\bibitem[PVSP06]{porta2006point}
Josep~M Porta, Nikos Vlassis, Matthijs~TJ Spaan, and Pascal Poupart.
\newblock Point-based value iteration for continuous pomdps.
\newblock 2006.

\bibitem[RGT04]{rosencrantz2004learning}
Matthew Rosencrantz, Geoff Gordon, and Sebastian Thrun.
\newblock Learning low dimensional predictive representations.
\newblock In {\em Proceedings of the twenty-first international conference on
  Machine learning}, page~88, 2004.

\bibitem[SBS{\etalchar{+}}10]{song2010hilbert}
Le~Song, Byron Boots, Sajid Siddiqi, Geoffrey~J Gordon, and Alex Smola.
\newblock Hilbert space embeddings of hidden markov models.
\newblock 2010.

\bibitem[SHSF09]{song2009hilbert}
Le~Song, Jonathan Huang, Alex Smola, and Kenji Fukumizu.
\newblock Hilbert space embeddings of conditional distributions with
  applications to dynamical systems.
\newblock In {\em Proceedings of the 26th Annual International Conference on
  Machine Learning}, pages 961--968, 2009.

\bibitem[Sin21]{singh2021finite}
Rahul Singh.
\newblock A finite sample theorem for longitudinal causal inference with
  machine learning: Long term, dynamic, and mediated effects.
\newblock {\em arXiv preprint arXiv:2112.14249}, 2021.

\bibitem[SJK{\etalchar{+}}19]{sun2019model}
Wen Sun, Nan Jiang, Akshay Krishnamurthy, Alekh Agarwal, and John Langford.
\newblock Model-based rl in contextual decision processes: Pac bounds and
  exponential improvements over model-free approaches.
\newblock In {\em Conference on learning theory}, pages 2898--2933. PMLR, 2019.

\bibitem[SJR04]{singh2004predictive}
Satinder Singh, Michael~R James, and Matthew~R Rudary.
\newblock Predictive state representations: a new theory for modeling dynamical
  systems.
\newblock In {\em Proceedings of the 20th conference on Uncertainty in
  artificial intelligence}, pages 512--519, 2004.

\bibitem[SKKS09]{srinivas2009gaussian}
Niranjan Srinivas, Andreas Krause, Sham~M Kakade, and Matthias Seeger.
\newblock Gaussian process optimization in the bandit setting: No regret and
  experimental design.
\newblock {\em arXiv preprint arXiv:0912.3995}, 2009.

\bibitem[SPK13]{shani2013survey}
Guy Shani, Joelle Pineau, and Robert Kaplow.
\newblock A survey of point-based pomdp solvers.
\newblock {\em Autonomous Agents and Multi-Agent Systems}, 27(1):1--51, 2013.

\bibitem[SSH20]{simchowitz2020improper}
Max Simchowitz, Karan Singh, and Elad Hazan.
\newblock Improper learning for non-stochastic control.
\newblock In {\em Conference on Learning Theory}, pages 3320--3436. PMLR, 2020.

\bibitem[SUJ21]{shi2021minimax}
Chengchun Shi, Masatoshi Uehara, and Nan Jiang.
\newblock A minimax learning approach to off-policy evaluation in partially
  observable markov decision processes.
\newblock {\em arXiv preprint arXiv:2111.06784}, 2021.

\bibitem[SVBB16]{sun2016learning}
Wen Sun, Arun Venkatraman, Byron Boots, and J~Andrew Bagnell.
\newblock Learning to filter with predictive state inference machines.
\newblock In {\em International conference on machine learning}, pages
  1197--1205. PMLR, 2016.

\bibitem[TJ15]{thon2015links}
Michael~R Thon and Herbert Jaeger.
\newblock Links between multiplicity automata, observable operator models and
  predictive state representations: a unified learning framework.
\newblock {\em J. Mach. Learn. Res.}, 16:103--147, 2015.

\bibitem[TSM20]{tennenholtz2020off}
Guy Tennenholtz, Uri Shalit, and Shie Mannor.
\newblock Off-policy evaluation in partially observable environments.
\newblock In {\em Proceedings of the AAAI Conference on Artificial
  Intelligence}, volume~34, pages 10276--10283, 2020.

\bibitem[UIJ{\etalchar{+}}21]{uehara2021finite}
Masatoshi Uehara, Masaaki Imaizumi, Nan Jiang, Nathan Kallus, Wen Sun, and
  Tengyang Xie.
\newblock Finite sample analysis of minimax offline reinforcement learning:
  Completeness, fast rates and first-order efficiency.
\newblock {\em arXiv preprint arXiv:2102.02981}, 2021.

\bibitem[UZS21]{uehara2021representation}
Masatoshi Uehara, Xuezhou Zhang, and Wen Sun.
\newblock Representation learning for online and offline rl in low-rank mdps.
\newblock {\em arXiv preprint arXiv:2110.04652}, 2021.

\bibitem[VODM12]{van2012subspace}
Peter Van~Overschee and Bart De~Moor.
\newblock {\em Subspace identification for linear systems:
  Theory—Implementation—Applications}.
\newblock Springer Science \& Business Media, 2012.

\bibitem[WCYW22]{wang2022embed}
Lingxiao Wang, Qi~Cai, Zhuoran Yang, and Zhaoran Wang.
\newblock Embed to control partially observed systems: Representation learning
  with provable sample efficiency.
\newblock {\em arXiv preprint arXiv:2205.13476}, 2022.

\bibitem[XCGZ21]{xiong2021sublinear}
Yi~Xiong, Ningyuan Chen, Xuefeng Gao, and Xiang Zhou.
\newblock Sublinear regret for learning pomdps.
\newblock {\em arXiv preprint arXiv:2107.03635}, 2021.

\bibitem[XKG21]{xu2021deep}
Liyuan Xu, Heishiro Kanagawa, and Arthur Gretton.
\newblock Deep proxy causal learning and its application to confounded bandit
  policy evaluation.
\newblock {\em Advances in Neural Information Processing Systems},
  34:26264--26275, 2021.

\bibitem[YW20]{yang2020reinforcement}
Lin Yang and Mengdi Wang.
\newblock Reinforcement learning in feature space: Matrix bandit, kernels, and
  regret bound.
\newblock In {\em International Conference on Machine Learning}, pages
  10746--10756. PMLR, 2020.

\bibitem[ZLKB20]{zanette2020learning}
Andrea Zanette, Alessandro Lazaric, Mykel Kochenderfer, and Emma Brunskill.
\newblock Learning near optimal policies with low inherent bellman error.
\newblock In {\em International Conference on Machine Learning}, pages
  10978--10989. PMLR, 2020.

\bibitem[ZSU{\etalchar{+}}22]{zhang2022efficient}
Xuezhou Zhang, Yuda Song, Masatoshi Uehara, Mengdi Wang, Wen Sun, and Alekh
  Agarwal.
\newblock Efficient reinforcement learning in block mdps: A model-free
  representation learning approach.
\newblock {\em arXiv preprint arXiv:2202.00063}, 2022.

\end{thebibliography}
